\documentclass[lettersize,journal]{IEEEtran}
\usepackage{amsmath,amsfonts}
\usepackage{algorithmic}
\usepackage{array}
\usepackage[caption=false,font=normalsize,labelfont=sf,textfont=sf]{subfig}
\usepackage{textcomp}
\usepackage{stfloats}
\usepackage{url}
\usepackage{verbatim}
\usepackage{graphicx}
\usepackage[english]{babel}
\usepackage[utf8]{inputenc}
\usepackage{algorithm}
\usepackage{siunitx}
\usepackage{mathrsfs}
\usepackage{amssymb}
\usepackage{xcolor}
\usepackage{amsthm}
\usepackage{lipsum}
\newtheorem{theorem}{Theorem}
\theoremstyle{definition}
\usepackage{orcidlink}

\newtheorem{lemma}{Lemma}
\usepackage{tikz}
\usetikzlibrary{shapes,arrows}
\usepackage{xcolor}
\usepackage[numbers]{natbib}
\def\BibTeX{{\rm B\kern-.05em{\sc i\kern-.025em b}\kern-.08em
    T\kern-.1667em\lower.7ex\hbox{E}\kern-.125emX}}
\usepackage{balance}
\usepackage{hyperref}
\begin{document}
\title{Granular Ball Twin Support Vector Machine}
\author{A. Quadir\,\orcidlink{0009-0002-0516-316X}, ~\IEEEmembership{Graduate Student Member,~IEEE}, M. Sajid\,\orcidlink{0009-0008-0465-5211}, ~\IEEEmembership{Graduate Student Member,~IEEE}, M. Tanveer{$^*$}\,\orcidlink{0000-0002-5727-3697}, ~\IEEEmembership{Senior Member,~IEEE} 
\thanks{ \noindent $^*$Corresponding Author\\
Manuscript submitted to IEEE TRANSACTIONS ON NEURAL NETWORKS AND LEARNING SYSTEMS: 19 September 2023; revised 13 February 2024 and
14 July 2024; accepted 05 October 2024. \\ 
 A. Quadir, M. Sajid, and M. Tanveer are with the Department of Mathematics, Indian Institute of Technology Indore, Simrol, Indore, 453552, India (e-mail: mscphd2207141002@iiti.ac.in, phd2101241003@iiti.ac.in and mtanveer@iiti.ac.in). \\
 Digital Object Identifier 10.1109/TNNLS.2024.3476391 
 }}

\maketitle
\begin{abstract}
Twin support vector machine (TSVM) is an emerging machine learning model with versatile applicability in classification and regression endeavors. Nevertheless, TSVM confronts noteworthy challenges: $(i)$ the imperative demand for matrix inversions presents formidable obstacles to its efficiency and applicability on large-scale datasets; $(ii)$ the omission of the structural risk minimization (SRM) principle in its primal formulation heightens the vulnerability to overfitting risks; and $(iii)$ the TSVM exhibits a high susceptibility to noise and outliers, and also demonstrates instability when subjected to resampling. In view of the aforementioned challenges, we propose the granular ball twin support vector machine (GBTSVM). GBTSVM takes granular balls, rather than individual data points, as inputs to construct a classifier. These granular balls, characterized by their coarser granularity, exhibit robustness to resampling and reduced susceptibility to the impact of noise and outliers. We further propose a novel large-scale granular ball twin support vector machine (LS-GBTSVM). LS-GBTSVM's optimization formulation ensures two critical facets: $(i)$ it eliminates the need for matrix inversions, streamlining the LS-GBTSVM's computational efficiency, and $(ii)$ it incorporates the SRM principle through the incorporation of regularization terms, effectively addressing the issue of overfitting. The proposed LS-GBTSVM exemplifies efficiency, scalability for large datasets, and robustness against noise and outliers. We conduct a comprehensive evaluation of the GBTSVM and LS-GBTSVM models on benchmark datasets from UCI and KEEL, both with and without the addition of label noise, and compared with existing baseline models. Furthermore, we extend our assessment to the large-scale NDC datasets to establish the practicality of the proposed models in such contexts. Our experimental findings and rigorous statistical analyses affirm the superior generalization prowess of the proposed GBTSVM and LS-GBTSVM models compared to the baseline models. The source code of the proposed GBTSVM and LS-GBTSVM models are available at \url{https://github.com/mtanveer1/GBTSVM}.
\end{abstract}
\section{Introduction}
\IEEEPARstart{S}{upport} vector machines (SVMs) \cite{cortes1995support} are advanced kernel-based machine learning models that maximize the margin between two classes in a classification problem, aiming to find the optimal hyperplane between two parallel supporting hyperplanes. SVM has proven its remarkable utility across diverse domains such as web mining \cite{bollegala2010web}, Alzheimer's disease diagnosis \cite{richhariya2020diagnosis}, and so on. SVM integrates the structural risk minimization (SRM) principle within its optimization framework, thereby enhancing its generalization capabilities by minimizing an upper bound of the generalization error. SVM solves one large quadratic programming problem (QPP), resulting in escalated computational complexity, which renders it less suitable for large-scale datasets. To alleviate the computational complexity of SVM, \citet{mangasarian2005multisurface} and \citet{khemchandani2007twin} proposed the generalized eigenvalue proximal SVM (GEPSVM) and twin SVM (TSVM), respectively. On the one hand, GEPSVM solves the generalized eigenvalue problem rather than dealing with a large QPP. On the other hand, TSVM solves two smaller-sized QPPs instead of a single large QPP, making TSVM four times faster than the standard SVM \cite{khemchandani2007twin, tanveer2022comprehensive} and firmly establishing TSVM as a standout and superior choice. TSVM generates a pair of non-parallel hyperplanes, with each hyperplane deliberately situated in close proximity to the data points belonging to one class while ensuring a minimum separation distance of at least one unit from the data points of the other class. However, TSVM encountered two notable challenges: the necessity for matrix inversions and the absence of the SRM principle in its formulation, which presented significant obstacles to its efficacy. \citet{shao2011improvements} and \citet{tian2014improved} proposed twin bounded SVM (TBSVM) and improved TSVM (ITSVM), respectively. Both the models incorporated a regularization term in their formulation, allowing the principle of SRM to be employed. In recent years, several variants of TSVM have been proposed for both small and large datasets such as pinball TSVMs \cite{tanveer2019general, xu2016novel}, least square TSVM (LTSVM) \cite{kumar2009least}, robust energy-based LTSVM (RELTSVM) \cite{tanveer2016robust} and many more.

While TSVM and its variants effectively address the computational complexity challenges posed by SVM, both SVM and TSVM encounter difficulties when faced with noisy or outlier-laden datasets. In scenarios where noise perturbs support vectors, the SVM's capacity to discern an optimal hyperplane is impeded, resulting in suboptimal outcomes. To address this issue, fuzzy SVM (FSVM) was introduced in \cite{lin2002fuzzy}, employing a degree of membership function for each training sample. Furthermore, intuitionistic fuzzy TSVM (IFTSVM) \cite{rezvani2019intuitionistic} is proposed and offers a substantial reduction in the adverse effects of noise and outliers by leveraging a set of membership and non-membership values to each training sample. Several other variants, such as \cite{liang2022intuitionistic}, have been proposed to mitigate the detrimental impact of noise and outliers; they often come at the price of increased computational complexity.

Human cognition follows the principle of ``large scope first", and the visual system is especially attuned to perceiving global topological features, processing information from larger to smaller scales or from coarse-grained to fine-grained \cite{xia2022efficient}. Unlike human cognition, most existing machine learning classifiers take inputs in the form of pixels or points because their training process consistently commences at the smallest level of granularity. This lacks the scalability and efficiency of the model. Being inspired by human cognition, \citet{xia2019granular} introduced granular ball classifiers, which utilize hyper-balls to divide the dataset into various sizes of granular balls \cite{xia2020fast}. Within the framework of granular computing, it is observed that larger granularity sizes present itself as a scalable, efficient, and robust approach that closely resembles the cognitive processes of the human brain \cite{xia2019granular, xia2021granular}. Moreover, a transition towards larger granularity entails an increased risk of reduced detail and compromised ACC. Conversely, opting for finer granularity allows for heightened detail focus, albeit potentially at the cost of efficiency and resilience in noisy environments. Hence, achieving a judicious balance in granular size is of paramount importance. Over the past decades, scholars have continually engaged in research \cite{zhang2021double,pedrycz2007development}, focusing on breaking down large volumes of information and knowledge into different granularities according to certain tasks. Choosing different granularities based on specific situations can enhance the effectiveness of multi-granularity learning approaches and efficiently tackle real-world challenges \cite{pedrycz1984identification, song2021integrating}. 

Again, to address noise and outliers, an efficient granular ball SVM (GBSVM) \cite{xia2022gbsvm} is proposed by integrating the granular ball concept into SVM. GBSVM takes inputs as granular balls generated from the dataset rather than the individual data points. GBSVM exhibits better resilience in contrast to standard SVM. Motivated by the merits of employing a granular approach for addressing noise and outliers alongside the efficiency exhibited by TSVM, the prospect of integrating these two concepts appears not only intriguing but also promising. Therefore, we propose a granular ball twin support vector machine (GBTSVM). GBTSVM takes granular balls as inputs, resulting in improved computational efficiency and a heightened ability to withstand noise and outliers. We further propose large-scale GBTSVM (LS-GBTSVM), which efficiently manages large-scale data by incorporating a regularization term in its primal form, eliminating the need for matrix inversions and reducing overfitting risk.

\vspace{0.2cm}
The main highlights of this study are outlined as follows:
\begin{itemize}
    \item We propose a novel granular ball twin support vector machine (GBTSVM). The proposed GBTSVM utilizes granular balls as inputs for classifier construction, offering enhanced robustness, resilience to resampling, and computational efficiency compared to SVM and GBSVM.    
    \item We propose a novel large-scale granular ball twin support vector machine (LS-GBTSVM) to overcome the challenges associated with handling large-scale datasets encountered by the proposed GBTSVM. The LS-GBTSVM aims to reduce the structural risk inherent in its formulation by incorporating the SRM principle through the addition of a regularization term in its primal formulation. The proposed LS-GBTSVM exhibits efficiency, scalability for large datasets, proficient handling of overfitting, robustness, and improved generalization performance compared to standard SVM, TSVM, and GBSVM. 
    \item We provide rigorous mathematical frameworks for both GBTSVM and LS-GBTSVM, covering linear and non-linear kernel spaces. 
    Further, we derive the violation tolerance upper bound (VTUB) for the proposed GBTSVM.
    \item We performed experiments on $36$ real-world UCI \cite{dua2017uci} and KEEL \cite{derrac2015keel} datasets; and 10k to 5m large-scale NDC\cite{ndc} datasets. Numerical experiments and statistical analyses confirm the superiority of the proposed GBTSVM and LS-GBTSVM models compared to the baseline models. 
    \item The proposed GBTSVM and LS-GBTSVM models are subjected to rigorous testing by adding noise to datasets. Testing under noisy conditions shows that the proposed GBTSVM and LS-GBTSVM models are insensitive to noise and stable to resampling.
\end{itemize}
The remaining structure of the paper is organized as follows. We discuss related works in Section \ref{Related Works}. In Section \ref{Granular-ball Twin Support Vector Machine (GBTSVM)} and \ref{large}, we derive the mathematical formulation of the proposed GBTSVM and LS-GBTSVM models, respectively. The time complexity of the proposed GBTSVM model is given in Section \ref{Time Complexity}. We discuss the violation tolerance upper bound of the proposed GBTSVM model in Section \ref{Violation tolerance upper bound}. In Section \ref{experimental result}, the discussion of the experimental results is made.  Finally, the conclusions and potential future research directions are given in Section \ref{conclusions}.
\section{Related Works}
\label{Related Works}
In this section, we go through the granular ball computing method. The mathematical formulation of GBSVM and TSVM is discussed in Section S.I of the supplementary material. 

\subsection{Notations}
Let $D=\{(x_i, t_i), i=1,2, \ldots, n\},$ be the traning dataset, where $t_i \in \{+1,-1\} $  represents the label of $x_i \in \mathbb{R}^{1 \times m}$. Let us consider the input matrices $A \in \mathbb{R}^{n_1 \times m}$ and $B \in \mathbb{R}^{n_2 \times m}$, where $n_1$ ($n_2$) is the number of data samples belonging to $+1$ ($-1$) class such that the total number of data samples is $n = n_1 + n_2$. The set of generated granular balls is denoted as $S=\{GB_i, \hspace{0.2cm} i=1,2, \ldots, p\} = \{((c_i, r_i), y_i),\hspace{0.2cm} i=1,2, \ldots, p\},$ where $c_i$ signifies the center, $r_i$ indicates the radius and $y_i$ is the label of the $i^{th}$ granular ball. Matrices $C_1 \in \mathbb{R}^{p_1 \times m}$ ($C_2 \in \mathbb{R}^{p_2 \times m}$) represent the centers associated with the class of $+1$ ($-1$), where $p_1+p_2 = p$. The center is computed as $c_i^\pm=\frac{1}{l_i^{\pm}}\sum_{s=1}^{l_i^\pm} x_{s}^\pm$, where $c_i^\pm$, $x_s^\pm$ $(s=1,2, \hdots, l_i^\pm)$, and $l_i^\pm$ denotes the center, data samples, and the total number of data samples inside the granular ball $GB_i$, respectively. $R_1 \in \mathbb{R}^{p_1 \times 1}$ ($R_2 \in \mathbb{R}^{p_2 \times 1}$) denote the vector encompassing the radius of the generated granular balls for $+1$ ($-1$) class and is calculated as $r_i^\pm=\frac{1}{l_i^\pm}\sum_{s=1}^{l_i^\pm}\lvert x_s^\pm - c_i^\pm\rvert$, where $r_i^\pm$ represents the radius of each granular balls of $+1$ ($-1$) class.

\subsection{Granular Ball Computing}
Granular ball computing is a substantial data processing technique introduced by \citet{xia2019granular} to address the scalability challenges associated with high-dimensional data. A notable advantage of this approach is that it requires only the radius and center to characterize a granular ball, rather than all the data points (samples) contained within that granular ball. Let ``$c$" represent the center of gravity of the data points contained inside a granular ball. The radius ``$r$" of the granular ball depicts the average distance between the center $c$ and the remaining samples contained within a granular ball. The granular ball's label is determined by selecting the label of the samples that appear most frequently among the samples enclosed within the granular ball. To perform a quantitative assessment of the divided granular ball's mass, the concept of a ``purity threshold" is introduced. This threshold refers to the percentage of samples within a granular ball that shares the same label, specifically the majority labels. 

The idea of the granular ball generation is illustrated in Figure \ref{Process of the granular-ball generation} and the optimization objective of generation of granular balls $GB_j$ $(j=1,2,\hdots,p)$ can be formulated as follows:
\begin{align}
\label{eq:1}
    & min \hspace{0.2cm} \vartheta_1 \times \frac{n}{\sum_{i=1}^{p}\lvert GB_i \rvert} + \vartheta_2 \times p \nonumber \\
    s.t. & \hspace{0.2cm} purity(GB_j) \geq T, \hspace{0.2cm} j=1,2, \ldots, p,
\end{align}
where $\vartheta_1$ and $\vartheta_2$ represent the weight coefficients, and $T$ denotes the purity threshold. The iterative process to generate the granular balls is demonstrated in Figure \ref{gb-generation}. At the initial stage, the entire dataset can be conceptualized as a single granular ball, as illustrated in Figure \ref{fig:1a}. The granular ball will undergo division, increasing its purity, as depicted in Figures \ref{fig:1b}-\ref{fig:1d}. Once the purity level of all the granular balls satisfies the specified threshold, the algorithm reaches convergence, as demonstrated in Figure \ref{fig:1e}. The granular balls obtained are visualized in Figure \ref{fig:1f}, illustrating the effectiveness of granular ball computing in capturing the underlying data distribution. 
\begin{figure}[ht]
    \centering       
\tikzstyle{decision} = [diamond, draw, fill=blue!20, 
    text width=7.0em, text badly centered, node distance=1.2cm, inner sep=0pt]
\tikzstyle{block} = [rectangle, draw, fill=blue!20, 
    text width=7em, text centered, rounded corners, minimum height=3em]
\tikzstyle{line} = [draw, -latex']
\tikzstyle{cloud} = [draw, ellipse,fill=red!20, node distance=2cm,
    minimum height=4em,text width=11em]
\begin{tikzpicture}[node distance = 2.5cm, auto]
    \node [cloud] (init) {Consider the whole dataset as a granular ball};
    \node [block, below of=init] (identify) {Split the current granular ball into k sub-balls using k-means};
    \node [block, left of=identify, node distance=4.0cm] (update) {Continue to split those balls whose qualities can not meet the requirements};
    \node [decision, below of=identify,node distance=3.5cm] (decide) {The quality of each granular ball meets the requirements};
    \node [block, below of=decide, node distance=3.2cm] (stop) {Stop splitting and converge};
    \path [line] (init) -- (identify);
    \path [line] (identify) -- (decide);
    \path [line] (decide) -| node [near start] {No} (update);
    \path [line] (update) -- (identify);
    \path [line] (decide) -- node {Yes}(stop);
\end{tikzpicture}
\caption{Process of the granular ball generation}
    \label{Process of the granular-ball generation}
\end{figure}
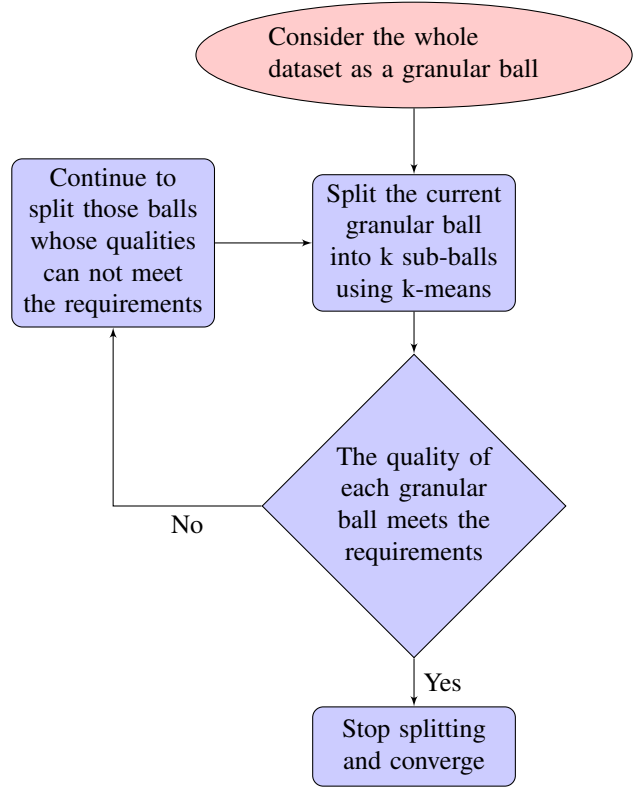
\begin{figure*}[ht]
\begin{minipage}{.333\linewidth}
\centering
\subfloat[The original dataset.]{\label{fig:1a}\includegraphics[scale=0.45]{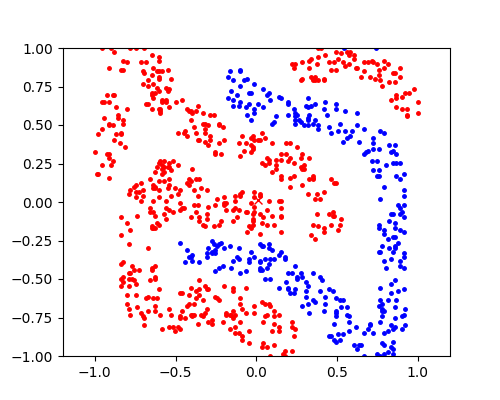}}
\end{minipage}
\begin{minipage}{.333\linewidth}
\centering
\subfloat[Generated granular balls in the first iteration]{\label{fig:1b}\includegraphics[scale=0.45]{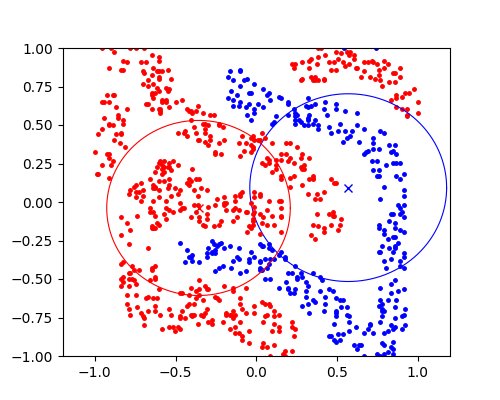}}
\end{minipage}
\begin{minipage}{.333\linewidth}
\centering
\subfloat[Generated granular balls in the second iteration]{\label{fig:1c}\includegraphics[scale=0.45]{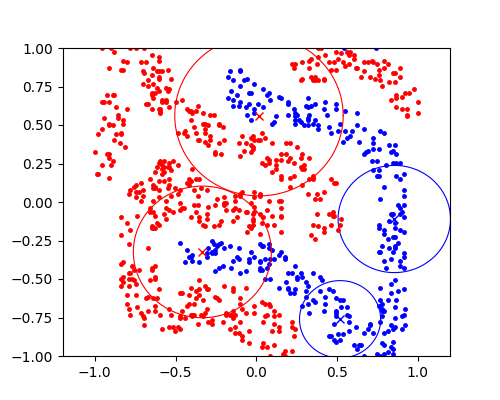}}
\end{minipage}
\par\medskip
\begin{minipage}{.333\linewidth}
\centering
\subfloat[Generated granular balls in the middle iteration]{\label{fig:1d}\includegraphics[scale=0.45]{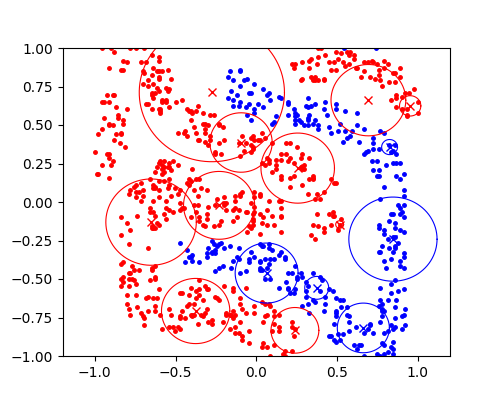}}
\end{minipage}
\begin{minipage}{.333\linewidth}
\centering
\subfloat[Results after stop splitting]{\label{fig:1e}\includegraphics[scale=0.45]{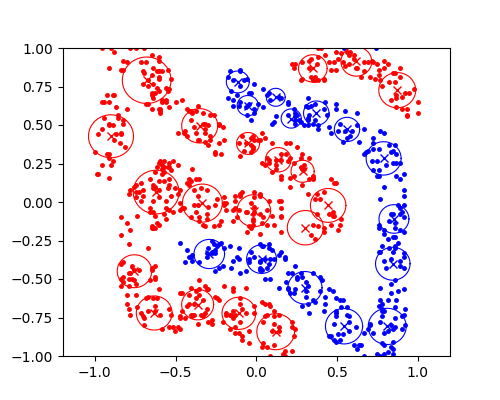}}
\end{minipage}
\begin{minipage}{.333\linewidth}
\centering
\subfloat[Extracted granular balls]{\label{fig:1f}\includegraphics[scale=0.45]{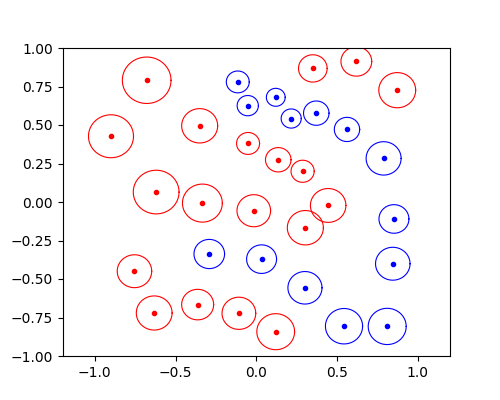}}
\end{minipage}
\caption{The existing method for generation of granular ball and splitting on the ``fourclass" dataset. Both the red granular ball and the red points bear the label ``$+1$”, while both the blue granular balls and the blue points are labeled as ``$-1$”.}
\label{gb-generation}
\end{figure*}

\section{The proposed Granular ball Twin Support Vector Machine (GBTSVM)}
\label{Granular-ball Twin Support Vector Machine (GBTSVM)}
In this section, we provide a detailed mathematical formulation of the proposed GBTSVM model tailored for linear case. Utilizing the granular ball as input offers two primary advantages. Firstly, it reduces the number of input samples, thereby enhancing training efficiency. Secondly, the encapsulation of data points within granular balls enhances the model's resilience to noise and outliers, as the influence of individual noisy points is mitigated within the localized context of a granular ball, thus leading to more reliable and accurate results.
The formulation of linear GBTSVM is given as follows:
\begin{figure}
    \centering
    \includegraphics[width=0.35\textwidth,height=4.5cm]{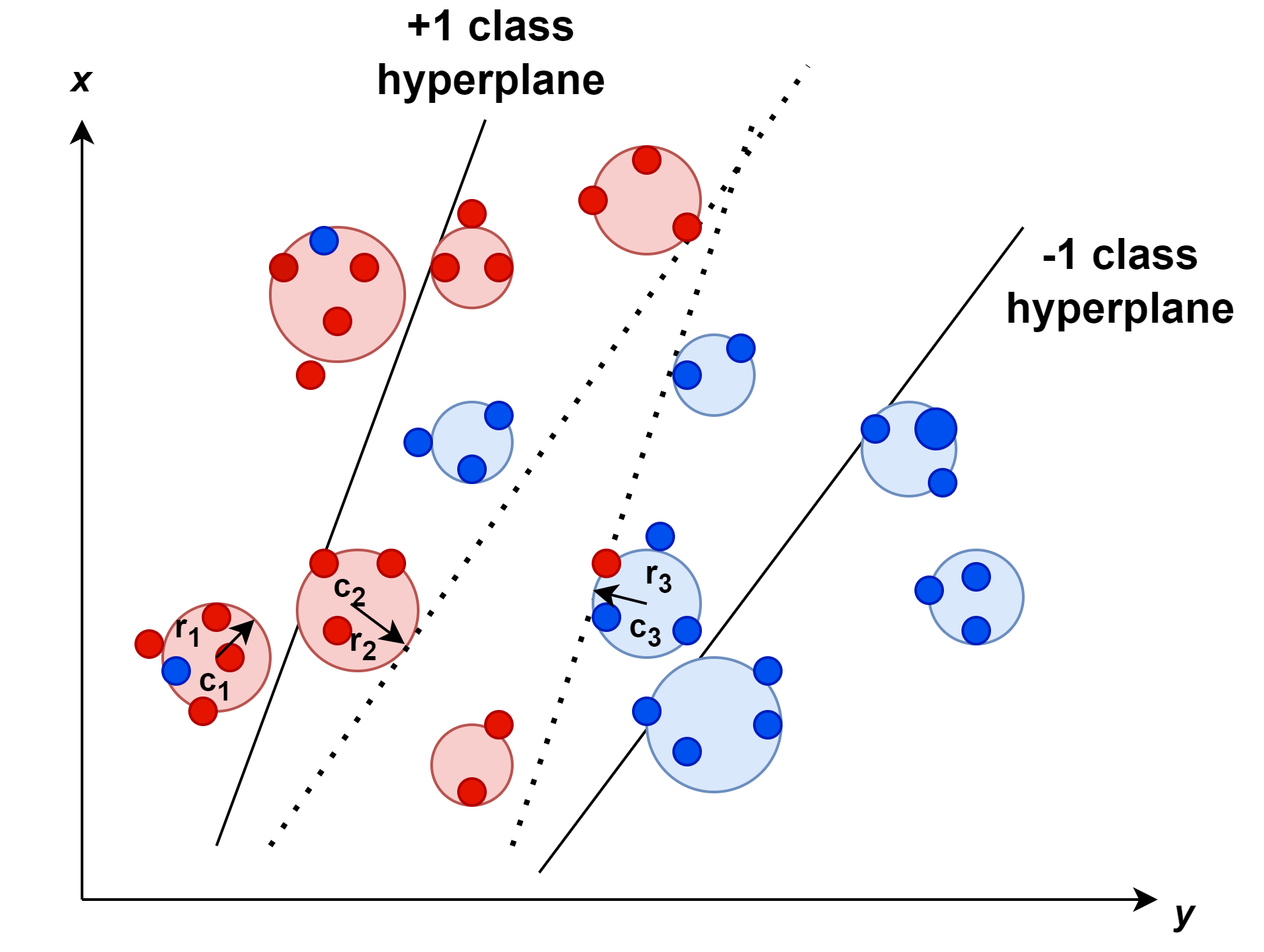}
    \caption{Geometrical depiction of GBTSVM: Samples in the red color symbolize instances of the $+1$ class, while those in the blue color signify instances of the $-1$ class. Additionally, $c_i$ and $r_i$ denote the center and radius of individual granular balls.}
    \label{Geometric representation of GBTSVM}
\end{figure}
\begin{align}
\label{eq:7}
      & \underset{w_1,b_1}{min} \hspace{0.2cm} \frac{1}{2} \|C_1w_1 + e_1b_1\|^2 + d_1e_2^T\xi_2 \nonumber \\
     & s.t. \hspace{0.2cm} -(C_2w_1 + e_2b_1) + \xi_2 \geq e_2+R_2, \nonumber \\
     & \hspace{1.2cm} \xi_2 \geq 0,
\end{align}
and 
\begin{align}
\label{eq:8}
      & \underset{w_2,b_2}{min} \hspace{0.2cm} \frac{1}{2} \|C_2w_2 + e_2b_2\|^2 + d_2e_1^T\xi_1 \nonumber \\
     & s.t. \hspace{0.2cm} (C_1w_2 + e_1b_2) + \xi_1 \geq e_1+R_1, \nonumber \\
     & \hspace{0.8cm} \xi_1 \geq 0,
\end{align}
where $d_1, d_2 ~(> 0)$ are tunable parameters, $\xi_1$ and $\xi_2$ represents the slack variables, and $e_1$ and $e_2$ represents the vector of ones with appropriate dimensions. The first term in the objective function of \eqref{eq:7} is the sum of squared distances between the hyperplane and the centers of granular balls associated with data points belonging to the $+1$ class. Therefore, minimizing it tends to keep the hyperplane close to the granular balls of one class. The constraints of the optimization problem \eqref{eq:7} mandate that the hyperplane be at least a unit distance from the tangent of the granular balls of another class. The tolerance parameter, $\xi_2$ introduced in \eqref{eq:7} to quantify discrepancies arising when the hyperplane deviates from a minimum unit distance threshold. The second term of the objective function of \eqref{eq:7} minimizes the sum of error variables, thus attempting to minimize misclassification due to points belonging to $-1$ class. Similarly, all the components of \eqref{eq:8} are defined for $-1$ class. Figure \ref{Geometric representation of GBTSVM} displays the geometric representation of the proposed GBTSVM model. 
The Lagrangian corresponding to the problem \eqref{eq:7} is given by
\begin{align}
\label{eq:9}
    L&=\frac{1}{2}\|C_1w_1 + e_1b_1\|^2 + d_1e_2^T\xi_2 - \alpha^T(-(C_2w_1 + e_2b_1) + \nonumber\\& ~~~~\xi_2 - e_2 -R_2) - \beta^T\xi_2,
\end{align}
where $\alpha \in \mathbb{R}^{p_2 \times 1}$ and $\beta \in \mathbb{R}^{p_1 \times 1}$ are the vectors of Lagrangian multipliers.\\
Using the Karush-Kuhn-Tucker (K.K.T.) conditions, we have
\begin{align}
    &C_1^T(C_1w_1 + e_1b_1) + C_2^T\alpha = 0, \label{eq:10}\\
    &e_1^T(C_1w_1 + e_1b_1) + e_2^T\alpha = 0, \label{eq:11}\\
    &e_2d_1 - \alpha - \beta =0, \label{eq:12}\\
    &-(C_2w_1 + e_2b_1) + \xi_2 \geq e_2+R_2, \hspace{0.3cm} \xi_2 \geq 0, \\
    &\alpha^T(-(C_2w_1 + e_2b_1) + \xi_2 - e_2 -R_2) = 0, \hspace{0.3cm} \beta^T\xi_2 = 0, \label{EQQ:9} \\
    &\alpha \geq 0, \hspace{0.3cm} \beta \geq 0.
\end{align}
Combining \eqref{eq:10} and \eqref{eq:11} leads to
\begin{align}
\label{eq:16}
    \binom{C_1^T}{e_1^T} \left(C_1 \hspace{0.2cm} e_1\right)\binom{w_1}{b_1} + \binom{C_2^T}{e_2^T}\alpha = 0.
\end{align}
Let $H=\left(C_1 \hspace{0.2cm} e_1\right),$ $G=\left(C_2 \hspace{0.2cm} e_2\right)$ and $u_1=\binom{w_1}{b_1}$
then, \eqref{eq:16} can be reformulated as:
\begin{align}
    &H^THu + G^T\alpha = 0, \nonumber \\
    i.e., \hspace{0.2cm} & u_1 = -(H^TH)^{-1}G^T\alpha \label{eq:17}.
\end{align}
Computing the inverse of $H^TH$ presents a formidable challenge. However, this difficulty can be effectively addressed by incorporating a regularization term denoted as $\delta I$ in \eqref{eq:17}, where $I$ denotes an identity matrix of appropriate dimensions. Thus,
\begin{align}
\label{eq:18}
    u_1 = -(H^TH+\delta I)^{-1}G^T\alpha.
\end{align}
Using equation \eqref{eq:18} and the above K.K.T. conditions, we can obtain the dual of \eqref{eq:7} as follows:
\begin{align}
\label{eq:19}
     \underset{\alpha}{max} & \hspace{0.2cm} \alpha^T(e_2 + R_2) - \frac{1}{2} \alpha^T G(H^TH + \delta I)^{-1}G^T \alpha \nonumber \\
     s.t. & \hspace{0.2cm} 0 \leq \alpha \leq d_1e_2.
\end{align}
Likewise, the Wolfe dual for \eqref{eq:8} can be obtained as:
\begin{align}
\label{eq:20}
     \underset{\gamma}{max} & \hspace{0.2cm} \gamma^T(e_1 + R_1) - \frac{1}{2} \gamma^T H(G^TG + \delta I)^{-1}H^T \gamma \nonumber \\
     s.t. & \hspace{0.2cm} 0 \leq \gamma \leq d_2e_1.
\end{align}
Analogously, $u_2=\binom{w_2}{b_2}$ corresponding to the $-1$ class can be calculated by the subsequent equation:
\begin{align}
\label{eq:21}
    u_2=(G^TG+\delta I)^{-1}H^T\gamma.
\end{align}
Once the optimal values of $u_1$ and $u_2$ are calculated. The classification of a new input data point $x$ into either the $1$ ($+1$ class) or $2$ ($-1$ class) class can be determined as follows:
\begin{align}
\label{eq:22}
    \text{class}(x) =  \underset{i \in \{1, 2\}}{\arg\min} \frac{\lvert w_i^Tx + b_i\rvert}{\|w_i\|}.
\end{align}
The formulation of the proposed GBTSVM model for the non-linear case is given in subsection S.II.B in the Supplementary material.
\section{Large Scale Granular ball Twin Support Vector Machine (LS-GBTSVM)}
\label{large}
The GBTSVM formulation exhibits certain limitations: $(i)$ performing a matrix inversion computation within the Wolfe dual formulation becomes costly when dealing with large datasets. $(ii)$ GBTSVM does not incorporate the SRM principle in its formulation, which leads to an elevated risk of overfitting. Our goal is to tailor our model for large-scale datasets by reformulating the optimization problem of GBTSVM to remove the need for computing extensive matrix inverses in equations \eqref{eq:18} - \eqref{eq:21}. To achieve this, we introduce a regularization term similar to TBSVM \cite{shao2011improvements} and incorporate an additional equality constraint into the optimization problem. The optimization problem of linear LS-GBTSVM are given as:
\begin{align}
\label{eq:23}
       \underset{w_1,b_1,\eta_1, \xi_2}{min} & \hspace{0.2cm} \frac{1}{2} d_3(\|w_1\|^2+b_1^2) + \frac{1}{2}\eta_1^T\eta_1 + d_1e_2^T\xi_2 \nonumber \\
       s.t. & \hspace{0.2cm} C_1w_1+e_1b_1 = \eta_1, \nonumber \\
     & \hspace{0.2cm} -(C_2w_1 + e_2b_1) + \xi_2 \geq e_2+R_2, \nonumber \\
     & \hspace{0.2cm} \xi_2 \geq 0,
\end{align}
and
\begin{align}
\label{eq:24}
       \underset{w_2,b_2,\eta_2,\xi_1}{min} & \hspace{0.2cm} \frac{1}{2}d_4 (\|w_2\|^2 + b_2^2) + \frac{1}{2}\eta_2^T\eta_2 + d_2e_1^T\xi_1 \nonumber \\
       s.t. & \hspace{0.2cm} C_2w_2 + e_2b_2 = \eta_2, \nonumber \\
     &\hspace{0.2cm} (C_1w_2 + e_1b_2) + \xi_1 \geq e_1+R_1, \nonumber \\
     & \hspace{0.2cm} \xi_1 \geq 0,
\end{align}
where $d_1, d_2, d_3, d_4 \hspace{0.1cm} (>0)$ are tunable parameters, $\xi_1, \xi_2$ are the slack variables and $e_1, e_2$ are the vector of ones with appropriate dimensions. The optimization problem \eqref{eq:23} and \eqref{eq:24}, incorporates regularization terms $\frac{1}{2} d_3(\|w_1\|^2+b_1^2)$ and $\frac{1}{2}d_4 (\|w_2\|^2 + b_2^2)$. The inclusion of the regularization terms in the LS-GBTSVM model contributes to SRM. The modified Lagrangian is designed in a way that bypasses the need to calculate the large matrix inverse. The dual of the \eqref{eq:23} is obtained as:
\begin{align}
\label{eq:35}
    \underset{\alpha_1, \beta_1}{min} \hspace{0.2cm} & \frac{1}{2}\begin{pmatrix}
\alpha_1^T & \beta_1^T
\end{pmatrix} \tilde{P} \begin{pmatrix}
\alpha_1 \\
\beta_1 
\end{pmatrix} - d_3 \beta_1^T(e_2 + R_2) \nonumber \\
s.t. \hspace{0.2cm} & 0 \leq \beta_1 \leq d_1 e_2, \nonumber \\
\text{where} \hspace{0.2cm} & \tilde{P} = \begin{pmatrix}
C_1C_1^T + d_3 I & C_1C_2^T\\
C_2C_1^T  & C_2C_2^T
\end{pmatrix} + E.
\end{align}
Here, matrix $E$ consists entirely of ones, while $I$ denotes the identity matrix of the suitable dimension.

The corresponding dual problems \eqref{eq:24} can be obtained as follows:
\begin{align}
\label{eq:37}
    \underset{\alpha_2, \beta_2}{min} \hspace{0.2cm} & \frac{1}{2}\begin{pmatrix}
\alpha_2^T & \beta_2^T
\end{pmatrix} \tilde{Q} \begin{pmatrix}
\alpha_2 \\
\beta_2 
\end{pmatrix} - d_4 \beta_2^T(e_1 + R_1) \nonumber \\
s.t. \hspace{0.2cm} & 0 \leq \beta_2 \leq d_2 e_1, \nonumber \\
\text{where} \hspace{0.2cm} & \tilde{Q} = \begin{pmatrix}
C_2C_2^T + d_4 I & C_2C_1^T\\
C_1C_2^T  & C_1C_1^T
\end{pmatrix} + E.
\end{align}
The optimal values of $\alpha_1, \alpha_2, \beta_1$ and $\beta_2$, are used to determine using equations of non-parallel hyperplanes $w_1^Tx + b_1 = 0$ and $w_2^Tx + b_2 = 0$ and is given as:
\begin{align}
    & w_1^*=-\frac{1}{d_3}(C_1^T\alpha_1 + C_2^T\beta_1),  \hspace{0.4cm}   b_1^*=-\frac{1}{d_3}(e_1^T\alpha_1 + e_2^T\beta_1), \nonumber \\
   & w_2^*=\frac{1}{d_4}(C_2^T\alpha_2 + C_1^T\beta_2),  \hspace{0.7cm}  b_2^*=\frac{1}{d_4}(e_1^T\alpha_2 + e_2^T\beta_2), \nonumber
\end{align}
here $*$ represents the optimal values for the corresponding entities. Once the optimal planes are determined, the labeling of the test sample $x$ into either the $1$ ($+1$ class) or $2$ ($-1$ class) class can be determined as follows:
\begin{align}
    \text{class}(x) =  \underset{i \in \{1, 2\} }{\arg\min} \frac{\lvert w_i^{*T}x + b_i^*\rvert}{\|w_i^*\|}.
\end{align}
The detailed mathematical formulation of the proposed LS-GBTSVM model for the linear and non-linear case are given in subsections S.II.A and S.II.C of the Supplementary material.
\section{Time Complexity and Algorithm}
\label{Time Complexity}
Let $D$ represent the training dataset having $n$ samples. The time complexity of the $k$-means algorithm is $\mathcal{O}(nkt)$ \cite{zhou2009novel}, where the parameter $k$ represents the number of clusters, and $t$ denotes the number of iterations. Also, the time complexity of standard TSVM \cite{khemchandani2007twin} is $\mathcal{O}(\frac{n^3}{4})$. We consider the training dataset $D$ as the initial granular ball ($GB$) set. The set $GB$ is divided into two granular balls using the $2$-means clustering. During the initial phase of splitting, the computational complexity is $\mathcal{O}(2n)$. In the second phase, the two generated granular balls are further divided into four granular balls (if both the granular balls are impure),  with a maximum computational complexity of $\mathcal{O}(2n)$, and so on. If there is a total $t$ number of iterations, then the overall computational complexity of generating granular balls is (or less than) $\mathcal{O}(2nt)$. The complexity of the proposed GBTSVM model is (or less than) $\mathcal{O}(\frac{p^3}{4}) + \mathcal{O}(2nt)$, where $p$ represents the total number of generated granular balls. Also, $\mathcal{O}(p^3) \ll \mathcal{O}(n^3)$ as $p \ll n$. Therefore, we compare our proposed GBTSVM model with the standard TSVM model given as follows: $\mathcal{O}(\frac{p^3}{4}) + \mathcal{O}(n) \ll \mathcal{O}(\frac{n^3}{4})$. The time complexity of the proposed GBTSVM model is much lower than TSVM, and hence, the proposed GBTSVM model is more efficient than the baseline models. The algorithm of the proposed GBTSVM is briefly described in Algorithm \ref{Algorithm for GBTSVM.}.
\begin{algorithm}
\caption{Algorithm of GBTSVM model.}
\label{Algorithm for GBTSVM.}
\textbf{Input:} Traning dataset $D$, and the threshold purity $T$. \\
\textbf{Output:} GBTSVM classifier.
\begin{algorithmic}[1]
\STATE Initialize the entire dataset $D$ as a granular ball $GB$ and set of granular balls $S$ to be empty set, i.e., $GB=D$ and $S=\{\hspace{0.1cm}\}$. 
\STATE $Dummy =\{GB\}$.
\STATE $for$ $i = 1:\lvert Dummy \rvert$ 
\STATE $if$ $pur(GB_i) < T$ 
\STATE Split $GB_i$ into $GB_{i1}$ and $GB_{i2}$, using $2$-means clustering algorithm. 
\STATE $Dummy \leftarrow GB_{i1}, \hspace{0.05cm} GB_{i2}$. 
\STATE $end$ $if$. 
\STATE $else$ $pur(GB_i)\geq T$ 
\STATE Calculate the center $c_i = \frac{1}{n_i} \sum_{j=1}^{n_i} x_j$ of $GB_i$, where $x_j \in GB_i$, $j=1, 2, \ldots, n_i$, and $n_i$ is the number of training sample in $GB_i$.  
\STATE Calculate the radius $ r_i = \frac{1}{n_i}\sum_{j=1}^{n_i}\left | x_j - c_i \right | $ of $GB_i$.
\STATE Calculate the label $y_i$ of $GB_i$, where $y_i$ is assigned the label of majority class samples within $GB_i$.
\STATE Put $GB_i = \{((c_i, r_i),y_i)\}$ in $S$. 
\STATE $end$ $else.$ 
\STATE $end$ $for.$  
\STATE $if$ $Dummy \neq \{\hspace{0.1cm}\}$ 
\STATE Go to step 3 (for further splitting). 
\STATE $end$ $if.$ 
\STATE Set $S=\{GB_i, \hspace{0.2cm} i=1,2, \ldots, p\} = \{((c_i, r_i), y_i),\hspace{0.2cm} i=1,2, \ldots, p\},$ where $c_i$ signifies the center, $r_i$ indicates the radius, $y_i$ is the label of $GB_i$ and $p$ is the number of generated granular balls. 
\STATE Solve \eqref{eq:19} and \eqref{eq:20} to obtain $\alpha$ and $\gamma$, where $\alpha$ and $\gamma$ are the Lagrange multipliers. 
\STATE Using \eqref{eq:18} and \eqref{eq:21}, find optimal values of $w_1$, $b_1$, $w_2$, and $b_2$. 
\STATE Testing sample is classified into $+1$ or $-1$ class using\eqref{eq:22}.
\end{algorithmic}
\end{algorithm}
\section{Violation tolerance upper bound}
\label{Violation tolerance upper bound}
In this section, we discuss the violation tolerance upper bound (VTUB) for the proposed GBTSVM. For GBTSVM, the slack variables can be interpreted as the tolerance for violations from the boundary hyperplane. A naive but reasonable perspective is that the closer two samples are, the more similar their violation tolerances will be \cite{qi2021elastic}. 

Let $S$ be a set of \(p = p_1 + p_2\) granular balls, where $p_1$ and $p_2$ represent the number of $+1$ and $-1$ granular balls, respectively. Let $H=\left(C_1, e_1\right)$ and $G=\left(C_2, e_2\right)$, where $C_1$ and $C_2$ are matrices of $+1$ and $-1$ GB centers, respectively. First, we discuss some preliminary lemmas to establish the proof of our theorems.
\begin{lemma}
    (Woodbury Formula \cite{press1988numerical}) Let \(A \in \mathbb{R}^{m_1 \times m_1}\), \(U \in \mathbb{R}^{m_1 \times m_2}\), \(C \in \mathbb{R}^{m_2 \times m_2}\), and \(V \in \mathbb{R}^{m_2 \times m_1}\). Given the existence of \(A^{-1}\), \(C^{-1}\), and \((A + UCV)^{-1}\), the following identity holds:
     \begin{align}
         (A + UCV)^{-1} = A^{-1} - A^{-1} U (C^{-1} + V A^{-1} U)^{-1} V A^{-1}.
     \end{align}
\end{lemma}
\begin{lemma}
 (Weyl Theorem \cite{horn2012matrix}) Let \(A\) and \(B\) be \(m_1 \times m_1\) Hermite matrices. Then, we have:
\begin{align}
    \lambda_i(A) + \lambda_{m_1}(B) \leq \lambda_i(A + B) \leq & \lambda_i(A) + \lambda_1(B), \nonumber \\
    & \quad i = 1, \ldots, m_1,
\end{align}
where \(\lambda_1(\cdot) \geq \lambda_2(\cdot) \geq \lambda_3(\cdot) \geq \ldots \geq \lambda_{m_1}(\cdot) \).
\end{lemma}
\begin{lemma}
\label{lemma3}
   Let $A \in \mathbb{R}^{m \times n}$ any matrix and $M \in \mathbb{R}^{n \times n}$ is a positive definite matrix. Then the matrix \( AMA^T\) is positive semi-definite.
\end{lemma}
\begin{proof}
    Let $x \in \mathbb{R}^{1 \times m} $ be a non-zero row vector. Then
    \begin{align}
        & xAMA^Tx^T =  yMy^T,   ~~~~~~~~~  \text{for}~~y=xA \nonumber \\
       & \text{If} ~~~~ y \neq 0, ~~~\implies ~~  yMy^T \geq 0,  ~~( \because M ~ \text{is positive definite}) \nonumber \\
       & \text{If} ~~~~ y = 0, ~~~\implies ~~  yMy^T = 0.  \nonumber 
    \end{align}
    Hence, $AMA^T$ is a positive semi-definite matrix.
\end{proof}
\begin{theorem}
    \label{TH1}
    If $(\hat{w}_1^T, \hat{b}_1^T, \hat{\xi}_2^T)$ is the optimal solution of the GBTSVM \eqref{eq:19}. Then, for any two positive GBs, \(((c_i, r_i), y_i)\) and \(((c_j, r_j), y_j)\), the estimations of the corresponding slack variables \(\xi_2^i\) and \(\xi_2^j\) satisfy
    \begin{align}
    \label{THHHH1}
        \lvert \hat{\xi}_2^i - \hat{\xi}_2^j \rvert \leq  \Delta^2(\delta + \tau_1)(\delta + \tau_1 + \tau_2)\sqrt{\kappa}\|G\|_F\cdot d_{ij}^3,
    \end{align}
    where \(d_{ij} = \|c_i - c_j\|\) is the distance between GB centers, \(c_i\) and \(c_j\), \(\|G\|_F\) is the Frobenius norm of the matrix \(G\), $\Delta$ $(>0)$ and $\delta$ $(>0)$ are very small real numbers, $\kappa$ is a positive real number, \(\tau_1\) and \(\tau_2\) correspond to the largest eigenvalues of \( H^TH\) and $G^TG$, respectively. 
\end{theorem}
\begin{proof}
   By introducing Lagrangian multipliers \(q\) and \(\sigma\) in QPP \eqref{eq:19}, we can derive its Lagrangian function as follows:
    \begin{align}
    \label{A1}
    L_1&=\alpha^T(e_2 + R_2) - \frac{1}{2} \alpha^T G(H^TH + \delta I)^{-1}G^T \alpha - q^T\alpha \nonumber \\
    & - \sigma^T(d_1e_2 - \alpha),
\end{align}
Using the K.K.T. conditions, we obtained
\begin{align}
    & \frac{\partial L_1}{\partial \alpha} = (e_2 + R_2) - G(H^TH + \delta I)^{-1}G^T \alpha - q + \sigma = 0, \label{A2}\\
    & q^T\alpha = 0, \label{A3} \\
    & \sigma^T(d_1e_2 - \alpha) = 0. \label{A4}
\end{align}
If \(\alpha^{i} > 0\), then the corresponding GB centre \(c_i\) is a support vector. We denote \(s = \{i \mid \alpha^{i} > 0, i = 1, \ldots, p_2\}\) as the corresponding positive index set with the cardinality \(|s| = d\). \\
From Eq. \eqref{A2}, we have
    \begin{align}
        \frac{\partial L_1}{\partial \alpha^s} = & (e_2 + R_2)^s - G^{ss}(H^{ss^T}H^{ss} + \delta I)^{-1}G^{ss^T} \alpha^s  \nonumber \\
        &~~~ - q^s + \sigma^s = 0,
    \end{align}
where \((\cdot)^s\) (support vectors) denotes the subvector of vector \((\cdot)\) and its elements are the elements in vector \((\cdot)\) corresponding to the index set \(s\); \((\cdot)^{ss}\)  denotes the submatrix of the matrix \((\cdot)\) and its elements are the crossing elements of rows and columns in matrix \((\cdot)\) corresponding to the index set \(s\). \\
If \(\alpha_i > 0\), then from Eqs. \eqref{A3} and \eqref{A4}, we can conclude that \(q = 0\) and $\sigma = 0$. Therefore, 
\begin{align}
    (e_2 + R_2)^s - G^{ss}(H^{ss^T}H^{ss} + \delta I)^{-1}G^{ss^T} \alpha^s = 0.
\end{align}
Then,
\begin{align}
\label{A8}
      \alpha^s = \left( G^{ss}(H^{ss^T}H^{ss} + \delta I)^{-1}G^{ss^T} \right)^{-1} (e_2 + R_2)^s.
\end{align}
Calculating the inverse of $\left( G^{ss}(H^{ss^T}H^{ss} + \delta I)^{-1}G^{ss^T} \right)$ possesses a significant challenge due to its positive semidefinite nature by lemma \ref{lemma3}. However, this obstacle can be effectively managed by adding a very small quantity, denoted as \(\vartheta I\) in Eq. \eqref{A8}, where $\vartheta >0$ is a very small quantity and \(I\) represents an identity matrix of suitable dimensions.
\begin{align}
\label{A9}
      \alpha^s = \left( G^{ss}(H^{ss^T}H^{ss} + \delta I)^{-1}G^{ss^T} + \vartheta I \right)^{-1} (e_2 + R_2)^s.
\end{align}
According to the Woodbury formula, the following equation holds:
\begin{align}
      \alpha^s =& \vartheta^{-1} I - \vartheta^{-2} G^{ss} \left( (H^{ss^T}H^{ss} + \delta I) + \vartheta^{-1} G^{ss^T}G^{ss} \right)^{-1}  \nonumber \\
      & \cdot G^{ss^T}(e_2 + R_2)^s.
\end{align}
Let $\Delta = \vartheta^{-1} \in \mathbb{R}$, then
\begin{align}
\label{A10}
      \alpha^s =& \Delta I - \Delta^2 G^{ss} \left( (H^{ss^T}H^{ss} + \delta I) + \Delta G^{ss^T}G^{ss} \right)^{-1}  \nonumber \\
      & \cdot G^{ss^T}(e_2 + R_2)^s.
\end{align}
Consequently, we can deduce that
\begin{align}
    \lvert \alpha^{i}  - \alpha^{j} \rvert  & = \Delta^2 \left| (\tilde{c}_i - \tilde{c}_j)^T \left( (H^{ss^T}H^{ss} + \delta I) \right. \right. \nonumber \\
     & \left. \left. + G^{ss^T}G^{ss} \right)^{-1} G^{ss^T}(e_2 + R_2)^s \right|, \label{A11} 
\end{align}
where $\tilde{c} = (c,~1)$. Let $J = \left( (H^{ss^T}H^{ss} + \delta I) + G^{ss^T}G^{ss} \right)^{-1}$. Since \(J\) is a Hermite matrix. Then, by the Cauchy-Schwarz inequality, it follows
\begin{align}
\label{A12}
    & \lvert (\tilde{c}_i - \tilde{c}_j)^T J G^{ss^T}(e_2 + R_2)^s \rvert^2  \leq (\tilde{c}_i - \tilde{c}_j)^T J (\tilde{c}_i - \tilde{c}_j) \cdot   \nonumber \\
     & ~~~~~~~~~~~~ (e_2 + R_2)^{s^T} G^{ss}JG^{ss^T} (e_2 + R_2)^s \nonumber \\
    & ~~~~~~~~~~~~~~~~~~~~~~~~~~ = h_1 \cdot h_2,
\end{align}
where $h_1 = (\tilde{c}_i - \tilde{c}_j)^T J (\tilde{c}_i - \tilde{c}_j)$ and $h_2 = (e_2 + R_2)^{s^T} G^{ss}JG^{ss^T} (e_2 + R_2)^s$. \\
For \(h_1\), according to the Rayleigh-Ritz Theorem \cite{horn2012matrix}, we have
\begin{align}
\label{A13}
    (\tilde{c}_i - \tilde{c}_j)^T J (\tilde{c}_i - \tilde{c}_j) \leq \lambda_{\max}(J)\cdot d_{ij}^2, 
\end{align}
where \(d_{ij} = \| \tilde{c}_i - \tilde{c}_j \| = \| c_i - c_j \|\) is the distance between \(c_i\) and \(c_j\). \(\lambda_{\max}(J)\) is the largest eigenvalue of matrix \(J\). \\
Similarly for \(h_2\), we have
\begin{align}
\label{A14}
    (e_2 + R_2)^{s^T} G^{ss}JG^{ss^T} (e_2 + R_2)^s \leq & \lambda_{\max}(J) (e_2 + R_2)^{s^T}  \nonumber \\ 
    & G^{ss}G^{ss^T} \cdot (e_2 + R_2)^s.
\end{align}   
Since  $(e_2 + R_2)^{s^T}(e_2 + R_2)^s = ((1+r_1)^2+ (1+r_2)^2 + \ldots + (1+r_{p_2})^2) = \kappa \in \mathbb{R}$ (say), here $r_i \in \mathbb{R}$ is the radius of the granular ball. Then Eq. \eqref{A14} reduced to
\begin{align}    
\label{A15}
 (e_2 + R_2)^{s^T} G^{ss}JG^{ss^T} (e_2 + R_2)^s \leq \lambda_{\max}(J)\cdot \kappa\|G^{ss}\|_F^2,
\end{align}
where \(\|G^{ss}\|_F\) is the Frobenius norm of the matrix \(G^{ss}\).  \\
Combining \eqref{A11}, \eqref{A12}, \eqref{A13}, and \eqref{A15}, we obtain 
\begin{align}
\label{Th19}
    \lvert \alpha^{i}  - \alpha^{j} \rvert & \leq \Delta^2\lambda_{\max}(J)\sqrt{\kappa}\|G^{ss}\|_F\cdot d_{ij} \nonumber \\
        &   \leq \Delta^2\lambda_{\max}(J)\sqrt{\kappa}\|G\|_F\cdot d_{ij}.
\end{align}
Moreover, $H^{ss^T}H^{ss}$ and $G^{ss^T}G^{ss}$ are Hermite matrices, according to the Weyl Theorem \cite{horn2012matrix}, we can obtain
\begin{align}
\label{Th20}
    \lambda_{\max}(J) & = \lambda_{\max} \left( (H^{ss^T}H^{ss} + \delta I) + G^{ss^T}G^{ss} \right) \nonumber \\
    & \leq \lambda_{\max} (H^{ss^T}H^{ss} + \delta I) + \lambda_{\max}(G^{ss^T}G^{ss}) \nonumber \\
    & \leq \delta + \lambda_{\max}  (H^{ss^T}H^{ss}) + \lambda_{\max}(G^{ss^T}G^{ss}) \nonumber \\
    & \leq \delta + \lambda_{\max}  (H^{T}H) + \lambda_{\max}(G^{T}G).
\end{align} 
Since $\alpha > 0$, from Eq. \eqref{EQQ:9}, we obtained 
\begin{align}
    & \xi_2 = R_2 + e_2 +(C_2w_1 + e_2b_1) \nonumber \\
   & \xi_2 = R_2 + e_2 + Gu_1    ~( \because G=(C_2, ~ e_2) ~ \text{and} ~ u_1=(w_1^T, ~b_1)^T) \nonumber \\
   & \xi_2 = R_2 + e_2 - G(H^TH+\delta I)^{-1}G^T\alpha ~~~ (\text{From}~ \eqref{eq:18}). 
\end{align}
Now, 
\begin{align}
    \lvert  \xi_2^i - \xi_2^j \rvert = \lvert  (\tilde{c}_i - \tilde{c}_j) L (\tilde{c}_i - \tilde{c}_j)^T (\alpha_i - \alpha_j) \rvert,
\end{align}
where \(L=(H^{ss^T}H^{ss} + \delta I)^{-1}\) is a Hermite matrix. Then, by the Rayleigh-Ritz Theorem, it follows
\begin{align}
\label{Th23}
     \lvert  \xi_2^i - \xi_2^j  \rvert \leq \lambda_{max}(L) d_{ij}^2  \lvert  \alpha_i - \alpha_j \rvert.
\end{align}
Since, $H^{ss^T}H^{ss}$ is a Hermite matrix. Then, according to the Weyl Theorem, we have
\begin{align}
\label{Th24}
    \lambda_{max}(L) & = \lambda_{max}(H^{ss^T}H^{ss} + \delta I) \nonumber \\
    & \leq \delta + \lambda_{max}(H^{ss^T}H^{ss}) \nonumber \\
    & \leq \delta + \lambda_{max}(H^{T}H).
\end{align}
Suppose $\tau_1 = \lambda_{\max}(H^TH)$ and $\tau_2 = \lambda_{\max}(G^TG)$, by using Eqs. \eqref{Th19}, \eqref{Th20}, \eqref{Th23} and \eqref{Th24} we can conclude that
\begin{align}
    \lvert \hat{\xi}_2^i - \hat{\xi}_2^j \rvert \leq  \Delta^2(\delta + \tau_1)(\delta + \tau_1 + \tau_2)\sqrt{\kappa}\|G\|_F\cdot d_{ij}^3.
\end{align}
\end{proof}
\begin{theorem}
\label{TH2}
    If $(\hat{w}_2^T, \hat{b}_2^T, \hat{\xi}_1^T)$ is the optimal solution of the GBTSVM \eqref{eq:20}. Then, for any two negative GBs \(((c_i, r_i), y_i)\) and \(((c_j, r_j), y_j)\), the estimations of the corresponding slack variables \(\xi_1^i\) and \(\xi_1^j\) satisfy
    \begin{align}
    \label{T1}
        \lvert \hat{\xi}_1^i - \hat{\xi}_1^j \rvert \leq  \Delta^2(\delta + \tau_2)(\delta + \tau_1 + \tau_2)\sqrt{\kappa}\|H\|_F\cdot d_{ij}^3.
    \end{align}
\end{theorem}
A similar approach can be applied to give the proof of Theorem \ref{TH2}.
\par
The derived VTUB, as established in Theorems \ref{TH1} and \ref{TH2}, is directly proportional to the distance between two training samples in the GBTSVM. This means that the closer the two training samples are, the more similar their corresponding tolerances from the GBTSVM will be. Consequently, if two training samples are identical, their corresponding tolerances will also be identical. These findings align with our intuitive understanding. 
\section{ Experimental results}
\label{experimental result}
To test the efficiency of proposed models, i.e., GBTSVM and LS-GBTSVM, we compare them to baseline models on publicly available UCI \cite{dua2017uci} and KEEL \cite{derrac2015keel} benchmark datasets under different noise levels. Furthermore, we conduct experiments on datasets generated by the NDC Data Generator \cite{ndc}. In the Supplementary material, we conduct a series of sensitivity analyses on various aspects of the proposed models. This includes investigating the impact of granular ball parameters $num$ and $pur$ on the proposed models in subsection S.IV.A, assessing different levels of label noise on both proposed and baseline models in subsection S.IV.B, and conducting sensitivity analyses on hyperparameters $\sigma$ and $pur$ in subsection S.IV.C. Additionally, sensitivity analyses on hyperparameters $d_1$ and $d_2$ are presented in subsection S.IV.D, and the relationship between the number of granular balls generated and the performance of the proposed GBTSVM model with different purities is analyzed in subsection S.IV.E.
\subsection{Experimental Setup}
The hardware environment for the experiment consists of a PC equipped with an Intel(R) Xeon(R) Gold $6226$R CPU $@$ $2.90$GHz  CPU and $128$ GB RAM running on Windows $11$ possessing Python $3.11$. The dual QPPs arising in GBTSVM, LS-GBTSVM and baseline models are solved using the ``qpsolvers" function in the CVXOPT package. Also, we use the SOR algorithm \cite{luo1993error, mangasarian1999successive} for solving LS-GBTSVM's dual problems, and the proposed model is named LS-GBTSVM (SOR). 
The dataset is randomly divided into training and testing sets in a $70:30$ ratio, respectively. We use five-fold cross-validation and grid search approach to optimize the models' hyperparameters from the following ranges: $d_i = \{10^{-5}, 10^{-4}, \ldots, 10^5\}$ for $i=1,2,3,4.$ For the nonlinear case we use Gaussian kernel and is given by $K(x_i,x_j) = e^{\frac{-1}{2\sigma^2}\|x_i - x_j\|^2}.$ Gaussian kernel parameter $\sigma$ is selected from the range $\{2^{-5}, 2^{-4}\ldots, 2^{5}\}$. In LS-GBTSVM, we adopt equal penalty parameters, \textit{i.e.,} $d_1 = d_2$ and $d_3 = d_4$, for both linear and non-linear cases. Further, we compare the results on NDC-based large-scale datasets. We assign a penalty parameter equals $1$ to the proposed and the baseline models, \textit{i.e.,} $(d_1 = d_2 = d_3 = d_4 = 1)$.
\subsection{Experiments on Real World UCI and KEEL Datasets}
\label{Experments on real-world UCI and KEEL datasets}
In this subsection, we present an intricate analysis involving a comparison of the proposed GBTSVM, LS-GBTSVM and LS-GBTSVM (SOR) with SVM \cite{cortes1995support}, GBSVM \cite{xia2022gbsvm} and TSVM \cite{khemchandani2007twin} models on $36$ UCI and KEEL  benchmark datasets. The optimization problem of GBSVM is solved by PSO algorithm \cite{xia2022gbsvm} and named GBSVM (PSO). The detailed experimental results are presented in Table S.III and Table S.IV of the Supplementary material. All the experimental results discussed in this subsection are obtained at a $0\%$ noise level for both linear and non-linear cases. The comparison in terms of accuracy (ACC) indicates that our proposed GBTSVM, LS-GBTSVM, and LS-GBTSVM (SOR) models yield better performance than the baseline SVM, GBSVM (PSO), and TSVM models on most of the datasets. From Table \ref{Classification performance in Linear Case.}, the average ACC of proposed GBTSVM, LS-GBTSVM, and LS-GBTSVM (SOR) models are $85.30\%$, $85.32\%$ and $81.34\%$, respectively. whereas the average ACC of SVM, GBSVM (PSO), and TSVM models are $82.21\%$, $72.46\%$ and $65.47\%$, respectively. The average ACC of the proposed GBTSVM, LS-GBTSVM, and LS-GBTSVM (SOR) models surpasses that of the baseline models. As the average ACC can be influenced by exceptional performance in one dataset that compensates for losses across multiple datasets, it might be a biased measure. Therefore, we employ the ranking method to gauge the effectiveness and appraise the performance of the models. Here, each classifier is assigned a ranking, with the superior-performing model receiving a lower rank, while the model with inferior performance is given a higher rank.
For evaluation of $q$ models across $N$ datasets, the rank of the $j^{th}$ model on the $i^{th}$ dataset can be denoted as $\mathfrak{R}_{j}^i$. Then the model's average rank is given by $\mathfrak{R}_j = \frac{1}{N} \sum_{i=1}^{N}\mathfrak{R}_j^i.$ The average rank of SVM, GBSVM (PSO), TSVM, GBTSVM, LS-GBTSVM, and LS-GBTSVM (SOR) are $3.46$, $4.79$, $5.69$, $1.97$, $1.94$ and $3.14$, respectively. It is evident that the proposed GBTSVM, LS-GBTSVM, and LS-GBTSVM (SOR) have the best average rank. Hence, the generalization ability of the proposed GBTSVM and LS-GBTSVM is superior compared to the baseline models. Now, we conduct the Friedman test \cite{demvsar2006statistical} to determine whether the models have significant differences. Under the null hypothesis of the Friedman test, it is presumed that all the models exhibit an equal average rank, signifying equal performance. The Friedman test adheres to the chi-squared distribution ($\chi_F^2$) with $(q-1)$ degree of freedom (d.o.f) and its calculation involves: $ \chi^2_F = \frac{12N}{q(q+1)}\left[\sum_{j}\mathfrak{R}_j^2 - \frac{q(q+1)^2}{4}\right].$ The $F_F$ statistic is calculated as: $F_F = \frac{(N - 1)\chi_F^2}{N(q-1) - \chi_F^2}$, where $F$- distribution has $(q-1)$ and $(N-1)\times (q-1)$ degrees of freedom. For $q=6$ and $N=36$, we get $\chi_F^2 = 116.19$ and $F_F = 63.70$ at $5\%$ level of significance. Referring to the statistical $F$-distribution table, $F_F(5, 175) = 2.2657$. Since $63.70 > 2.2657$, we reject the null hypothesis. As a result, there exists a statistical distinction among the models being compared. Next, we employ the Nemenyi post hoc test to examine the pairwise distinctions between the models. The value of the critical difference $(C.D.)$ is evaluated as $C.D. = q_\alpha \sqrt{\frac{q(q+1)}{6N}}$, where $q_\alpha$ represents the critical value from the distribution table for the two-tailed Nemenyi test. According to statistical $F$-distribution table, $q_\alpha = 2.850$ at $5\%$ significance level, the $C.D.$ is calculated to be $1.256$. The average rank disparities between the proposed (GBTSVM, LS-GBTSVM, LS-GBTSVM (SOR)) models with SVM, GBSVM (PSO), and TSVM are $(1.49, 1.52, 0.32)$, $(2.82, 2.85, 1.65)$ and $(3.72, 3.75, 2.55)$ respectively. According to Nemenyi post hoc test, the proposed models GBTSVM, LS-GBTSVM, and LS-GBTSVM (SOR) significantly differ from the baseline models except for LS-GBTSVM (SOR) with SVM. It is evident that the proposed GBTSVM and LS-GBTSVM show better performance compared to baseline models. LS-GBTSVM (SOR) is significantly different among the models (except SVM) and LS-GBTSVM (SOR) surpasses SVM in terms of average rank.
\begin{table*}[]
\centering
    \caption{Average accuracy (ACC) and average rank of the proposed GBTSVM and LS-GBTSVM along with the baseline models over for UCI and KEEL datasets with linear kernel.}
    \label{Classification performance in Linear Case.}
    \resizebox{0.9\linewidth}{!}{
\begin{tabular}{cccccccc}
\hline
 & Noise & SVM \cite{cortes1995support} & GBSVM (PSO) \cite{xia2022gbsvm} & TSVM \cite{khemchandani2007twin} & GBTSVM$^{\dagger}$ & LS-GBTSVM$^{\dagger}$ & LS-GBTSVM (SOR)$^{\dagger}$ \\ \hline
Average ACC (Rank) & 0 \% & 82.21 (3.46) & 72.46 (4.79) & 65.47 (5.69) & 85.30 (1.97) & \textbf{85.32} (1.94) & 81.34 (3.14) \\
 & 5 \% & 82.11 (3.74) & 75.92 (4.78) & 82.27 (3.88)& \textbf{83.75} (2.83) & 82.96 (2.68)& 82.49 (3.10) \\
 & 10 \% & 83.06 (3.65)& 75.36 (4.97) & 84.64 (3.10)& \textbf{85.82} (2.42) & 82.11 (3.28) & 82.25 (3.58) \\
 & 15 \% & 81.25 (3.44) & 74.44 (4.94) & 81.11 (3.53) & \textbf{85.30} (2.26) & 81.50 (3.44) & 80.67 (3.38)\\
 & 20 \% & 82.41 (3.35)& 74.89 (4.90) & 82.46 (3.40)& \textbf{84.86} (2.99)& 83.11 (3.19) & 82.75 (3.17) \\ \hline
 \multicolumn{8}{l}{$^{\dagger}$ represents the proposed models. Bold text denotes the model with the highest average ACC.}
\end{tabular}}
\end{table*}
\begin{table}[]
\centering
    \caption{Pairwise win-tie-loss test of all the compared models with linear kernel.}
    \label{win tie loss sign test}
    \resizebox{\columnwidth}{!}{%
\begin{tabular}{lccccc}
\hline
\textbf{}                         & SVM \cite{cortes1995support}         & GBSVM (PSO) \cite{xia2022gbsvm}       & TSVM \cite{khemchandani2007twin}          & GBTSVM$^{\dagger}$  & LS-GBTSVM$^{\dagger}$ \\ \hline
GBSVM (PSO) \cite{xia2022gbsvm}                    & $[4,     4,    28]$  &                       &                       &                          &                             \\
TSVM \cite{khemchandani2007twin}                     & $[0,    0,    36]$   & $[9,     0,    27]$   &                       &                          &                             \\
GBTSVM$^{\dagger}$          & $[33,     2,     1]$ & $[ 35,     0,     1]$ & $[ 36,     0,     0]$ &                          &                             \\
LS-GBTSVM$^{\dagger}$       & $[28,     2,     6]$ & $[33,     1,     2]$  & $[36,     0,     0]$  & $[18,     2,    16]$     &                             \\
LS-GBTSVM (SOR)$^{\dagger}$ & $[19,     1,    16]$ & $[29,     0,     7]$  & $[34,     0,     2]$  & $[12,     2,    22]$     & $[5,     5,    26]$ \\ \hline 
\multicolumn{6}{l}{$^{\dagger}$ represents the proposed models.}
\end{tabular}}
\end{table}
\begin{table}[ht!]
\centering
    \caption{Pairwise win-tie-loss test of all the compared models with the non-linear kernel.}
    \label{win tie loss sign test for non linear}
     \resizebox{\columnwidth}{!}{%
\begin{tabular}{lccccc}
\hline
\multicolumn{1}{c}{} & \multicolumn{1}{c}{SVM \cite{cortes1995support}} & \multicolumn{1}{c}{GBSVM (PSO) \cite{xia2022gbsvm}} & \multicolumn{1}{c}{TSVM \cite{khemchandani2007twin}} & \multicolumn{1}{c}{GBTSVM$^{\dagger}$} & \multicolumn{1}{c}{LS-GBTSVM$^{\dagger}$} \\ \hline
GBSVM (PSO) \cite{xia2022gbsvm}  & [19,    3,    14] &  &  &  &  \\
TSVM \cite{khemchandani2007twin} & [24,     2,    10] & [18,     5,    13] &  &  &  \\
GBTSVM$^{\dagger}$ & [27,     1,     8] & [27,     1,     8] & [26,     2,     8]&  &  \\
LS-GBTSVM$^{\dagger}$ & [20,    10,     6] & [22,     1,    13] & [19,     2,    15] & [8,     2,    26] &  \\
LS-GBTSVM (SOR)$^{\dagger}$ & [19,    12,     5] & [22,     2,    12] & [15,     3,    18] & [6,     2,    28] & [3,    25,     8] \\ \hline
\multicolumn{6}{l}{$^{\dagger}$ represents the proposed models.}
\end{tabular}}
\end{table}
\begin{table*}[ht!]
\centering
    \caption{Average accuracy (ACC) and average rank of the proposed GBTSVM and LS-GBTSVM along with the baseline models over for UCI and KEEL datasets with non-linear kernel.}
    \label{Classification performance in nonLinear Case.}
    \resizebox{0.9\linewidth}{!}{
\begin{tabular}{cccccccc}
\hline
 & Noise & SVM \cite{cortes1995support} & GBSVM (PSO) \cite{xia2022gbsvm} & TSVM \cite{khemchandani2007twin} & GBTSVM$^{\dagger}$ & LS-GBTSVM$^{\dagger}$ & LS-GBTSVM (SOR)$^{\dagger}$ \\ \hline
{Average ACC (Rank)} & 0 \% & 76.27 (4.42) & 79.43 (4.03) & 84.83 (3.5) & \textbf{88.74} (2.17) & 85.85 (3.31) & 84.93 (3.58) \\
 & 5 \% & 77.86 (4.53) & 80.7 (4.38) & 85.13 (3.56) & \textbf{88.61} (2.21) & 86.95 (2.86) & 84.57 (3.47) \\
 & 10 \% & 77.19 (4.43) & 84.05 (3.9) & 84.51 (3.81) & \textbf{89.25} (2.1) & 84.79 (3.33) & 84.68 (3.43) \\
 & 15 \% & 79.3 (4.24) & 83.24 (4.5) & 85.37 (3.28) & \textbf{88.49} (2.43)& 85.32 (3.21) & 85.48 (3.35) \\
 & 20 \% & 81.08 (4.18) & 81.76 (4.47)& 84.6 (3.44) & \textbf{86.97} (2.54) & 85.39 (3.1) & 85.2 (3.26) \\ \hline
  \multicolumn{8}{l}{$^{\dagger}$ represents the proposed models. Bold text denotes the model with the highest average ACC.}
\end{tabular}}
\end{table*}
\par
Furthermore, to analyze the models, we use pairwise win-tie-loss sign test. As per the win-tie-loss sign test, under the null hypothesis, it is assumed that two models perform equivalently and are expected to win in $N/2$ datasets, where $N$ represents the dataset count. If the classification model win on approximately $\frac{N}{2} + 1.96 \frac{\sqrt{N}}{2}$, then the model is significantly better. Also, if there is an even count of ties between the two models, these ties are evenly divided between them. However, if the number of ties is odd, we disregard one tie and allocate the remaining ties among the specified classifiers. In this case, when $N=36$ if one of the models wins is at least $23.88$, then there is a significant difference between the models. Table \ref{win tie loss sign test} illustrates the comparative performance of the proposed GBTSVM, LS-GBTSVM, and LS-GBTSVM (SOR) models along with the baseline models, presenting their outcomes in terms of pairwise wins, ties, and losses using UCI and KEEL datasets. In Table \ref{win tie loss sign test}, the entry $[x, y, z]$ indicates that the model mentioned in the row wins $x$ times, ties $y$ times, and loses $z$ times in comparison to the model mentioned in the respective column. Table \ref{win tie loss sign test} clearly indicates that the proposed GBTSVM and LS-GBTSVM model exhibits significant superiority compared to the baseline models. Moreover, the proposed LS-GBTSVM (SOR) model achieves a statistically significant difference from GBSVM (PSO) and TSVM. Demonstrating a significant level of performance, the LS-GBTSVM (SOR) model succeeds in $19$ out of $36$ datasets. Therefore, the proposed GBTSVM, LS-GBTSVM, and LS-GBTSVM (SOR) models are significantly superior compared to the existing models.
\par
For the non-linear case, the average ACC and rank values are shown for the proposed GBTSVM, LS-GBTSVM, and LS-GBTSVM (SOR) with SVM, GBSVM, and TSVM in Table \ref{Classification performance in nonLinear Case.}. From the Table, it is evident that the proposed GBTSVM, LS-GBTSVM, and LS-GBTSVM (SOR) outperform the baseline models on the majority of the datasets. The average ACC of proposed GBTSVM, LS-GBTSVM, LS-GBTSVM (SOR), SVM, GBTSVM, and TSVM are $88.74\%$, $85.85\%$, $84.93\%$, $76.27\%$, $79.43\%$ and $84.83\%$, respectively. This demonstrates a clear performance improvement, with our proposed models securing the top position compared to the baseline models. It can be noted that among all the models, our proposed GBTSVM holds the lowest average rank. Furthermore, we conduct the Friedman statistical test along with Nemenyi post hoc tests. We compute $\chi^2_F = 30.94$ and $F_F = 7.26$ and $F_F(5, 175) = 2.2657$ at $5\%$ level of significance. Since $F_F(5, 175)<F_F$, therefore we reject the null hypothesis. Moreover, the Nemenyi post-hoc test is employed to identify significant differences among the pairwise comparisons. We compute $C.D. = 1.256$ and the average rankings of the models listed in Table \ref{Classification performance in nonLinear Case.} should have a minimum difference by $1.256$. The average rank disparities between the proposed (GBTSVM, LS-GBTSVM, LS-GBTSVM (SOR)) models with SVM, GBSVM (PSO), and TSVM are $(2.25, 1.11, 0.84)$, $(1.86, 0.72, 0.45)$ and $(1.33, 0.19, 0.08)$ respectively. Hence, the proposed GBTSVM exhibits significant superiority over the baseline models. The Friedman test did not show the statistical difference between the proposed LS-GBTSVM and LS-GBTSVM (SOR) with baseline SVM, GBTSVM, and TSVM models as well as GBTSVM with SVM. However, Table \ref{Classification performance in nonLinear Case.} unequivocally demonstrates that both the proposed LS-GBTSVM and LS-GBTSVM (SOR) consistently achieve the lowest average rank and higher average ACC in comparison to the baseline classifiers. Also, the proposed GBTSVM has a lower rank when compared to SVM. As a result, the proposed GBTSVM, LS-GBTSVM, and LS-GBTSVM (SOR) outperformed the baseline models.
\begin{table*}[ht!]
\centering
    \caption{Testing accuracy (ACC) and training time of classifiers on NDC datasets with linear kernel.}
    \label{The average learning results of classifiers on NDC}
    \resizebox{0.9\linewidth}{!}{
\begin{tabular}{lcccccc}
\hline
NDC datasets & SVM \cite{cortes1995support} & GBSVM (PSO) \cite{xia2022gbsvm}  & TSVM \cite{khemchandani2007twin} & GBTSVM$^{\dagger}$  & LS-GBTSVM$^{\dagger}$  & LS-GBTSVM (SOR)$^{\dagger}$  \\  
& ACC(\%) (Time(s)) &  ACC(\%) (Time(s)) & ACC(\%) (Time(s)) & ACC(\%) (Time(s)) & ACC(\%) (Time(s)) & ACC(\%) (Time(s)) \\ \hline
NDC-10k & 80.64 (309.0300) & 52.43 (1044.2187) & \textbf{86.59} (209.606) & 81.44 (0.1562) & 83.89 (0.5472) & 79.59 (0.2343) \\
NDC-50k & 79.42 (809.5466) & 53.41 (2478.1415) & \textbf{86.21} (715.689) & 80.84 (0.578) & 83.42 (2.5467) & 79.9 (1.2856) \\
NDC-1l & b & a & b & \textbf{85.77} (0.3562) & 84.17 (3.6237) & 74.6 (1.5005) \\
NDC-3l &  b& a & b & 80.52 (0.906) & \textbf{82.97} (3.7371) & 73.41 (4.1415) \\
NDC-5l &  b& a & b & 81.12 (1.5326) & \textbf{82.64} (5.7185) & 75.65 (5.8356) \\
NDC-1m &  b& a & b & 79.42 (2.7863) & 78.9 (6.9668) & \textbf{79.98} (12.3445) \\
NDC-3m & b & a & b & 78.04 (8.1229) & \textbf{84.73} (8.1191) & 79.54 (14.6126) \\
NDC-5m & b & a & b & \textbf{78.68} (12.8752) & 77.54 (24.0667) & 73.65 (32.146) \\ \hline
\multicolumn{7}{l}{\begin{tabular}[c]{@{}l@{}}$^a$ Experiment is terminated because of the out of bound issue shown by PSO algorithm. $^b$ Terminated because of out of memory.\\
$^{\dagger}$ represents the proposed models. Bold text denotes the model with the highest average ACC.
\end{tabular}}
\end{tabular}}
\end{table*}
We also conduct a win-tie-loss sign test for non-linear cases. Table \ref{win tie loss sign test for non linear} shows the pairwise win-tie-loss of the compared models on UCI and KEEL datasets. In our case, if any of the two models wins on at least $23.88$ datasets, the two models are statistically different. It is evident from Table \ref{win tie loss sign test for non linear} that the proposed GBTSVM model statistically outperforms the baseline SVM, GBSVM (PSO), and TSVM models. The proposed LS-GBTSVM is statistically better than SVM. In general, from Table \ref{win tie loss sign test for non linear} and the aforementioned analysis; GBTSVM, LS-GBTSVM, and LS-GBTSVM (SOR) models with different performance metrics and statistical tests, it becomes clear that the proposed GBTSVM, LS-GBTSVM, and LS-GBTSVM (SOR) models exhibit competitive or even superior performance when compared with the baseline models. 
\subsection{Experiments on UCI and KEEL Datasets with Label Noise}
To validate the effectiveness and noise resilience of the proposed GBTSVM and LS-GBTSVM models, we contaminate the label noise, including $5\%$, $10\%$, $15\%$, and $20\%$ on each dataset. The comparative experimental results of the proposed GBTSVM and LS-GBTSVM models along with the baseline models with linear and non-linear cases are shown in Table S.III and Table S.IV in Section S.V of the Supplementary material. The average ACC and average rank of the models with the linear kernel are shown in Table \ref{Classification performance in Linear Case.} with different percentages of noise labels. The classification ACC of proposed GBTSVM and LS-GBTSVM models is better than the baseline models. It indicates that the GBTSVM and LS-GBTSVM show better robustness among the compared models because a granular ball possesses a coarser granularity, which can mitigate the impact of label noise within it. The label assigned to a granular ball is primarily determined by the predominant label contained within it, and the presence of label noise with minority labels does not significantly influence the determination of the granular ball. We observe that using the ``qpsolvers" function to solve the dual problem of LS-GBTSVM gets more stable classification results than using the SOR algorithm in LS-GBTSVM on most of the dataset. The average ACC of LS-GBTSVM is higher than the LS-GBTSVM (SOR). Furthermore, the average rank of LS-GBTSVM under $0\%$ and $5\%$ label noise is $1.94$ and $2.68$, respectively, which is the lowest rank and GBTSVM is $2.97$ and $2.83$, respectively, are the second lowest. Similarly, GBTSVM has the lowest rank under $10\%$, $15\%$, and $20\%$ label noise. In conclusion, GBTSVM and LS-GBTSVM consistently outperform the compared models across various levels of label noise. The sensitivity analysis of the proposed GBTSVM model, considering different levels of label noise is presented in subsection S.IV.B of the Supplementary material.
\subsection{Experiment on Artificial NDC Datasets.}
To showcase the superiority of the proposed GBTSVM, LS-GBTSVM, and LS-GBTSVM (SOR) models in terms of training speed and scalability, we perform experiments using the NDC datasets \cite{ndc}. In this experiment, the NDC datasets are generated with varying sizes, ranging from $10\text{k}$ to $5\text{m}$ while keeping the number of features constant to $32$. Table \ref{The average learning results of classifiers on NDC} presents the ACC and training time of the compared models on the NDC datasets. The results show that the proposed GBTSVM, LS-GBTSVM, and LS-GBTSVM (SOR) are more efficient among the baseline models. The following issues arise in the baseline models while handling large-scale datasets: $(i)$ training of TSVM demands substantial memory consumption to compute matrix inversion. As a result, an ``out-of-memory" issue occurs when the scale reaches $1\text{l}$, $(ii)$ training SVM requires solving a QPP, which requires significant computational resources. As the dataset size increases, the time required for training and prediction can become prohibitive, and $(iii)$ the PSO algorithm in the GBSVM (PSO) is halted due to the emergence of an out-of-bounds issue. The results of the experiments demonstrate that the proposed GBTSVM, LS-GBTSVM, and LS-GBTSVM (SOR) exhibit efficiency several hundreds or even thousands of times faster than the compared models. This is due to the fact that the count of generated granular balls on a dataset is significantly lower compared to the total number of samples.
\section{Conclusions}
\label{conclusions}
In this paper, we proposed a novel granular ball twin support vector machine (GBTSVM) as a solution to the challenges faced by TSVM. GBTSVM utilized the coarse granularity of granular balls for input, leading to two nonparallel hyperplanes for sample classification. The proposed GBTSVM mitigates the impact of noise and outliers while also eliminating the overhead of higher computational costs typically associated with standard SVM, TSVM, and their variants.
We again proposed a novel large-scale GBTSVM (LS-GBTSVM) by incorporating a regularization term in the primal optimization formulation to implement the SRM principle. LS-GBTSVM's key advantage lies in avoiding matrix inversions, which made it suitable for large-scale problems; and effectively addressing overfitting concerns.

To demonstrate the effectiveness, robustness, scalability, and efficiency of the proposed GBTSVM and LS-GBTSVM models, we conducted a series of rigorous experiments and subjected them to comprehensive statistical analyses. Our experimental results, encompassing 36 UCI and KEEL datasets (with and without label noise), were subjected to a series of statistical tests. The experimental results, along with the statistical analyses, indicate that the proposed linear and non-linear GBTSVM and LS-GBTSVM models beat baseline approaches in efficiency and generalization performance. Here are the key findings: $(i)$ The proposed models exhibit an average ACC improvement of up to $20\%$ in comparison to the baseline models for the linear case. $(ii)$ 
Our models have demonstrated an up to $12\%$ increase in average ACC compared to the baseline models in noisy conditions, showcasing exceptional resilience when contrasted with the baseline models for nonlinear cases. $(iii)$ We tested the models on large-scale NDC datasets from 10k to 5m samples. Baseline models faced memory issues beyond NDC-50k, but our proposed models excelled, demonstrating scalability and efficiency on large-scale datasets. $(iv)$ We conducted a series of sensitivity analyses to understand the behavior of hyperparameters of the proposed models. The key hyperparameters under investigation include the granular ball parameters $num$, $pur$, $\sigma$, different levels of label noise, $d_1$, and $d_2$ with different-different combinations. While our proposed models have showcased outstanding performance in binary classification problems. An essential avenue for future research would involve adapting and extending these models to address the complexities associated with multi-class classification scenarios. 
\vspace{-0.2cm}
\section*{Acknowledgment}
This project received funding from the Indian government's Department of Science and Technology (DST) and the Ministry of Electronics and Information Technology (MeitY) through the MTR/2021/000787 grant as part of the Mathematical Research Impact-Centric Support (MATRICS) scheme. Md Sajid's fellowship is provided by the Council of Scientific and Industrial Research (CSIR), New Delhi, under the grants 09/1022(13847)/2022-EMR-I.
\bibliographystyle{IEEEtranN}
\bibliography{refs.bib}

\clearpage
\section*{Supplementary Material}
\renewcommand{\thesection}{S.I}
\section{Mathematical Formulation of GBSVM and TSVM}
In this section, we go through the mathematical formulation of GBSVM and TSVM.
\subsection{Granular Ball Support Vector Machine (GBSVM)}
The GBSVM \cite{xia2022gbsvm} model initiates by partitioning the input data points into granular balls of different sizes. These balls, characterized by their centers and radii, are then fed into the classifier.
In Figure \ref{Schematics of GBSVM}, the red and blue colors represent the granular balls of the $+1$ class and the $-1$ class, respectively. The optimization problem of GBSVM is given as follows: 
\begin{align}
\label{Seq:2}
    & min \hspace{0.2cm} \frac{1}{2} \|w\|^2  + C\sum_{i=1}^{p} \xi_i\nonumber \\
    & s.t. \hspace{0.2cm}  y_i(wc_i+b) - \|w\|r_i \geq 1 - \xi_i, \nonumber \\
    & \hspace{0.7cm} \xi_i \geq 0  ,\hspace{0.2cm} i=1,2 \ldots, p,
\end{align}
where $w$ and $b$ denote the normal vector and bias of the decision plane; $\xi$ represent the slack variable along with penalty coefficient $C$.
The dual of \eqref{Seq:2} is given as:
\begin{align}
    & max -\frac{1}{2} \|w\|^2  +  \sum_{i=1}^{p} \alpha_i \nonumber \\
    & \text{s.t.} \hspace{0.2cm} \sum_{i=1}^{p} \alpha_iy_i = 0, \nonumber \\
    & \hspace{0.6cm} 0 \leq \alpha_i \leq C, \hspace{0.2cm} i=1, 2, \ldots p, 
\end{align}
where $\alpha_i$'s are Lagrangian multipliers.
\subsection{Twin Support Vector Machine (TSVM)}
In TSVM \cite{khemchandani2007twin}, two non-parallel hyperplanes are generated, with each plane passing through the corresponding samples of the respective classes and maximizing the distance of the hyperplanes from samples of the other class. The optimization problem of TSVM can be written as:
\begin{align}
\label{Seq:3}
      & \underset{w_1,b_1}{min} \hspace{0.2cm} \frac{1}{2} \|Aw_1 + e_1b_1\|^2 + d_1e_2^T\xi_2 \nonumber \\
     & s.t. \hspace{0.2cm} -(Bw_1 + e_2b_1) + \xi_2 \geq e_2, \nonumber \\
     & \hspace{0.8cm} \xi_2 \geq 0,
\end{align}
and
\begin{align}
\label{eq:4}
    & \underset{w_2,b_2}{min} \hspace{0.2cm} \frac{1}{2} \|Bw_2 + e_2b_2\|^2 + d_2e_1^T\xi_1 \nonumber \\
     & s.t. \hspace{0.2cm} (Aw_2 + e_1b_2) + \xi_1 \geq e_1, \nonumber \\
     & \hspace{0.8cm} \xi_1 \geq 0,
\end{align}
where $d_1$ and $d_2$ $(> 0)$ are penalty parameters, $e_1$ and $e_2$ are column vectors of ones with appropriate dimensions, $\xi_1$ and $\xi_2$ are slack vectors, respectively. Once the optimal parameters, $i.e.$, $(w_1, b_1)$ and $(w_2, b_2)$, are obtained, a new input sample $x$ into either the $1$ ($+1$ class) or $2$ ($-1$ class) class can be labeled as follows:
\begin{align}
\label{eq:5}
    \text{class}(x) =  \underset{i \in \{1, 2\}}{\arg\min} \frac{\lvert w_i^Tx + b_i\rvert}{\|w_i\|}.
\end{align}

\renewcommand{\thefigure}{S.1}
\begin{figure}
    \centering
\includegraphics[width=0.5\textwidth,height=4cm]{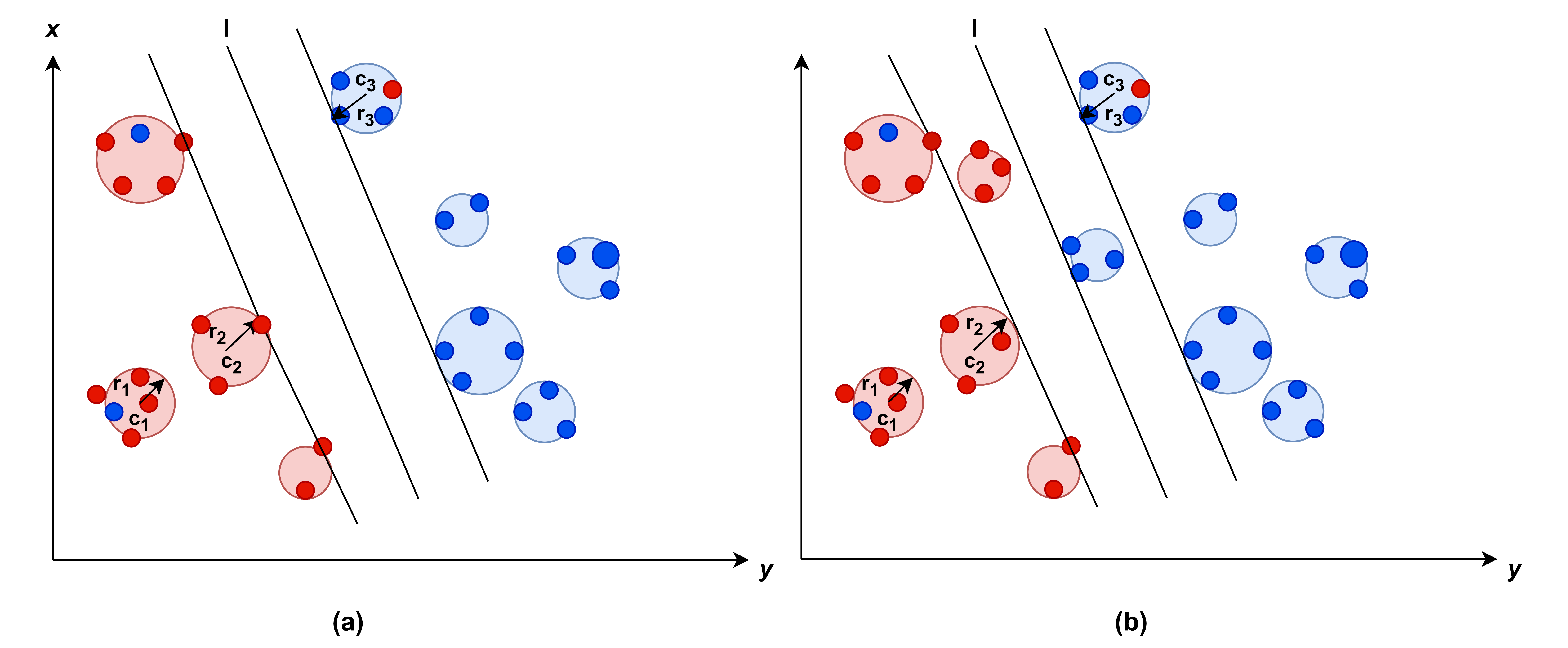}
    \caption{Schematics diagram of GBSVM, where (a) represents the separable GBSVM and (b) represents the inseparable GBSVM.}
    \label{Schematics of GBSVM}
\end{figure}

\renewcommand{\thesection}{S.II}
\section{Mathematical formulation of the proposed GBTSVM and LS-GBTSVM models for linear and non-linear cases}
In this section, we present the formulation of the proposed model, i.e., GBTSVM and LS-GBTSVM for linear and non-linear cases.
\subsection{Linear LS-GBTSVM}
The optimization problem of LS-GBTSVM for linear case is given as follows:
\begin{align}
\label{Seq:1}
       \underset{w_1,b_1,\eta_1, \xi_2}{min} & \hspace{0.2cm} \frac{1}{2} d_3(\|w_1\|^2+b_1^2) + \frac{1}{2}\eta_1^T\eta_1 + d_1e_2^T\xi_2 \nonumber \\
       s.t. & \hspace{0.2cm} C_1w_1+e_1b_1 = \eta_1, \nonumber \\
     & \hspace{0.2cm} -(C_2w_1 + e_2b_1) + \xi_2 \geq e_2+R_2, \nonumber \\
     & \hspace{0.2cm} \xi_2 \geq 0,
\end{align}
and
\begin{align}
\label{SSeq:2}
       \underset{w_2,b_2,\eta_2,\xi_1}{min} & \hspace{0.2cm} \frac{1}{2}d_4 (\|w_2\|^2 + b_2^2) + \frac{1}{2}\eta_2^T\eta_2 + d_2e_1^T\xi_1 \nonumber \\
       s.t. & \hspace{0.2cm} C_2w_2 + e_2b_2 = \eta_2, \nonumber \\
     &\hspace{0.2cm} (C_1w_2 + e_1b_2) + \xi_1 \geq e_1+R_1, \nonumber \\
     & \hspace{0.2cm} \xi_1 \geq 0.
\end{align}
The Lagrangian function \eqref{Seq:1} can be expressed as follows:
\begin{align}
\label{eq:3}
    L= &\frac{1}{2} d_3(\|w_1\|^2+b_1^2) + \frac{1}{2}\eta_1^T\eta_1 + d_1e_2^T\xi_2  \nonumber \\ 
    &+ \alpha_1^T (C_1w_1 + e_1b_1 - \eta_1)   \nonumber   \\ 
   & - \beta_1^T (-(C_2w_1+e_2b_1) + \xi_2 -e_2 - R_2)  \nonumber   \\ 
  &  - \gamma_1^T \xi_2,
\end{align}
where $\alpha_1, \beta_1, \gamma_1$ are the Lagrangian multipliers. By applying the K.K.T. conditions to equation \eqref{eq:3}, we obtain the following:
\begin{align}
    &\frac{\partial L}{\partial w_1} = d_3w_1 + C_1^T\alpha_1 + C_2^T\beta_1 = 0, \label{Seq:4} \\
   & \frac{\partial L}{\partial b_1} = d_3b_1 + e_1^T\alpha_1 + e_2^T\beta_1 = 0, \label{Seq:5} \\
   & \frac{\partial L}{\partial \xi_2} = d_1e_2 -\beta_1 -\gamma_1 = 0, \label{eq:6} \\
   & \frac{\partial L}{\partial \eta_1} = \eta_1 - \alpha_1 = 0, \label{Seq:7}
\end{align}
\begin{align}   
     & \beta_1^T(-(C_2w_1+e_2b_1) + \xi_2 -e_2 - R_2) = 0,\\
     & \gamma_1^T\xi_2 = 0, \\
      & \alpha_1^T (C_1w_1 + e_1b_1 - \eta_1) = 0.
\end{align}
Expressing \eqref{Seq:4} and \eqref{Seq:5} in a matrix representation, we obtain:
\begin{align}
\label{Seq:11}
    \begin{pmatrix}
w_1 \\
b_1 
\end{pmatrix}  = - \frac{1}{d_3} \begin{pmatrix}
C_1^T & C_2^T\\
e_1^T & e_2^T
\end{pmatrix} \begin{pmatrix}
\alpha_1 \\
\beta_1 
\end{pmatrix}.
\end{align}
From equation \eqref{Seq:1}, it can be observed that the hyperplane parameters $\binom{w_1}{b_1}$ is obtained without the need of matrix inversion. The dual of \eqref{Seq:1} is given as follows:
\begin{align}
\label{Seq:12}
    \underset{\alpha_1, \beta_1}{max} & -\frac{1}{2}\begin{pmatrix}
\alpha_1^T & \beta_1^T
\end{pmatrix} \tilde{P} \begin{pmatrix}
\alpha_1 \\
\beta_1 
\end{pmatrix} + d_3 \beta_1^T(e_2 + R_2) \nonumber \\
s.t. \hspace{0.2cm} &  d_1 e_2 - \beta_1 - \gamma_1 =  0, \nonumber \\
\text{where} \hspace{0.2cm} & \tilde{P} = \begin{pmatrix}
C_1C_1^T + d_3 I & C_1C_2^T\\
C_2C_1^T  & C_2C_2^T
\end{pmatrix} + E.
\end{align} 
Here, matrix $E$ consists entirely of ones, while $I$ denotes the identity matrix of the suitable dimension.\\
Since $\gamma_1 \geq 0$, \eqref{Seq:12} can be equivalently written as:
\begin{align}
\label{eq:13}
    \underset{\alpha_1, \beta_1}{min} \hspace{0.2cm} & \frac{1}{2}\begin{pmatrix}
\alpha_1^T & \beta_1^T
\end{pmatrix} \tilde{P} \begin{pmatrix}
\alpha_1 \\
\beta_1 
\end{pmatrix} - d_3 \beta_1^T(e_2 + R_2) \nonumber \\
s.t. \hspace{0.2cm} & 0 \leq \beta_1 \leq d_1 e_2, \nonumber \\
\text{where} \hspace{0.2cm} & \tilde{P} = \begin{pmatrix}
C_1C_1^T + d_3 I & C_1C_2^T\\
C_2C_1^T  & C_2C_2^T
\end{pmatrix} + E.
\end{align}
Similarly, $w_2$ and $b_2$ for the problem \eqref{SSeq:2} is determined as:
\begin{align}
\label{eq:36}
    \begin{pmatrix}
w_2 \\
b_2 
\end{pmatrix}  = \frac{1}{d_4} \begin{pmatrix}
C_2^T & C_1^T\\
e_2^T & e_1^T
\end{pmatrix} \begin{pmatrix}
\alpha_2 \\
\beta_2 
\end{pmatrix}.
\end{align}
The corresponding dual problems \eqref{SSeq:2} can be obtained as follows:
\begin{align}
\label{eq:15}
    \underset{\alpha_2, \beta_2}{min} \hspace{0.2cm} & \frac{1}{2}\begin{pmatrix}
\alpha_2^T & \beta_2^T
\end{pmatrix} \tilde{Q} \begin{pmatrix}
\alpha_2 \\
\beta_2 
\end{pmatrix} - d_4 \beta_2^T(e_1 + R_1) \nonumber \\
s.t. \hspace{0.2cm} & 0 \leq \beta_2 \leq d_2 e_1, \nonumber \\
\text{where} \hspace{0.2cm} & \tilde{Q} = \begin{pmatrix}
C_2C_2^T + d_4 I & C_2C_1^T\\
C_1C_2^T  & C_1C_1^T
\end{pmatrix} + E.
\end{align}

\subsection{Non-linear GBTSVM}

To construct non-linear GBTSVM, consider the mapping $x^\phi = \phi(x): \mathbb{R}^n \to \mathscr{H}$, where $\mathscr{H}$ is a Hilbert space. Define $D^\phi = \{x^\phi \hspace{0.1cm} |  \hspace{0.1cm} x \in D\}$, where $D$ is input data samples. Then the granular ball generated on the set $D^\phi$ is denoted by $S^\phi = \{((c_i^\phi, r_i^\phi), y_i), i=1,2, \ldots, p\}$, where $p$ is the number of granular balls. Matrices $C_1^\phi$ and $C_2^\phi$, along with vectors $R_1^\phi$ and $R_2^\phi$, represent the centers and radii of the $+1$ and $-1$ classes, respectively. Then, the primal problems of non-linear GBTSVM can be expressed as:
\begin{align}
\label{Seq:16}
      & \underset{w_1,b_1}{min} \hspace{0.2cm} \frac{1}{2} \|C_1^\phi w_1 + e_1b_1\|^2 + d_1e_2^T\xi_2 \nonumber \\
     & s.t. \hspace{0.2cm} -(C_2^\phi w_1 + e_2b_1) + \xi_2 \geq e_2+R_2^\phi, \nonumber \\
     & \hspace{0.8cm} \xi_2 \geq 0,
\end{align}
and 
\begin{align}
\label{Seq:17}
      & \underset{w_2,b_2}{min} \hspace{0.2cm} \frac{1}{2} \|C_2^\phi w_2 + e_2b_2\|^2 + d_2e_1^T\xi_1 \nonumber \\
     & s.t. \hspace{0.2cm} (C_1^\phi w_2 + e_1b_2) + \xi_1 \geq e_1+R_1^\phi, \nonumber \\
     & \hspace{0.8cm} \xi_1 \geq 0,
\end{align}
where $d_1, d_2 > 0$ and $\xi_1, \xi_2$ are slack vectors.

Using the K.K.T. conditions, the wolf dual problems of \eqref{Seq:16} and \eqref{Seq:17} are obtained as:
\begin{align}
\label{Seq:18}
     \underset{\alpha}{max} & \hspace{0.2cm} \alpha^T(e_2 + R_2^\phi) - \frac{1}{2} \alpha^T E(F^TF+\delta I)^{-1}E^T \alpha \nonumber \\
     s.t. & \hspace{0.2cm} 0 \leq \alpha \leq d_1e_2,
\end{align}
and
\begin{align}
\label{Seq:19}
     \underset{\gamma}{max} & \hspace{0.2cm} \gamma^T(e_1 + R_1^\phi) - \frac{1}{2} \gamma^T F(E^TE + \delta I)^{-1}F^T \gamma \nonumber \\
     s.t. & \hspace{0.2cm} 0 \leq \gamma \leq d_2e_1,
\end{align}
where $F=[C_1^\phi  \hspace{0.4cm} e_1]$ and $E=[C_2^\phi  \hspace{0.4cm} e_2]$.

Once the optimal values of $u_1=\binom{w_1}{b_1}$ and $u_2=\binom{w_2}{b_2}$ are calculated. The non-linear hyperplanes $C_1^\phi w_1 + b_1 = 0$ and $C_2^\phi w_2 + b_2 = 0$ are generated. The vectors $u_1$ and $u_2$ can be obtained as follows:
\begin{align}
\label{Seq:20}
    & u_1 = -(F^TF+\delta I)^{-1}E^T\alpha  \hspace{0.4cm} \text{and} \nonumber \\ 
    &  u_2 =(E^TE+\delta I)^{-1}F^T\gamma,
\end{align}
were $\delta$ is a positive small value used to handle situations involving singular matrices.

\subsection{Non-linear LS-GBTSVM}
The non-linear LS-GBTSVM comprises the following pair of constrained minimization problems:
\begin{align}
\label{Seq:21}
       \underset{w_1,b_1,\eta_1, \xi_2}{min} & \hspace{0.2cm} \frac{1}{2} d_3(\|w_1\|^2+b_1^2) + \frac{1}{2}\eta_1^T\eta + d_1e_2^T\xi_2 \nonumber \\
       s.t. & \hspace{0.2cm} C_1^\phi w_1+e_1b_1 = \eta_1, \nonumber \\
     & \hspace{0.2cm} -(C_2^\phi w_1 + e_2b_1) + \xi_2 \geq e_2+R_2^\phi, \nonumber \\
     & \hspace{0.2cm} \xi_2 \geq 0,
\end{align}
and
\begin{align}
\label{Seq:22}
       \underset{w_2,b_2,\eta_2,\xi_1}{min} & \hspace{0.2cm} \frac{1}{2}d_4 \|w_2\|^2 + b_2^2) + \frac{1}{2}\eta_2^T\eta_2 + d_2e_1^T\xi_1 \nonumber \\
       s.t. & \hspace{0.2cm} C_2^\phi w_2 + e_2b_2 = \eta_2, \nonumber \\
     &\hspace{0.2cm} (C_1^\phi w_2 + e_1b_2) + \xi_1 \geq e_1+R_1^\phi, \nonumber \\
     & \hspace{0.2cm} \xi_1 \geq 0,
\end{align}
The dual formulation and solutions of the problem \eqref{Seq:21} and \eqref{Seq:22} can be calculated in a similar way as in the linear case.

\renewcommand{\thesection}{S.III}
\section{Discussion of the proposed GBTSVM and LS-GBTSVM models}
In this section, the advantages and limitations of the proposed GBTSVM and LS-GBTSVM are discussed in detail.\\
\noindent \underline{GBTSVM:} The proposed GBTSVM utilizes granular balls as inputs for classifier construction, offering enhanced robustness, resilience to resampling, and computational efficiency compared to the standard SVM and GBSVM. \\The detailed advantages of the proposed GBTSVM model are given as follows:
\begin{enumerate}
     \item The efficacy of the proposed GBTSVM model becomes particularly prominent under conditions of elevated noise levels. This can be attributed to their utilization of granular balls as units instead of individual sample points; this characteristic contributes to the models’ ability to navigate and minimize the effects of noise, emphasizing their robust performance in the face of such challenges.
    \item The efficiency of the GBTSVM is significantly elevated by using the centers of granular balls rather than all the samples of the entire granular ball.
    \item The proposed GBTSVM model effectively captures intricate data patterns and complex relationships through non-linear transformations in the kernel space and elevates its performance.
\end{enumerate}
Limitations of the proposed GBTSVM model are as follows:
\begin{enumerate}
    \item Performing a matrix inversion computation within the Wolf-dual formulation becomes costly when dealing with a large dataset.
    \item GBTSVM does not incorporate the SRM principle in its formulation, which leads to an elevated risk of overfitting. 
\end{enumerate}
\underline{LS-GBTSVM:} The proposed LS-GBTSVM model shares fundamental characteristics with GBTSVM while also offering additional advantages. The distinct benefits of the LS-GBTSVM model are delineated as follows:
\begin{enumerate}
    \item The optimization problem of LS-GBTSVM eliminates the need for matrix inversions, streamlining the LS-GBTSVM’s computational efficiency.
    \item We incorporate the SRM principle through the incorporation of regularization terms, effectively addressing the issue of overfitting.
    \item The LS-GBTSVM model showcases efficiency, scalability for large datasets, and robustness against noise and outliers. It achieves this through the utilization of granular balls as inputs for classifier construction, enhancing resilience to resampling and computational efficiency compared to standard SVM and GBSVM approaches.
\end{enumerate}
Limitations of the proposed LS-GBTSVM model are as follows:
\begin{enumerate}
    \item External package ``CVXOPT" is needed to be employed to solve the dual of the QPPs arising in the LS-GBTSVM model, utilizing the ``qp-solvers" function. 
\end{enumerate}

\renewcommand{\thesection}{S.VI}
\section{Sensitivity Analysis}
We conduct sensitivity analyses on several key hyperparameters of the proposed GBTSVM and LS-GBTSVM models. These analyses encompassed various factors, including the granular ball parameters $num$ and $pur$, explored in subsection \ref{Sensitivity Analysis of Granular ball Parameters}. Additionally, we examined the effects of different levels of label noise in subsection \ref{Sensitivity Analysis of Label Noise}. Furthermore, the impact of hyperparameters $\sigma$ and $pur$ is investigated, detailed in subsection \ref{Sensitivity Analysis of Hyperparameter sigma}. We also assess the influence of hyperparameters $d_1$ and $d_2$, discussed in subsection \ref{Sensitivity Analysis of Hyperparameters d1}. Finally, we analyzed the relationship between the number of granular balls generated and the resulting accuracy (ACC) of the proposed GBTSVM model across varying purities, presented in subsection \ref{Evaluation on Datasets with Purities}.

\subsection{Sensitivity Analysis of GB Parameters $num$ and $pur$}
\label{Sensitivity Analysis of Granular ball Parameters}
In granular ball computing, we denote `$num$' to be the minimum number of granular balls required to be generated from the training dataset. Since we are addressing a binary classification problem, the threshold value for `$num$' is set to be $2$. The purity ($pur$) of a granular ball stands as a pivotal characteristic. By adjusting the purity level of the granular balls, we can refine how data points are spread out in space, effectively capturing their distribution. To scrutinize the influence of `$num$' and `$pur$' on the generalization performance of the proposed GBTSVM and LS-GBTSVM models, we vary `$num$' within the set $\{2, 3, 4, 5\}$, and `$pur$' within the range $\{1, 0.97, 0.94, 0.91, 0.88, 0.85, 0.82, 0.79\}$. Figure \ref{effect of parameter pur and num} provides insightful visualizations illustrating the impact of these granular parameters on the performance of the proposed GBTSVM and LS-GBTSVM models. In the scenario depicted in Figure \ref{effect of parameter pur and num}, a simultaneous increase in both `$pur$' and `$num$' results in a notable elevation of ACC. As the `$pur$' increases, these granular balls undergo further division, leading to an augmented generation of granular balls. This process effectively captures the underlying data patterns, resulting in optimal generalization performance. Therefore, careful selection of model hyperparameters is crucial for achieving optimal performance in the proposed GBTSVM and LS-GBTSVM models.
\renewcommand{\thefigure}{S.2}
\begin{figure*}
\begin{minipage}{.246\linewidth}
\centering
\subfloat[aus (GBTSVM)]{\includegraphics[scale=0.20]{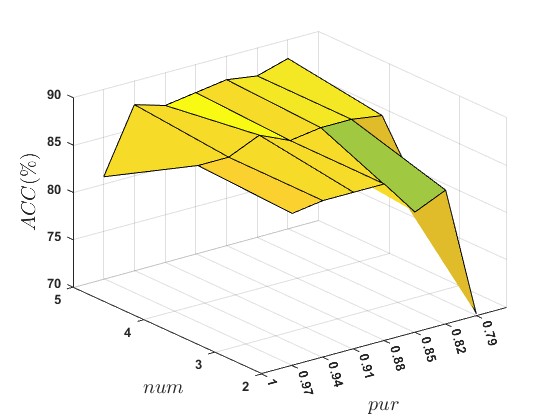}}
\end{minipage}
\begin{minipage}{.246\linewidth}
\centering
\subfloat[heart-stat (GBTSVM)]{\includegraphics[scale=0.20]{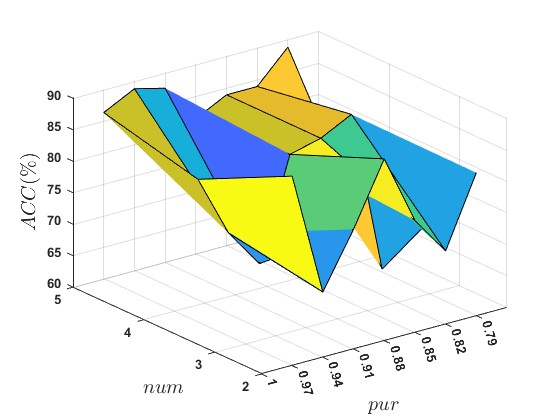}}
\end{minipage}
\begin{minipage}{.246\linewidth}
\centering
\subfloat[aus (LS-GBTSVM)]{\includegraphics[scale=0.20]{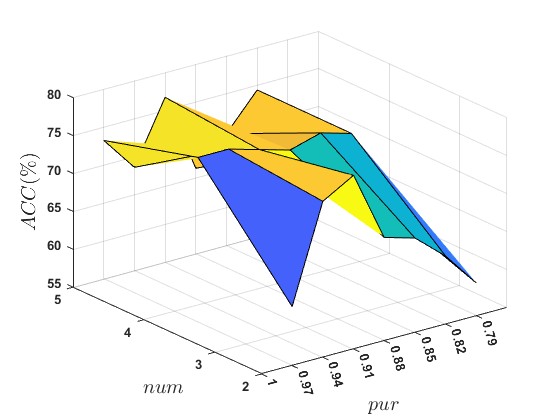}}
\end{minipage}
\begin{minipage}{.246\linewidth}
\centering
\subfloat[heart-stat (LS-GBTSVM)]{\includegraphics[scale=0.20]{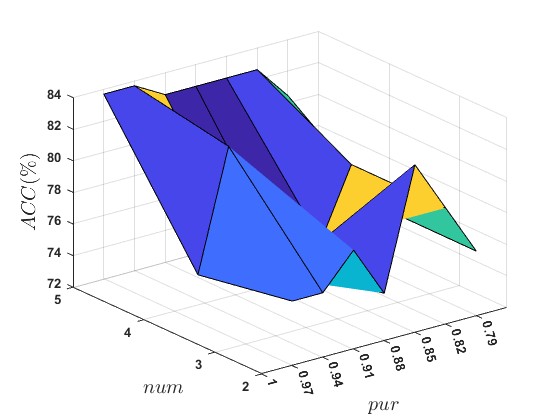}}
\end{minipage}
\caption{Effect of granular ball parameters $pur$ and $num$ on the performance of the proposed GBTSVM model with linear kernel.}
\label{effect of parameter pur and num}
\end{figure*}
\vspace{-0.2cm}
\subsection{Sensitivity Analysis of Label Noise}
\label{Sensitivity Analysis of Label Noise}
One of the focuses of the proposed GBTSVM and LS-GBTSVM models is to reduce the detrimental effect of noise. The resilience of the proposed GBTSVM and LS-GBTSVM models is demonstrated under various levels of noisy labels. Analyzing Figure \ref{The performance of the proposed GBTSVM model with linear kernel with the different labels of noise.}, it is evident that the performance of baseline models fluctuates significantly and declines with variations in noise labels. In contrast, the proposed GBTSVM and LS-GBTSVM models demonstrate consistent and superior performance despite changes in noise levels. This resilience and adaptability to mislabeled instances can be attributed to the coarser granularity inherent in granular balls, effectively mitigating the influence of label noise points within them. The assignment of a label to a granular ball is predominantly influenced by the prevailing label within it, and the presence of label noise points associated with minority labels does not exert a substantial influence on the determination of the granular ball's label. This characteristic contributes to the models' ability to navigate and minimize the effects of label noise, emphasizing their robust performance in the face of such challenges.
\renewcommand{\thefigure}{S.3}
\begin{figure*}
\begin{minipage}{.246\linewidth}
\centering
\subfloat[aus]{\includegraphics[scale=0.20]{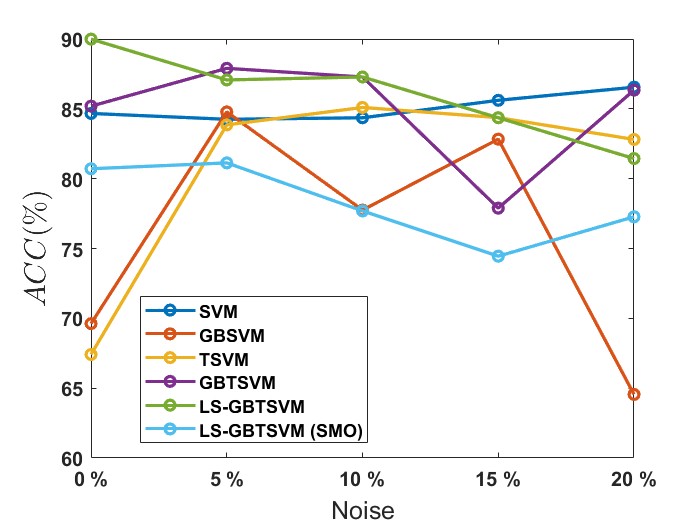}}
\end{minipage}
\begin{minipage}{.246\linewidth}
\centering
\subfloat[ecoli-0-1-4-6\_vs\_5]{\includegraphics[scale=0.20]{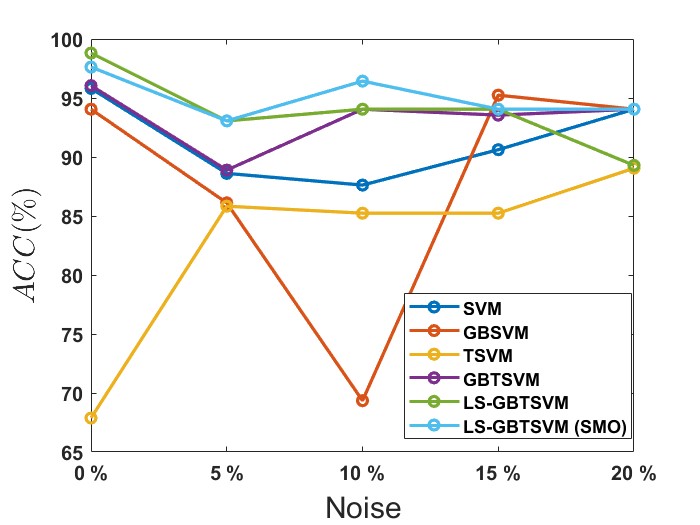}}
\end{minipage}
\begin{minipage}{.246\linewidth}
\centering
\subfloat[ozone]{\includegraphics[scale=0.20]{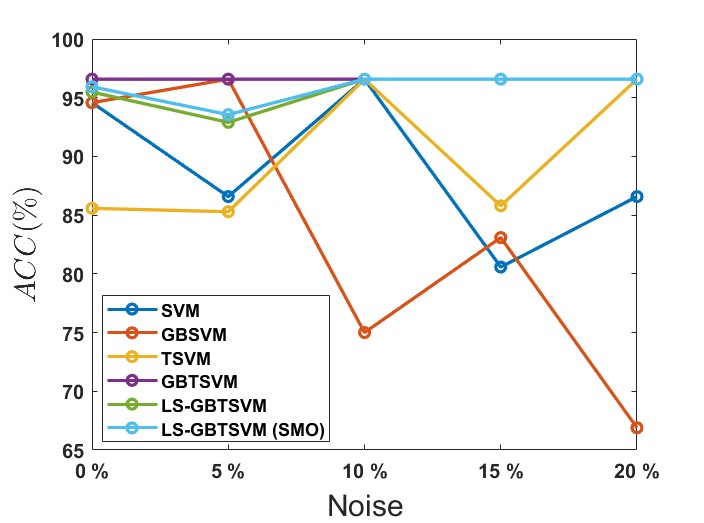}}
\end{minipage}
\begin{minipage}{.246\linewidth}
\centering
\subfloat[yeast-2\_vs\_4]{\includegraphics[scale=0.20]{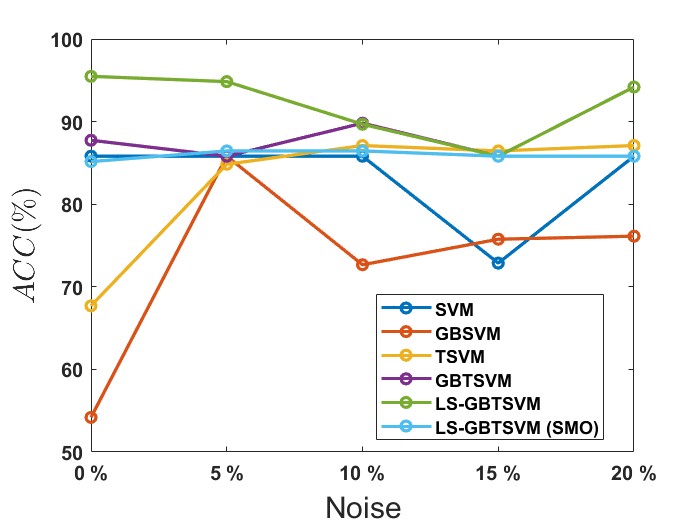}}
\end{minipage}
\caption{Effect of different labels of noise on the performance of the proposed GBTSVM and LS-GBTSVM models with linear kernel.}
\label{The performance of the proposed GBTSVM model with linear kernel with the different labels of noise.}
\end{figure*}
\vspace{-0.2cm}
\renewcommand{\thetable}{S.I}
\begin{table*}[htp!]
\centering
    \caption{The number of granular balls and the corresponding accuracies by the proposed GBTSVM model under different purities.}
    \label{The number and results of granular ball generated by proposed GBTSVM model under different purities.}
   \resizebox{0.9\linewidth}{!}{
\begin{tabular}{cccccccc}
\hline
$pur$ & $1$ & $0.97$ & $0.94$ & $0.91$ & $0.88$ & $0.85$ & $0.82$ \\ \hline
{Dataset} & ACC & ACC & ACC & ACC & ACC & ACC & ACC \\
 & $number(GB)$ & $number(GB)$ & $number(GB)$ & $number(GB)$ & $number(GB)$ & $number(GB)$ & $number(GB)$ \\ \hline
checkerboard\_Data & 86.06 & 86.06 & 86.54 & 85.5 & 85.02 & 85.02 & 86.06 \\
 & 76 & 72 & 68 & 52 & 51 & 26 & 15 \\
chess\_krvkp & 85.19 & 82.78 & 80.14 & 85.05 & 84.89 & 84.49 & 84.82 \\
 & 372 & 357 & 332 & 302 & 276 & 241 & 224 \\
mammographic & 83.04 & 81.31 & 79.93 & 82.16 & 82.58 & 74.78 & 82.06 \\
 & 103 & 101 & 85 & 76 & 69 & 12 & 10 \\
spambase & 89.79 & 89.44 & 88.57 & 87.78 & 86.6 & 89.47 & 89.02 \\
 & 184 & 154 & 125 & 87 & 80 & 76 & 54 \\
tic\_tac\_toe & 99.65 & 99.65 & 99.65 & 99.65 & 99.65 & 99.65 & 99.65 \\
 & 108 & 101 & 93 & 91 & 80 & 74 & 70 \\ \hline
 \multicolumn{8}{l}{ACC refers to accuracy, $number(GB)$ indicates the number of granular balls, and $pur$ represents the purity.}
\end{tabular}}
\end{table*}
\vspace{-0.2cm}
\renewcommand{\thetable}{S.II}
\begin{table*}[htp!]
\centering
    \caption{The number of granular balls and the corresponding accuracies by the proposed GBTSVM model under different purities.}
    \label{The number and results of granular ball generated by proposed LS-GBTSVM model under different purities.}
  \resizebox{0.9\linewidth}{!}{
\begin{tabular}{cccccccc}
\hline
$pur$ & $1$ & $0.97$ & $0.94$ & $0.91$ & $0.88$ & $0.85$ & $0.82$ \\ \hline
{Dataset} & ACC & ACC & ACC & ACC & ACC & ACC & ACC \\
 & $number(GB)$ & $number(GB)$ & $number(GB)$ & $number(GB)$ & $number(GB)$ & $number(GB)$ & $number(GB)$ \\ \hline
checkerboard\_Data & 87.5 & 87.5 & 86.05 & 86.05 & 88.94 & 82.69 & 75.48 \\
 & 76 & 72 & 68 & 52 & 51 & 26 & 15 \\
chess\_krvkp & 91.86 & 91.24 & 90.92 & 92.59 & 91.34 & 91.65 & 89.25 \\
 & 372 & 357 & 332 & 302 & 276 & 241 & 224 \\
mammographic & 80.27 & 79.93 & 76.81 & 80.62 & 78.54 & 74.39 & 75.39 \\
 & 103 & 101 & 85 & 76 & 69 & 12 & 10 \\
spambase & 89.35 & 89.06 & 88.99 & 83.92 & 87.76 & 89.57 & 90.07 \\
 & 184 & 154 & 125 & 87 & 80 & 76 & 54 \\
tic\_tac\_toe & 95.16 & 95.51 & 95.51 & 94.65 & 94.75 & 95.65 & 95.16 \\
 & 108 & 101 & 93 & 91 & 80 & 74 & 70 \\ \hline
   \multicolumn{8}{l}{ACC refers to accuracy, $number(GB)$ indicates the number of granular balls, and $pur$ represents the purity.}
\end{tabular}}
\end{table*}
\subsection{Sensitivity Analysis of Hyperparameter $\sigma$ and $pur$}
\label{Sensitivity Analysis of Hyperparameter sigma}
Here, the performance of the proposed GBTSVM model is evaluated by varying the values of $\sigma$ and $pur$. This thorough exploration enables us to pinpoint the configuration that maximizes predictive ACC and fortifies the model's resilience when confronted with unseen data. Figure \ref{Effect of parameters sigma and pur} illustrates a discernible fluctuation in the model's ACC across a spectrum of $\sigma$ and $pur$ values, underscoring the sensitivity of our model's performance to these specific hyperparameters.

According to the findings presented in Figure \ref{Effect of parameters sigma and pur}, optimal performance of the proposed model is observed within the $\sigma$ ranges of $2^1$ to $2^5$ and $2^{-5}$ to $2^{-3}$. Similarly, Figure \ref{Effect of parameters sigma and pur}d illustrates an increase in testing ACC within the $\sigma$ range spanning from $2^{-3}$ to $2^{3}$. These results suggest that, when considering the parameters $\sigma$ and $pur$, the performance of the model is predominantly influenced by $\sigma$ rather than $pur$. This underscores the significance of kernel space and the effective extraction of nonlinear features in the proposed GBTSVM model. Consequently, it is recommended that careful attention be given to the selection of the hyperparameter $\sigma$ in GBTSVM models to attain superior generalization performance.
\vspace{-0.2cm}
\subsection{Sensitivity Analysis of Hyperparameters $d_1$ and $d_2$}
\label{Sensitivity Analysis of Hyperparameters d1}
We examine the role of the hyperparameters $d_1$ and $d_2$'s impact on the overall predictive capability of the proposed GBTSVM model. Figure \ref{Effect of parameters} shows sensitivity analysis on KEEL and UCI datasets.
The ACC is evaluated by varying the parameters $d_1$ and $d_2$. It can be noticed that as the values of $d_1$ and $d_2$ rise, the ACC also demonstrates an increase. Once a specific threshold is surpassed, the ACC reaches a maximum, signifying that additional increments in $d_1$ and $d_2$ beyond $10^{-2}$ result in diminishing improvements in testing ACC. As a result, it is crucial to meticulously select the hyperparameters for the proposed GBTSVM and LS-GBTSVM models in order to achieve the best possible generalization performance.
\renewcommand{\thefigure}{S.4}
\begin{figure*}[ht!]
\begin{minipage}{.246\linewidth}
\centering
\subfloat[crossplane130]{\includegraphics[scale=0.20]{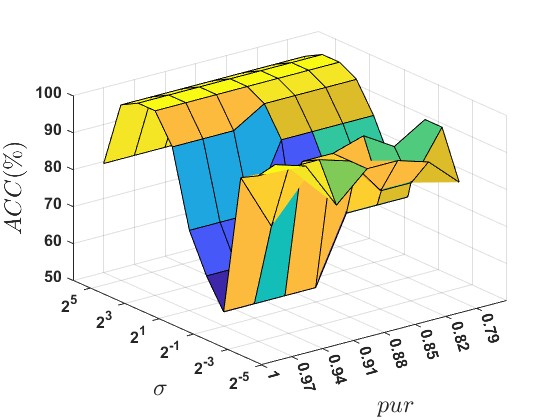}}
\end{minipage}
\begin{minipage}{.246\linewidth}
\centering
\subfloat[heart-stat]{\includegraphics[scale=0.20]{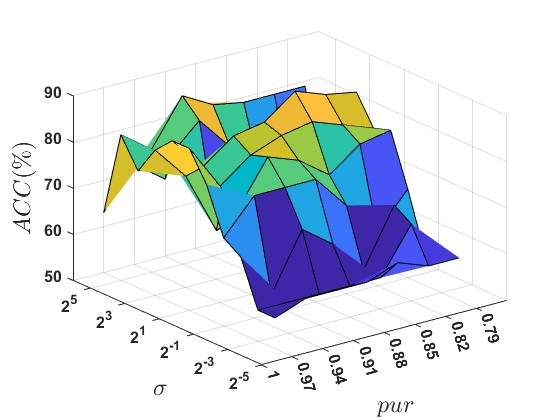}}
\end{minipage}
\begin{minipage}{.246\linewidth}
\centering
\subfloat[mammographic]{\includegraphics[scale=0.20]{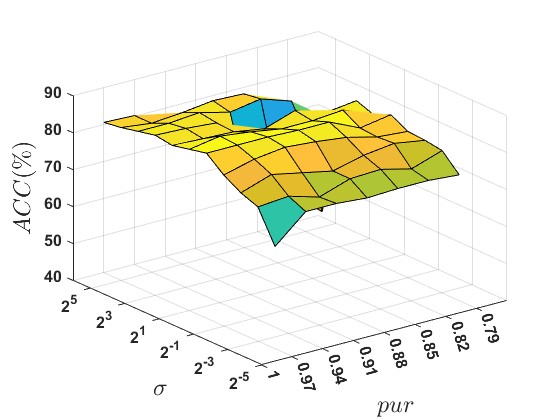}}
\end{minipage}
\begin{minipage}{.246\linewidth}
\centering
\subfloat[mushroom]{\includegraphics[scale=0.20]{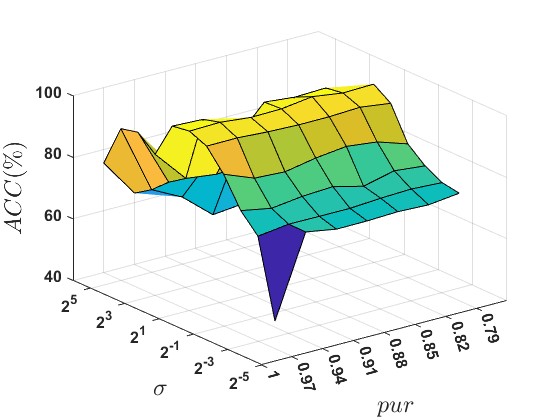}}
\end{minipage}
\caption{Effect of parameters $pur$ and $\sigma$ on the performance of the proposed GBTSVM model with linear kernel.}
\label{Effect of parameters sigma and pur}
\end{figure*}
\renewcommand{\thefigure}{S.5}
\begin{figure*}[ht!]
\begin{minipage}{.246\linewidth}
\centering
\subfloat[checkerboard\_Data (GBTSVM)]{\label{fig:1s1}\includegraphics[scale=0.20]{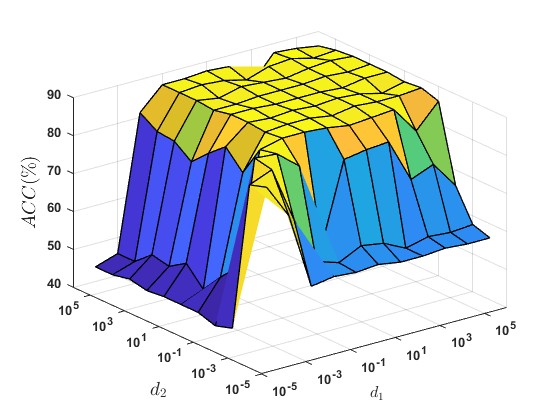}}
\end{minipage}
\begin{minipage}{.246\linewidth}
\centering
\subfloat[yeast-0-2-5-6\_vs\_3-7-8-9 (GBTSVM)]{\label{fig:1sb}\includegraphics[scale=0.20]{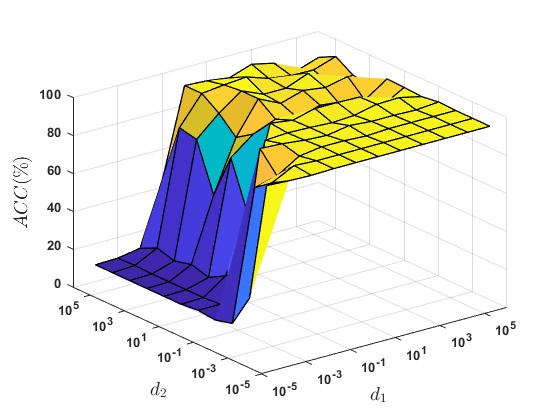}}
\end{minipage}
\begin{minipage}{.246\linewidth}
\centering
\subfloat[checkerboard\_Data (LS-GBTSVM)]{\label{fig:1sa}\includegraphics[scale=0.20]{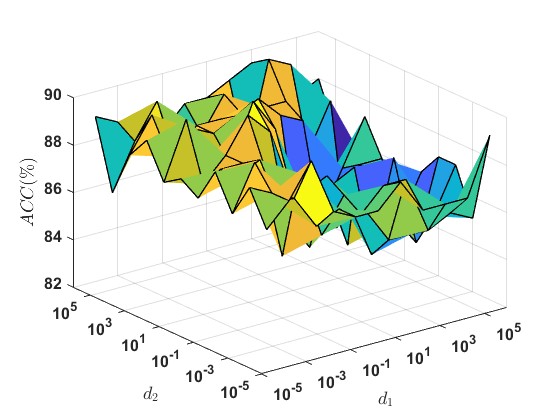}}
\end{minipage}
\begin{minipage}{.246\linewidth}
\centering
\subfloat[yeast-0-2-5-6\_vs\_3-7-8-9 (LS-GBTSVM)]{\label{fig:1s}\includegraphics[scale=0.20]{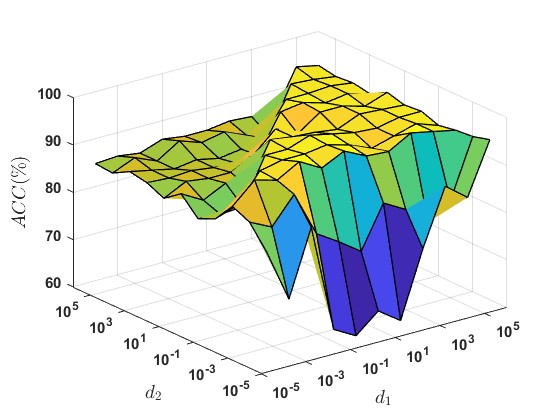}}
\end{minipage}
\caption{Effect of parameters $d_1$ and $d_2$ on the performance of the proposed GBTSVM and LS-GBTSVM models with linear kernel.}
\label{Effect of parameters}
\end{figure*}

\subsection{The Number of Granular Balls Generated and the Resulting ACC of the Proposed GBTSVM and LS-GBTSVM Models with Different Purities}
\label{Evaluation on Datasets with Purities}
In this subsection, we investigate the relationship between the number of granular balls generated and the resulting ACC of the proposed GBTSVM and LS-GBTSVM models across different purities. We select five diverse UCI and KEEL datasets to demonstrate the performance of the proposed GBTSVM and LS-GBTSVM models. Table \ref{The number and results of granular ball generated by proposed GBTSVM model under different purities.} and \ref{The number and results of granular ball generated by proposed LS-GBTSVM model under different purities.} shows the number of granular balls and the corresponding ACC achieved by GBTSVM and LS-GBTSVM models across various purities ranging from $0.82$ to $1.0$. Based on the data presented in Table \ref{The number and results of granular ball generated by proposed GBTSVM model under different purities.} and \ref{The number and results of granular ball generated by proposed LS-GBTSVM model under different purities.}, it is evident that a decrease in purity value correlates with a reduction in the number of granular balls generated. Remarkably, despite these variations in granular ball count, the ACC of the proposed GBTSVM and LS-GBTSVM models consistently remains within a stable range. This observation underscores the model's inherent adaptability, as it intelligently adjusts its granular ball generation in response to changes in purity thresholds without compromising performance. 
Such adaptability highlights the robustness of the proposed GBTSVM and LS-GBTSVM models framework, emphasizing its enhanced versatility and adeptness across varying conditions.
\vspace{-0.2cm}
\section{Evaluation on UCI and KEEL Datasets}
In this section, the performance of the proposed GBTSVM and LS-GBTSVM models, along with the baseline models, are presented in Table \ref{Classification performance in Linear Case.} and Table \ref{Classification performance in NonLinear Case.} for the linear and non-linear cases, respectively, with and without the label of noise.

\renewcommand{\thetable}{S.III}
\begin{table*}[ht!]
\centering
    \caption{Performance comparison of the proposed GBTSVM and LS-GBTSVM along with the baseline models over for UCI and KEEL datasets with linear kernel.}
    \label{SClassification performance in Linear Case.}
    \resizebox{0.8\linewidth}{!}{
\begin{tabular}{cccccccc}
\hline
Dataset & Noise & \begin{tabular}[c]{@{}c@{}}SVM \cite{cortes1995support}\\ ACC (\%)\\ $(d_1)$\end{tabular} & \begin{tabular}[c]{@{}c@{}}GBSVM (PSO) \cite{xia2022gbsvm} \\ ACC (\%)\\ $(d_1)$\end{tabular} & \begin{tabular}[c]{@{}c@{}}TSVM \cite{khemchandani2007twin} \\ ACC (\%)\\ $(d_1, d_2)$\end{tabular} & \begin{tabular}[c]{@{}c@{}}GBTSVM$^{\dagger}$\\ ACC (\%)\\ $(d_1, d_2)$\end{tabular} & \begin{tabular}[c]{@{}c@{}}LS-GBTSVM$^{\dagger}$\\ ACC (\%)\\ $(d_1, d_3)$\end{tabular} & \begin{tabular}[c]{@{}c@{}}LS-GBTSVM (SOR)$^{\dagger}$\\ ACC (\%)\\ $(d_1, d_3)$\end{tabular} \\ \hline
{\begin{tabular}[c]{@{}c@{}}aus\\ (690 x 15)\end{tabular}} & 0 \% & 88.46 & 81.11 & 64.31 & 89.06 & 90.38 & 90.38 \\
 &  & $(10^{-3})$ & $(10^{4})$ & $(10^{-1}, 10^{-1})$ & $(10^{-5}, 10^{-5})$ & $(10^{-5}, 10^{-5})$ & $(10^{-5}, 10^{-5})$ \\
 & 5 \% & 86.46 & 81.11 & 86.46 & 87.02 & 88.94 & 87.98 \\
 &  & $(10^{-3})$ & $(10^{5})$ & $(10^{-1}, 10^{-1})$ & $(10^{-5}, 10^{-5})$ & $(10^{-3}, 10^{-3})$ & $(10^{-5}, 10^{-3})$ \\
 & 10 \% & 85.47 & 82.94 & 87.5 & 87.98 & 85.58 & 87.02 \\
 &  & $(10^{-2})$ & $(10^{4})$ & $(10^{-1}, 10^{-1})$ & $(10^{-5}, 10^{-5})$ & $(10^{-4}, 10^{-3})$ & $(10^{-5}, 10^{-5})$ \\
 & 15 \% & 89.42 & 86.54 & 86.06 & 86.54 & 86.54 & 87.02 \\
 &  & $(10^{-2})$ & $(10^{-5})$ & $(10^{-1}, 1)$ & $(10^{-5}, 10^{-5})$ & $(10^{-4}, 10^{-3})$ & $(10^{-5}, 10^{-3})$ \\
 & 20 \% & 82.94 & 86.54 & 87.02 & 87.98 & 83.65 & 84.13 \\
 &  & $(10^{-3})$ & $(10^{-5})$ & $(1, 1)$ & $(10^{-5}, 10^{-5})$ & $(10^{4}, 10^{-5})$ & $(10^{-5}, 10^{-3})$ \\ \hline
{\begin{tabular}[c]{@{}c@{}}breast\_cancer\\ (286 x 10)\end{tabular}} & 0 \% & 72.09 & 62.79 & 60 & 73.26 & 67.44 & 73.26 \\
 &  & $(10^{-2})$ & $(10^{-5})$ & $(10^{-2}, 10^{-1})$ & $(10^{-5}, 10^{-5})$ & $(10^{-5}, 10^{-5})$ & $(10^{-5}, 10^{-5})$ \\
 & 5 \% & 70.93 & 80 & 65.12 & 72.09 & 68.14 & 60.47 \\
 &  & $(10^{-2})$ & $(10^{5})$ & $(10, 10^{2})$ & $(10^{-5}, 10^{-5})$ & $(10^{-5}, 10^{-5})$ & $(10^{-5}, 10^{-5})$ \\
 & 10 \% & 68.6 & 74.42 & 62.09 & 66.28 & 68.84 & 70.93 \\
 &  & $(10^{-2})$ & $(10^{2})$ & $(10^{-5}, 10^{-5})$ & $(10^{-5}, 10^{-5})$ & $(10^{-1}, 10^{-5})$ & $(10^{-5}, 10^{-5})$ \\
 & 15 \% & 68.6 & 54.65 & 75.58 & 74.42 & 74.42 & 74.42 \\
 &  & $(10^{-2})$ & $(10^{-5})$ & $(10^{-5}, 10^{-5})$ & $(10^{-5}, 10^{-5})$ & $(10^{-5}, 10^{-5})$ & $(10^{-5}, 10^{-5})$ \\
 & 20 \% & 65.12 & 59.3 & 74.42 & 70.93 & 67.44 & 60.47 \\
 &  & $(10^{-2})$ & $(10^{3})$ & $(10, 1)$ & $(10^{-5}, 10^{-5})$ & $(10^{-5}, 10^{-5})$ & $(10^{-5}, 10^{-4})$ \\ \hline
{\begin{tabular}[c]{@{}c@{}}checkerboard\_Data\\ (690 x 15)\end{tabular}} & 0 \% & 85.46 & 71.06 & 64.31 & 86.06 & 90.38 & 90.38 \\
 &  & $(10^{-3})$ & $(10^{4})$ & $(10^{-1}, 10^{-1})$ & $(10^{-2}, 10^{-2})$ & $(10^{-5}, 10^{-5})$ & $(10^{-5}, 10^{-5})$ \\
 & 5 \% & 78.46 & 81.11 & 88.46 & 87.02 & 88.94 & 87.98 \\
 &  & $(10^{-3})$ & $(10^{5})$ & $(10^{-1}, 10^{-1})$ & $(10^{-5}, 10^{-5})$ & $(10^{-3}, 10^{-3})$ & $(10^{-5}, 10^{-3})$ \\
 & 10 \% & 85.57 & 83.94 & 87.5 & 87.98 & 85.58 & 87.02 \\
 &  & $(10^{-2})$ & $(10^{4})$ & $(10^{-1}, 10^{-1})$ & $(10^{-5}, 10^{-5})$ & $(10^{-4}, 10^{-3})$ & $(10^{-5}, 10^{-5})$ \\
 & 15 \% & 89.42 & 83.94 & 86.06 & 86.54 & 86.54 & 87.02 \\
 &  & $(10^{-2})$ & $(10^{-5})$ & $(10^{-1}, 1)$ & $(10^{-5}, 10^{-5})$ & $(10^{-4}, 10^{-3})$ & $(10^{-5}, 10^{-3})$ \\
 & 20 \% & 88.94 & 87.5 & 87.02 & 87.98 & 83.65 & 84.13 \\
 &  & $(10^{-3})$ & $(10^{-5})$ & $(1, 1)$ & $(10^{-5}, 10^{-5})$ & $(10^{4}, 10^{-5})$ & $(10^{-5}, 10^{-3})$ \\ \hline
 {\begin{tabular}[c]{@{}c@{}}chess\_krvkp\\ (3196 x 37)\end{tabular}} & 0 \% & 84.67 & 69.62 & 67.41 & 85.19 & 89.99 & 80.71 \\
 &  & $(10^{-2})$ & $(10^{4})$ & $(10^{2}, 1)$ & $(1, 10^{-1})$ & $(10^{-4}, 10^{-3})$ & $(10^{-5}, 10^{-3})$ \\
 & 5 \% & 84.25 & 84.78 & 83.85 & 87.9 & 87.07 & 81.13 \\
 &  & $(10^{4})$ & $(10^{-3})$ & $(10, 10)$ & $(1, 10)$ & $(10^{-2}, 10^{-3})$ & $(10^{-5}, 10^{-3})$ \\
 & 10 \% & 84.36 & 77.75 & 85.1 & 87.28 & 87.28 & 77.69 \\
 &  & $(10^{3})$ & $(10^{5})$ & $(10^{-1}, 10^{-1})$ & $(1, 1)$ & $(10^{3}, 10^{-4})$ & $(10^{-5}, 10^{-4})$ \\
 & 15 \% & 85.61 & 82.83 & 84.37 & 77.89 & 84.36 & 74.45 \\
 &  & $(10^{-1})$ & $(10^{4})$ & $(10^{-3}, 10^{-3})$ & $(1, 10^{-1})$ & $(1, 10^{-5})$ & $(10^{-5}, 10^{-4})$ \\
 & 20 \% & 86.55 & 64.55 & 82.81 & 86.34 & 81.44 & 77.27 \\
 &  & $(10^{-3})$ & $(10^{5})$ & $(10^{-5}, 10^{-5})$ & $(10, 1)$ & $(10^{4}, 10^{-5})$ & $(10^{-5}, 10^{-3})$ \\ \hline
 {\begin{tabular}[c]{@{}c@{}}crossplane130\\ (130 x 3)\end{tabular}} & 0 \% & 97.24 & 100 & 71.35 & 97.44 & 100 & 82.05 \\
 &  & $(10^{-2})$ & $(1)$ & $(10^{-5}, 10^{-5})$ & $(10^{-4}, 10^{-5})$ & $(10^{-3}, 10^{-3})$ & $(10^{-5}, 10^{-5})$ \\
 & 5 \% & 97.44 & 97.44 & 100 & 100 & 100 & 97.44 \\
 &  & $(10^{-1})$ & $(10^{-5})$ & $(10^{-4}, 10^{-5})$ & $(10^{-4}, 10^{-5})$ & $(10^{-3}, 10^{-3})$ & $(10^{-5}, 10^{-5})$ \\
 & 10 \% & 100 & 97.44 & 100 & 100 & 94.87 & 94.87 \\
 &  & $(10^{-2})$ & $(10^{-5})$ & $(10^{-5}, 10^{-5})$ & $(10^{-4}, 10^{-5})$ & $(10^{-3}, 10^{-5})$ & $(10^{-5}, 10^{-4})$ \\
 & 15 \% & 100 & 100 & 100 & 100 & 97.44 & 89.74 \\
 &  & $(10^{-2})$ & $(10)$ & $(10^{-5}, 10^{-5})$ & $(10^{-5}, 10^{-5})$ & $(10^{-3}, 10^{-3})$ & $(10^{-5}, 10^{-4})$ \\
 & 20 \% & 97.44 & 97.44 & 100 & 100 & 97.44 & 97.44 \\
 &  & $(10^{-2})$ & $(10^{-5})$ & $(10, 10^{3})$ & $(10^{-4}, 10^{-5})$ & $(10^{-3}, 10^{-5})$ & $(10^{-5}, 10^{-5})$ \\ \hline
{\begin{tabular}[c]{@{}c@{}}ecoli-0-1\_vs\_2-3-5\\ (244 x 8)\end{tabular}} & 0 \% & 81.89 & 77.3 & 68.24 & 85.14 & 93.24 & 89.19 \\
 &  & $(10^{-2})$ & $(10^{-5})$ & $(1, 1)$ & $(10^{-4}, 10^{-5})$ & $(10^{-3}, 10^{-3})$ & $(10^{-5}, 10^{-5})$ \\
 & 5 \% & 87.84 & 57.3 & 90.46 & 90.54 & 59.46 & 89.19 \\
 &  & $(10^{-2})$ & $(10)$ & $(1, 1)$ & $(10^{-4}, 10^{-5})$ & $(10^{-5}, 10^{-5})$ & $(10^{-5}, 10^{-5})$ \\
 & 10 \% & 86.49 & 81.08 & 83.24 & 89.19 & 93.24 & 85.14 \\
 &  & $(10^{-2})$ & $(10)$ & $(1, 1)$ & $(10^{-5}, 10^{-5})$ & $(10^{-3}, 10^{-3})$ & $(10^{-5}, 10^{-5})$ \\
 & 15 \% & 82.43 & 79.73 & 81.89 & 94.05 & 79.73 & 89.19 \\
 &  & $(10^{-2})$ & $(10)$ & $(1, 1)$ & $(10^{-5}, 10^{-5})$ & $(10^{-5}, 10^{-5})$ & $(10^{-5}, 10^{-5})$ \\
 & 20 \% & 87.84 & 55.41 & 81.89 & 89.19 & 83.78 & 93.24 \\
 &  & $(10^{-2})$ & $(10^{-5})$ & $(10^{-5}, 10^{-5})$ & $(10^{-5}, 10^{-5})$ & $(10^{-5}, 10^{-3})$ & $(10^{-5}, 10^{-5})$ \\ \hline
{\begin{tabular}[c]{@{}c@{}}ecoli-0-1\_vs\_5\\ (240 x7)\end{tabular}} & 0 \% & 88.61 & 77.5 & 66.68 & 88.89 & 94.44 & 81.94 \\
 &  & $(10^{-2})$ & $(10^{3})$ & $(10^{-1}, 1)$ & $(10^{-4}, 10^{-5})$ & $(10^{-3}, 10^{-3})$ & $(10^{-5}, 10^{-4})$ \\
 & 5 \% & 88.61 & 86.11 & 85.83 & 88.89 & 93.06 & 93.06 \\
 &  & $(10^{-2})$ & $(10)$ & $(10^{-1}, 1)$ & $(10^{-3}, 10^{-2})$ & $(10^{-4}, 10^{-3})$ & $(10^{-5}, 10^{-4})$ \\
 & 10 \% & 87.22 & 96.43 & 84.44 & 88.89 & 88.89 & 88.89 \\
 &  & $(10^{-2})$ & $(10)$ & $(10^{-1}, 1)$ & $(10^{-5}, 10^{-5})$ & $(10^{-5}, 10^{-5})$ & $(10^{-5}, 10^{-5})$ \\
 & 15 \% & 87.22 & 79.17 & 86.11 & 91.89 & 88.89 & 88.89 \\
 &  & $(10^{-2})$ & $(10)$ & $(10^{-3}, 10^{-3})$ & $(10^{-2}, 10^{-1})$ & $(10^{-5}, 10^{-5})$ & $(10^{-5}, 10^{-5})$ \\
 & 20 \% & 85.83 & 88.89 & 88.89 & 88.89 & 88.89 & 88.89 \\
 &  & $(10^{-2})$ & $(10^{4})$ & $(10^{-4}, 10^{-5})$ & $(10^{-5}, 10^{-5})$ & $(10^{-3}, 10^{-3})$ & $(10^{-5}, 10^{-5})$ \\ \hline 
 \multicolumn{8}{l}{$^{\dagger}$ represents the proposed models. ACC represents the accuracy metric.} 
 \end{tabular}}
\end{table*}

\begin{table*}[htp]
\ContinuedFloat
\centering
    \caption{(Continued)}
    \resizebox{0.8\textwidth}{!}{                                                                      
    \begin{tabular}{cccccccc}
\hline
Dataset  & Noise & \begin{tabular}[c]{@{}c@{}}SVM \cite{cortes1995support} \\ ACC (\%)\\ ($d_1$)\end{tabular} & \begin{tabular}[c]{@{}c@{}}GBSVM (PSO) \cite{xia2022gbsvm} \\ ACC (\%)\\ ($d_1$)\end{tabular} & \begin{tabular}[c]{@{}c@{}}TSVM \cite{khemchandani2007twin} \\ ACC (\%)\\ $(d_1, d_2)$\end{tabular} & \begin{tabular}[c]{@{}c@{}}GBTSVM$^{\dagger}$\\ ACC (\%)\\ $(d_1, d_2)$\end{tabular} & \begin{tabular}[c]{@{}c@{}}LS-GBTSVM$^{\dagger}$\\ ACC (\%)\\ $(d_1, d_3)$\end{tabular} & \begin{tabular}[c]{@{}c@{}}LS-GBTSVM (SOR)$^{\dagger}$\\ ACC (\%)\\ $(d_1, d_3)$\end{tabular} \\ \hline
 {\begin{tabular}[c]{@{}c@{}}ecoli-0-1-4-6\_vs\_5\\ (280 x 7)\end{tabular}} & 0 \% & 95.81 & 94.05 & 67.88 & 96.05 & 98.81 & 97.62 \\
 &  & $(10^{-2})$ & $(10)$ & $(1, 1)$ & $(10^{-4}, 10^{-5})$ & $(10^{-1}, 10^{-2})$ & $(10^{-5}, 10^{-5})$ \\
 & 5 \% & 97.62 & 94.05 & 88.81 & 89.29 & 79.76 & 96.43 \\
 &  & $(10^{-2})$ & $(10)$ & $(10, 1)$ & $(10^{-2}, 10^{-2})$ & $(10^{-5}, 10^{-2})$ & $(10^{-5}, 10^{-5})$ \\
 & 10 \% & 87.62 & 69.35 & 85.24 & 94.05 & 94.05 & 96.43 \\
 &  & $(10^{-2})$ & $(10)$ & $(10^{-5}, 10^{-5})$ & $(10^{-5}, 10^{-5})$ & $(10^{-5}, 10^{-5})$ & $(10^{-5}, 10^{-5})$ \\
 & 15 \% & 90.62 & 95.24 & 85.24 & 93.55 & 94.05 & 94.05 \\
 &  & $(10^{-2})$ & $(10)$ & $(10^{-1}, 10^{-1})$ & $(10^{-5}, 10^{-5})$ & $(10^{-5}, 10^{-5})$ & $(10^{-5}, 10^{-5})$ \\
 & 20 \% & 94.05 & 94.05 & 89.05 & 94.05 & 89.29 & 94.05 \\
 &  & $(10^{-5})$ & $(10^{4})$ & $(1, 1)$ & $(10^{-5}, 10^{-5})$ & $(10^{-5}, 10^{-5})$ & $(10^{-5}, 10^{-5})$ \\ \hline
{\begin{tabular}[c]{@{}c@{}}ecoli-0-1-4-7\_vs\_2-3-5-6\\ (336 x 8)\end{tabular}} & 0 \% & 85.05 & 52.48 & 68.51 & 87.13 & 74.26 & 88.12 \\
 &  & $(10^{-2})$ & $(10^{5})$ & $(10^{-1}, 1)$ & $(10^{-5}, 10^{-5})$ & $(10, 10^{-5})$ & $(10^{-5}, 10^{-5})$ \\
 & 5 \% & 84.06 & 73.27 & 87.13 & 87.13 & 87.13 & 87.13 \\
 &  & $(10^{-2})$ & $(10)$ & $(10^{-1}, 10^{-2})$ & $(10^{-1}, 10^{-1})$ & $(10^{-5}, 10^{-5})$ & $(10^{-5}, 10^{-5})$ \\
 & 10 \% & 94.06 & 64.36 & 91.09 & 87.13 & 96.04 & 87.13 \\
 &  & $(10^{-2})$ & $(10)$ & $(10^{-5}, 10^{-5})$ & $(10^{-5}, 10^{-5})$ & $(10^{-5}, 10^{-5})$ & $(10^{-5}, 10^{-5})$ \\
 & 15 \% & 85.05 & 84.16 & 74.55 & 87.13 & 87.13 & 87.13 \\
 &  & $(10^{-2})$ & $(10)$ & $(10^{5}, 10^{2})$ & $(10^{-5}, 10^{-5})$ & $(10^{-5}, 10^{-5})$ & $(10^{-5}, 10^{-5})$ \\
 & 20 \% & 87.13 & 74.26 & 81.09 & 87.13 & 87.13 & 95.05 \\
 &  & $(10^{-5})$ & $(10)$ & $(10^{2}, 10^{5})$ & $(10^{-5}, 10^{-5})$ & $(10^{-5}, 10^{-5})$ & $(10^{-5}, 10^{-5})$ \\ \hline
{\begin{tabular}[c]{@{}c@{}}ecoli-0-1-4-7\_vs\_5-6\\ (332 x 7)\end{tabular}} & 0 \% & 87 & 76 & 67.32 & 91 & 95 & 94 \\
 &  & $(10^{-2})$ & $(10^{4})$ & $(10, 10)$ & $(10^{-1}, 10^{-1})$ & $(10^{-2}, 10^{-2})$ & $(10^{-5}, 10^{-5})$ \\
 & 5 \% & 84 & 85 & 89 & 93 & 96 & 94 \\
 &  & $(10^{-2})$ & $(10)$ & $(10, 10)$ & $(10^{-1}, 10^{-1})$ & $(10^{-4}, 10^{-4})$ & $(10^{-5}, 10^{-5})$ \\
 & 10 \% & 84 & 92 & 95 & 94 & 95 & 91 \\
 &  & $(10^{-2})$ & $(10)$ & $(1, 1)$ & $(10^{-5}, 10^{-5})$ & $(10^{-3}, 10^{-3})$ & $(10^{-5}, 10^{-5})$ \\
 & 15 \% & 84 & 94 & 84 & 87 & 83 & 83 \\
 &  & $(10^{-2})$ & $(10^{2})$ & $(1, 1)$ & $(10^{-5}, 10^{-4})$ & $(10^{-5}, 10^{-5})$ & $(10^{-5}, 10^{-3})$ \\
 & 20 \% & 91 & 91 & 93 & 94 & 91 & 94 \\
 &  & $(10^{-2})$ & $(10^{2})$ & $(10^{-5}, 10^{-5})$ & $(10^{-2}, 10^{-1})$ & $(10^{-5}, 10^{-3})$ & $(10^{-5}, 10^{-5})$ \\ \hline
 {\begin{tabular}[c]{@{}c@{}}haber\\ (306 x 4)\end{tabular}} & 0 \% & 77.17 & 77.17 & 57.96 & 82.61 & 81.52 & 79.35 \\
 &  & $(10^{-2})$ & $(10)$ & $(10^{-5}, 10^{-5})$ & $(10^{-5}, 10^{-5})$ & $(10^{-5}, 10^{-5})$ & $(10^{-5}, 10^{-5})$ \\
 & 5 \% & 76.09 & 77.17 & 75 & 82.61 & 80.43 & 80.43 \\
 &  & $(10^{-2})$ & $(10)$ & $(10^{3}, 10)$ & $(10^{-5}, 10^{-5})$ & $(10^{-5}, 10^{-5})$ & $(10^{-5}, 10^{-5})$ \\
 & 10 \% & 77.17 & 58.04 & 78.26 & 82.61 & 66.3 & 79.35 \\
 &  & $(10^{-2})$ & $(10^{2})$ & $(10^{3}, 10)$ & $(10^{-5}, 10^{-5})$ & $(10^{-3}, 10^{-5})$ & $(10^{-5}, 10^{-5})$ \\
 & 15 \% & 78.26 & 59.78 & 78.26 & 82.61 & 82.61 & 82.61 \\
 &  & $(10^{-2})$ & $(10^{-5})$ & $(10^{-5}, 10^{-5})$ & $(10^{-5}, 10^{-5})$ & $(10^{-5}, 10^{-5})$ & $(10^{-5}, 10^{-5})$ \\
 & 20 \% & 76.09 & 68.48 & 73.91 & 80.43 & 82.61 & 78.26 \\
 &  & $(10^{-2})$ & $(10^{-5})$ & $(10^{-5}, 10^{-5})$ & $(10^{-5}, 10^{-5})$ & $(10^{-4}, 10^{-3})$ & $(10^{-5}, 10^{-5})$ \\ \hline
 {\begin{tabular}[c]{@{}c@{}}haberman\\ (306 x 4)\end{tabular}} & 0 \% & 77.17 & 77.17 & 57.96 & 82.61 & 81.52 & 79.35 \\
 &  & $(10^{-2})$ & $(10)$ & $(10^{-5}, 10^{-5})$ & $(10^{-5}, 10^{-5})$ & $(10^{-5}, 10^{-5})$ & $(10^{-5}, 10^{-5})$ \\
 & 5 \% & 76.09 & 77.17 & 75 & 82.61 & 80.43 & 80.43 \\
 &  & $(10^{-3})$ & $(10)$ & $(10, 10^{3})$ & $(10^{-5}, 10^{-5})$ & $(10^{-5}, 10^{-5})$ & $(10^{-5}, 10^{-5})$ \\
 & 10 \% & 77.17 & 58.04 & 78.26 & 82.61 & 66.3 & 79.35 \\
 &  & $(10^{-2})$ & $(10^{2})$ & $(10, 10^{3})$ & $(10^{-5}, 10^{-5})$ & $(10^{-3}, 10^{-5})$ & $(10^{-5}, 10^{-5})$ \\
 & 15 \% & 78.26 & 59.78 & 78.26 & 82.61 & 82.61 & 82.61 \\
 &  & $(10^{-2})$ & $(10^{-5})$ & $(10^{-5}, 10^{-5})$ & $(10^{-5}, 10^{-5})$ & $(10^{-5}, 10^{-5})$ & $(10^{-5}, 10^{-5})$ \\
 & 20 \% & 76.09 & 68.48 & 73.91 & 80.43 & 82.61 & 78.26 \\
 &  & $(10^{-2})$ & $(10^{-5})$ & $(10, 10^{3})$ & $(10^{-5}, 10^{-5})$ & $(10^{-4}, 10^{-3})$ & $(10^{-5}, 10^{-5})$ \\ \hline
 {\begin{tabular}[c]{@{}c@{}}haberman\_survival\\ (306 x 4)\end{tabular}} & 0 \% & 77.17 & 78.26 & 57.96 & 82.61 & 80.43 & 81.52 \\
 &  & $(10^{-2})$ & $(1)$ & $(10^{-5}, 10^{-5})$ & $(10^{-5}, 10^{-5})$ & $(10^{-5}, 10^{-5})$ & $(10^{-5}, 10^{-5})$ \\
 & 5 \% & 76.09 & 78.26 & 75 & 82.61 & 78.26 & 79.35 \\
 &  & $(10^{-2})$ & $(10^{-5})$ & $(10, 10^{3})$ & $(10^{-5}, 10^{-5})$ & $(10^{-5}, 10^{-5})$ & $(10^{-5}, 10^{-5})$ \\
 & 10 \% & 77.17 & 75 & 78.26 & 79.35 & 79.35 & 79.35 \\
 &  & $(10^{-2})$ & $(10)$ & $(10^{3}, 10)$ & $(10^{-5}, 10^{-5})$ & $(10^{-5}, 10^{-5})$ & $(10^{-5}, 10^{-5})$ \\
 & 15 \% & 78.26 & 59.78 & 78.26 & 82.61 & 82.61 & 82.61 \\
 &  & $(10^{-2})$ & $(10^{-5})$ & $(10^{-5}, 10^{-5})$ & $(10^{-5}, 10^{-5})$ & $(10^{-5}, 10^{-5})$ & $(10^{-5}, 10^{-5})$ \\
 & 20 \% & 76.09 & 59.78 & 73.91 & 81.52 & 81.52 & 80.43 \\
 &  & $(10^{-2})$ & $(10^{-5})$ & $(10^{3}, 10)$ & $(10^{-5}, 10^{-4})$ & $(10^{-5}, 10^{-5})$ & $(10^{-5}, 10^{-5})$ \\ \hline
{\begin{tabular}[c]{@{}c@{}}heart-stat\\ (270 x 14)\end{tabular}} & 0 \% & 90.12 & 85.93 & 58.69 & 90.12 & 90.12 & 87.65 \\
 &  & $(10^{-1})$ & $(10^{5})$ & $(1, 10)$ & $(10^{-5}, 10^{-5})$ & $(10^{-4}, 10^{-4})$ & $(10^{-5}, 10^{-4})$ \\
 & 5 \% & 88.89 & 82.59 & 80.12 & 87.65 & 88.89 & 88.89 \\
 &  & $(10^{-1})$ & $(10^{5})$ & $(1, 1)$ & $(10^{-2}, 10^{-2})$ & $(10^{-4}, 10^{-3})$ & $(10^{-5}, 10^{-5})$ \\
 & 10 \% & 78.89 & 63.58 & 70.12 & 62.96 & 58.02 & 85.19 \\
 &  & $(10^{-2})$ & $(10^{5})$ & $(1, 1)$ & $(10^{-1}, 10^{-2})$ & $(10^{-2}, 10^{-5})$ & $(10^{-5}, 10^{-5})$ \\
 & 15 \% & 82.72 & 74.44 & 76.12 & 77.78 & 76.54 & 65.43 \\
 &  & $(10^{-3})$ & $(10^{5})$ & $(10^{-1}, 10^{-1})$ & $(10^{-5}, 10^{-5})$ & $(10^{-3}, 10^{-3})$ & $(10^{-5}, 10^{-4})$ \\
 & 20 \% & 87.65 & 83.58 & 77.78 & 86.42 & 87.65 & 84.32 \\
 &  & $(10^{-2})$ & $(10^{4})$ & $(10, 10^{3})$ & $(10^{-1}, 10^{-2})$ & $(10^{-4}, 10^{-5})$ & $(10^{-5}, 10^{-5})$ \\ \hline 
{\begin{tabular}[c]{@{}c@{}}led7digit-0-2-4-5-6-7-8-9\_vs\_1\\ (443 x 8)\end{tabular}} & 0 \% & 92.23 & 78.35 & 66.77 & 93.23 & 94.74 & 93.98 \\
 &  & $(10^{-5})$ & $(10^{5})$ & $(1, 1)$ & $(10^{-4}, 10^{-5})$ & $(10^{-2}, 10^{-3})$ & $(10^{-5}, 10^{-5})$ \\
 & 5 \% & 83.23 & 78.35 & 83.98 & 78.2 & 79.7 & 94.74 \\
 &  & $(10^{-5})$ & $(10^{5})$ & $(1, 1)$ & $(10^{-2}, 10^{-2})$ & $(10^{-3}, 10^{-4})$ & $(10^{-5}, 10^{-5})$ \\
 & 10 \% & 83.23 & 90.98 & 84.74 & 84.96 & 91.73 & 91.73 \\
 &  & $(10^{-5})$ & $(10)$ & $(10^{-5}, 10^{-5})$ & $(10^{-4}, 10^{-5})$ & $(10^{-5}, 10^{-3})$ & $(10^{-5}, 10^{-5})$ \\
 & 15 \% & 83.23 & 68.42 & 84.74 & 87.97 & 78.95 & 87.97 \\
 &  & $(10^{-5})$ & $(10^{2})$ & $(10^{3}, 10^{2})$ & $(10^{-5}, 10^{-5})$ & $(10^{-5}, 10^{-5})$ & $(10^{-5}, 10^{-4})$ \\
 & 20 \% & 93.23 & 80.41 & 83.74 & 81.2 & 93.23 & 93.23 \\
 &  & $(10^{-5})$ & $(10^{5})$ & $(10^{3}, 10^{2})$ & $(10^{-5}, 10^{-5})$ & $(10^{-5}, 10^{-5})$ & $(10^{-5}, 10^{-5})$ \\ \hline
 \multicolumn{8}{l}{$^{\dagger}$ represents the proposed models. ACC represents the accuracy metric.}
\end{tabular}}
\end{table*}

\begin{table*}[htp]
\ContinuedFloat
\centering
    \caption{(Continued)}
    \resizebox{0.8\textwidth}{!}{                                                                      
    \begin{tabular}{cccccccc}
\hline
Dataset  & Noise & \begin{tabular}[c]{@{}c@{}}SVM \cite{cortes1995support} \\ ACC (\%)\\ ($d_1$)\end{tabular} & \begin{tabular}[c]{@{}c@{}}GBSVM (PSO) \cite{xia2022gbsvm} \\ ACC (\%)\\ ($d_1$)\end{tabular} & \begin{tabular}[c]{@{}c@{}}TSVM \cite{khemchandani2007twin} \\ ACC (\%)\\ $(d_1, d_2)$\end{tabular} & \begin{tabular}[c]{@{}c@{}}GBTSVM$^{\dagger}$\\ ACC (\%)\\ $(d_1, d_2)$\end{tabular} & \begin{tabular}[c]{@{}c@{}}LS-GBTSVM$^{\dagger}$\\ ACC (\%)\\ $(d_1, d_3)$\end{tabular} & \begin{tabular}[c]{@{}c@{}}LS-GBTSVM$^{\dagger}$ (SOR)\\ ACC (\%)\\ $(d_1, d_3)$\end{tabular} \\ \hline
 {\begin{tabular}[c]{@{}c@{}}mammographic\\ (961 x 6)\end{tabular}} & 0 \% & 79.58 & 80.28 & 60.41 & 83.04 & 80.97 & 79.93 \\
 &  & $(10^{-3})$ & $(10^{-5})$ & $(10^{-5}, 10^{-5})$ & $(10^{-2}, 10^{-2})$ & $(10^{-5}, 10^{-5})$ & $(10^{-5}, 10^{-4})$ \\
 & 5 \% & 79.7 & 80.28 & 64.01 & 79.93 & 80.62 & 79.93 \\
 &  & $(10^{-3})$ & $(10^{-5})$ & $(10^{4}, 10)$ & $(10^{-2}, 10^{-2})$ & $(10^{-4}, 10^{-4})$ & $(10^{-5}, 10^{-5})$ \\
 & 10 \% & 80.28 & 79.93 & 55.02 & 77.85 & 79.93 & 80.28 \\
 &  & $(10^{-3})$ & $(10^{-5})$ & $(10, 10^{2})$ & $(10^{-5}, 10^{-5})$ & $(10^{-4}, 10^{-3})$ & $(10^{-5}, 10^{-5})$ \\
 & 15 \% & 79.93 & 80.62 & 81.31 & 80.62 & 79.58 & 79.93 \\
 &  & $(10^{-3})$ & $(10^{-5})$ & $(1, 1)$ & $(10^{-2}, 10^{-2})$ & $(10^{-5}, 10^{-5})$ & $(10^{-5}, 10^{-5})$ \\
 & 20 \% & 79.58 & 79.58 & 82.01 & 78.55 & 78.89 & 80.28 \\
 &  & $(10^{-3})$ & $(10^{5})$ & $(1, 1)$ & $(10^{-5}, 10^{-5})$ & $(10^{-4}, 10^{-5})$ & $(10^{-5}, 10^{-5})$ \\ \hline
{\begin{tabular}[c]{@{}c@{}}monks\_3\\ (554 x 7)\end{tabular}} & 0 \% & 75.45 & 59.88 & 59.7 & 78.44 & 74.85 & 70.66 \\
 &  & $(10^{-1})$ & $(10^{5})$ & $(1, 1)$ & $(1, 10^{-1})$ & $(10, 10^{-3})$ & $(10^{-5}, 10^{-3})$ \\
 & 5 \% & 73.65 & 59.88 & 77.25 & 77.25 & 77.25 & 67.07 \\
 &  & $(10^{-1})$ & $(1)$ & $(10, 10^{4})$ & $(1, 10^{-2})$ & $(10, 10^{-4})$ & $(10^{-5}, 10^{-5})$ \\
 & 10 \% & 73.05 & 70.66 & 76.65 & 78.44 & 70.06 & 72.28 \\
 &  & $(10^{-1})$ & $(10^{-5})$ & $(10^{-5}, 10^{-5})$ & $(10^{-1}, 10^{-1})$ & $(10^{-2}, 10^{-4})$ & $(10^{-5}, 10^{-3})$ \\
 & 15 \% & 73.05 & 70.06 & 70.44 & 73.89 & 60.48 & 55.69 \\
 &  & $(10^{5})$ & $(10^{5})$ & $(10^{-5}, 10^{-5})$ & $(10^{-2}, 10^{-1})$ & $(10^{-5}, 10^{-5})$ & $(10^{-5}, 10^{-5})$ \\
 & 20 \% & 71.86 & 80 & 71.26 & 80.24 & 81.08 & 89.7 \\
 &  & $(10^{-2})$ & $(10^{5})$ & $(10, 10^{4})$ & $(1, 10^{-1})$ & $(10^{-2}, 10^{-5})$ & $(10^{-5}, 10^{-5})$ \\ \hline
{\begin{tabular}[c]{@{}c@{}}mushroom\\ (8124 x 22)\end{tabular}} & 0 \% & 80.06 & 50.83 & 77.5 & 83.46 & 82.69 & 80.01 \\
 &  & $(10^{-2})$ & $(10^{-5})$ & $(10^{-2}, 10^{-2})$ & $(10^{-2}, 10^{-1})$ & $(10^{5}, 10^{-5})$ & $(10^{-5}, 10^{-5})$ \\
 & 5 \% & 88.02 & 61.08 & 94.09 & 65.46 & 74.98 & 70.34 \\
 &  & $(10^{-3})$ & $(10^{4})$ & $(10^{3}, 10^{4})$ & $(1, 10^{-1})$ & $(10^{-4}, 10^{-5})$ & $(10^{-5}, 10^{-3})$ \\
 & 10 \% & 87.98 & 88.56 & 96.23 & 96.88 & 90.81 & 87.94 \\
 &  & $(10^{5})$ & $(10^{-1})$ & $(10^{2}, 10^{5})$ & $(10^{-2}, 10^{-2})$ & $(10, 10^{-4})$ & $(10^{-5}, 10^{-3})$ \\
 & 15 \% & 87.94 & 85.76 & 81.76 & 98.48 & 89.46 & 87.28 \\
 &  & $(10^{-3})$ & $(10^{2})$ & $(10^{5}, 10)$ & $(10^{-2}, 10^{-2})$ & $(10^{-3}, 10^{-4})$ & $(10^{-5}, 10^{-2})$ \\
 & 20\% & 87.74 & 85.67 & 85.59 & 93.81 & 90.89 & 85.81 \\
 &  & $(10^{-4})$ & $(10^{2})$ & $(10^{5}, 10)$ & $(1, 1)$ & $(10, 10^{-4})$ & $(10^{-5}, 10^{-3})$ \\ \hline
{\begin{tabular}[c]{@{}c@{}}musk\_1\\ (476 x 167)\end{tabular}} & 0 \% & 68.53 & 52.66 & 59.15 & 72.03 & 69.93 & 56.64 \\
 &  & $(10^{-3})$ & $(10)$ & $(1, 1)$ & $(10^{-5}, 10^{-5})$ & $(10^{-2}, 10^{-3})$ & $(10^{-5}, 10^{-3})$ \\
 & 5 \% & 67.13 & 56.15 & 77.62 & 54.55 & 68.53 & 56.64 \\
 &  & $(10^{-3})$ & $(10^{5})$ & $(1, 10^{-1})$ & $(10^{-3}, 10^{-3})$ & $(10^{-4}, 10^{-3})$ & $(10^{-5}, 10^{-5})$ \\
 & 10 \% & 76.92 & 80 & 79.02 & 66.43 & 60.14 & 67.34 \\
 &  & $(10^{-3})$ & $(10^{5})$ & $(1, 1)$ & $(10^{-2}, 10^{-2})$ & $(10^{-4}, 10^{-3})$ & $(10^{-5}, 10^{-3})$ \\
 & 15 \% & 69.02 & 71.32 & 65.03 & 64.34 & 71.33 & 53.85 \\
 &  & $(10^{-3})$ & $(10^{5})$ & $(10, 10^{-1})$ & $(10^{-2}, 10^{-2})$ & $(1, 10^{-5})$ & $(10^{-5}, 10^{-3})$ \\
 & 20 \% & 67.83 & 70 & 67.13 & 63.64 & 64.55 & 64.55 \\
 &  & $(10^{-2})$ & $(10^{5})$ & $(10, 10^{-1})$ & $(10^{-2}, 10^{-2})$ & $(10^{-5}, 10^{-4})$ & $(10^{-5}, 10^{-3})$ \\ \hline
 {\begin{tabular}[c]{@{}c@{}}new-thyroid1\\ (215 x 16)\end{tabular}} & 0 \% & 88.46 & 85.38 & 66 & 89.23 & 98.46 & 93.85 \\
 &  & $(10^{-2})$ & $(10^{4})$ & $(10^{-1}, 1)$ & $(10^{-4}, 10^{-5})$ & $(10^{-3}, 10^{-3})$ & $(10^{-5}, 10^{-5})$ \\
 & 5 \% & 98.46 & 95.38 & 90 & 89.23 & 86.15 & 90.77 \\
 &  & $(10^{-2})$ & $(10)$ & $(10^{-1}, 1)$ & $(10^{-4}, 10^{-5})$ & $(10^{-3}, 10^{-3})$ & $(10^{-5}, 10^{-5})$ \\
 & 10 \% & 88.46 & 85.38 & 85.38 & 87.69 & 90.77 & 86.15 \\
 &  & $(10^{-2})$ & $(10)$ & $(10^{-1}, 1)$ & $(10^{-5}, 10^{-5})$ & $(10^{-5}, 10^{-5})$ & $(10^{-5}, 10^{-5})$ \\
 & 15 \% & 88.46 & 84.62 & 88.46 & 87.69 & 87.69 & 84.62 \\
 &  & $(10^{-2})$ & $(1)$ & $(1, 1)$ & $(10^{-5}, 10^{-5})$ & $(10^{-5}, 10^{-5})$ & $(10^{-5}, 10^{-5})$ \\
 & 20 \% & 90.46 & 87.69 & 88.46 & 86.15 & 87.69 & 90.77 \\
 &  & $(10^{-2})$ & $(1)$ & $(10^{-5}, 10^{-5})$ & $(10^{-5}, 10^{-5})$ & $(10^{-5}, 10^{-5})$ & $(10^{-5}, 10^{-5})$ \\ \hline
 {\begin{tabular}[c]{@{}c@{}}oocytes\_merluccius\_nucleus\_4d\\ (1022 x 42)\end{tabular}} & 0 \% & 64.82 & 63.19 & 59.58 & 69.06 & 68.86 & 59.61 \\
 &  & $(10^{-1})$ & $(10^{-5})$ & $(1, 1)$ & $(10^{-4}, 10^{-5})$ & $(10^{-4}, 10^{-2})$ & $(10^{-5}, 10^{-5})$ \\
 & 5 \% & 64.82 & 62.21 & 58.83 & 67.75 & 58.96 & 64.82 \\
 &  & $(10^{-5})$ & $(10^{-5})$ & $(10, 10)$ & $(10^{-2}, 1)$ & $(10^{-5}, 10^{-2})$ & $(10^{-5}, 10^{-3})$ \\
 & 10 \% & 64.82 & 62.87 & 81.11 & 68.4 & 52.77 & 54.4 \\
 &  & $(10^{-5})$ & $(10^{-5})$ & $(1, 1)$ & $(10^{-2}, 1)$ & $(10^{-5}, 10^{-3})$ & $(10^{-5}, 10^{-3})$ \\
 & 15 \% & 64.82 & 61.89 & 60.78 & 69.71 & 54.72 & 63.84 \\
 &  & $(10^{-4})$ & $(10^{-5})$ & $(1, 1)$ & $(10^{-2}, 10^{-2})$ & $(10^{-3}, 10^{-2})$ & $(10^{-5}, 10^{-5})$ \\
 & 20 \% & 65.8 & 64.82 & 58.5 & 63.52 & 58.96 & 55.05 \\
 &  & $(10^{-4})$ & $(10^{2})$ & $(1, 1)$ & $(10^{-2}, 1)$ & $(10^{-1}, 10^{-1})$ & $(10^{-5}, 10^{-5})$ \\ \hline
{\begin{tabular}[c]{@{}c@{}}ozone\\ (2536 x 6)\end{tabular}} & 0 \% & 94.58 & 94.58 & 85.58 & 96.58 & 95.48 & 95.93 \\
 &  & $(10^{-5})$ & $(10)$ & $(1, 1)$ & $(10^{-5}, 10^{-5})$ & $(10^{-5}, 10^{-5})$ & $(10^{-5}, 10^{-5})$ \\
 & 5 \% & 86.58 & 96.58 & 85.29 & 96.58 & 92.9 & 93.56 \\
 &  & $(10^{-5})$ & $(10^{-3})$ & $(10^{-5}, 10^{-5})$ & $(10^{-2}, 10^{-2})$ & $(10^{-3}, 10^{-2})$ & $(10^{-5}, 10^{-5})$ \\
 & 10 \% & 96.58 & 75 & 96.58 & 96.58 & 96.58 & 96.58 \\
 &  & $(10^{-5})$ & $(10^{-5})$ & $(10^{-5}, 10^{-5})$ & $(10^{-5}, 10^{-5})$ & $(10^{-5}, 10^{-5})$ & $(10^{-5}, 10^{-5})$ \\
 & 15 \% & 80.58 & 83.09 & 85.79 & 96.58 & 96.58 & 96.58 \\
 &  & $(10^{-5})$ & $(10^{2})$ & $(1, 1)$ & $(10^{-5}, 10^{-5})$ & $(10^{-5}, 10^{-5})$ & $(10^{-5}, 10^{-5})$ \\
 & 20 \% & 86.58 & 66.89 & 96.58 & 96.58 & 96.58 & 96.58 \\
 &  & $(10^{-5})$ & $(10)$ & $(10, 10)$ & $(10^{-5}, 10^{-5})$ & $(10^{-5}, 10^{-5})$ & $(10^{-5}, 10^{-5})$ \\ \hline
{\begin{tabular}[c]{@{}c@{}}ringnorm\\ (7400 x 21)\end{tabular}} & 0 \% & 72.62 & 63.56 & 64.86 & 73.56 & 75.05 & 70.95 \\
 &  & $(10^{-3})$ & $(10^{-5})$ & $(10^{-1}, 1)$ & $(1, 10^{2})$ & $(1, 10^{-5})$ & $(10^{-5}, 10^{-5})$ \\
 & 5 \% & 76.44 & 70.96 & 75.9 & 73.65 & 75.99 & 76.26 \\
 &  & $(10^{-4})$ & $(10^{-4})$ & $(1, 1)$ & $(1, 10^{2})$ & $(10^{-2}, 10^{-5})$ & $(10^{-5}, 10^{-4})$ \\
 & 10 \% & 76.4 & 64.89 & 76.22 & 77.21 & 76.71 & 76.53 \\
 &  & $(10^{-2})$ & $(10^{5})$ & $(1, 1)$ & $(1, 10)$ & $(10^{-2}, 10^{-5})$ & $(10^{-5}, 10^{-4})$ \\
 & 15 \% & 76.04 & 68.6 & 76.17 & 76.67 & 72.48 & 76.53 \\
 &  & $(10^{-2})$ & $(10^{5})$ & $(1, 1)$ & $(1, 10^{2})$ & $(10^{-3}, 10^{-5})$ & $(10^{-5}, 10^{-4})$ \\
 & 20 \% & 75.95 & 47.97 & 75.9 & 75.9 & 75.68 & 76.44 \\
 &  & $(10^{-1})$ & $(10^{5})$ & $(1, 1)$ & $(1, 10)$ & $(10^{-1}, 10^{-5})$ & $(10^{-5}, 10^{-3})$ \\ \hline
\multicolumn{8}{l}{$^{\dagger}$ represents the proposed models. ACC represents the accuracy metric.}
\end{tabular}}
\end{table*}

\begin{table*}[htp]
\centering
\ContinuedFloat
    \caption{(Continued)}
    \resizebox{0.78\textwidth}{!}{                                                                                          
    \begin{tabular}{cccccccc}
\hline
Dataset                                                                                         & Noise & \begin{tabular}[c]{@{}c@{}}SVM \cite{cortes1995support} \\ ACC (\%)\\ ($d_1$)\end{tabular} & \begin{tabular}[c]{@{}c@{}}GBSVM (PSO) \cite{xia2022gbsvm} \\ ACC (\%)\\ ($d_1$)\end{tabular} & \begin{tabular}[c]{@{}c@{}}TSVM \cite{khemchandani2007twin} \\ ACC (\%)\\ $(d_1, d_2)$\end{tabular} & \begin{tabular}[c]{@{}c@{}}GBTSVM$^{\dagger}$\\ ACC (\%)\\ $(d_1, d_2)$\end{tabular} & \begin{tabular}[c]{@{}c@{}}LS-GBTSVM$^{\dagger}$\\ ACC (\%)\\ $(d_1, d_3)$\end{tabular} & \begin{tabular}[c]{@{}c@{}}LS-GBTSVM (SOR)$^{\dagger}$\\ ACC (\%)\\ $(d_1, d_3)$\end{tabular} \\ \hline
 {\begin{tabular}[c]{@{}c@{}}shuttle-6\_vs\_2-3\\ (230 x 10)\end{tabular}} & 0 \% & 95.55 & 95.55 & 70.08 & 95.65 & 100 & 98.55 \\
 &  & $(10^{-2})$ & $(10^{5})$ & $(10^{-1}, 10^{-1})$ & $(10^{-5}, 10^{-5})$ & $(10^{-4}, 10^{-5})$ & $(10^{-5}, 10^{-4})$ \\
 & 5 \% & 98.55 & 89.86 & 88.76 & 95.65 & 100 & 97.1 \\
 &  & $(10^{-2})$ & $(10^{3})$ & $(10^{-1}, 10^{-2})$ & $(10^{-5}, 10^{-5})$ & $(10^{-5}, 10^{-3})$ & $(10^{-5}, 10^{-5})$ \\
 & 10 \% & 100 & 86.87 & 100 & 95.65 & 94.2 & 95.65 \\
 &  & $(10^{-2})$ & $(10^{2})$ & $(10^{-1}, 10^{-2})$ & $(10^{-5}, 10^{-5})$ & $(10^{-3}, 10^{-3})$ & $(10^{-5}, 10^{-5})$ \\
 & 15 \% & 85.65 & 65.22 & 87.1 & 95.65 & 95.65 & 95.65 \\
 &  & $(10^{-5})$ & $(10^{5})$ & $(10^{-1}, 10^{-2})$ & $(10^{-5}, 10^{-5})$ & $(10^{-5}, 10^{-5})$ & $(10^{-5}, 10^{-5})$ \\
 & 20 \% & 95.65 & 69.57 & 97.1 & 95.65 & 95.65 & 95.65 \\
 &  & $(10^{-5})$ & $(10^{2})$ & $(10^{-2}, 10^{-3})$ & $(10^{-5}, 10^{-5})$ & $(10^{-5}, 10^{-5})$ & $(10^{-5}, 10^{-5})$ \\ \hline
{\begin{tabular}[c]{@{}c@{}}spambase\\ (4601 x 58)\end{tabular}} & 0 \% & 88.78 & 79.79 & 74.18 & 89.79 & 87.69 & 79.22 \\
 &  & $(10^{-3})$ & $(10^{2})$ & $(10^{-3}, 10^{-3})$ & $(10^{-1}, 10^{-2})$ & $(10^{-4}, 10^{-3})$ & $(10^{-5}, 10^{-2})$ \\
 & 5 \% & 88.49 & 51.99 & 81.17 & 90.73 & 88.2 & 81.1 \\
 &  & $(10^{-3})$ & $(10^{4})$ & $(10^{-1}, 10^{-1})$ & $(10^{-1}, 10^{-1})$ & $(10^{2}, 10^{-4})$ & $(10^{-5}, 10^{-3})$ \\
 & 10 \% & 88.41 & 81.25 & 91.09 & 91.46 & 88.63 & 82.62 \\
 &  & $(10^{-3})$ & $(10^{4})$ & $(10^{-1}, 10^{-1})$ & $(10^{-1}, 10^{-1})$ & $(10^{-4}, 10^{-3})$ & $(10^{-5}, 10^{-2})$ \\
 & 15 \% & 88.27 & 74.41 & 88.85 & 90.51 & 88.92 & 88.63 \\
 &  & $(10^{-3})$ & $(10^{2})$ & $(10^{-3}, 10^{-3})$ & $(10^{-2}, 10^{-2})$ & $(10^{-5}, 10^{-4})$ & $(10^{-5}, 10^{-3})$ \\
 & 20 \% & 88.63 & 88.92 & 89.43 & 88.63 & 87.4 & 87.98 \\
 &  & $(10)$ & $(10^{4})$ & $(10^{-1}, 10^{-1})$ & $(10^{-2}, 10^{-2})$ & $(10^{-3}, 10^{-2})$ & $(10^{-5}, 10^{-3})$ \\ \hline
{\begin{tabular}[c]{@{}c@{}}spectf\\ (267 x 45)\end{tabular}} & 0 \% & 76.54 & 70.4 & 62.39 & 80.25 & 71.6 & 62.96 \\
 &  & $(10^{-2})$ & $(10^{2})$ & $(10^{-1}, 10^{-2})$ & $(10^{-1}, 10^{-1})$ & $(10^{-2}, 10^{-2})$ & $(10^{-5}, 10^{-5})$ \\
 & 5 \% & 76.54 & 70.99 & 77.78 & 80.25 & 82.72 & 80.25 \\
 &  & $(10^{-3})$ & $(10^{5})$ & $(10^{-1}, 10^{-1})$ & $(10^{-1}, 10^{-1})$ & $(10^{-2}, 10^{-3})$ & $(10^{-5}, 10^{-5})$ \\
 & 10 \% & 76.54 & 62.67 & 80.25 & 80.25 & 81.48 & 59.26 \\
 &  & $(10^{-3})$ & $(10^{5})$ & $(10^{-2}, 10^{-2})$ & $(10^{-5}, 10^{-3})$ & $(10^{-2}, 10^{-3})$ & $(10^{-5}, 10^{-5})$ \\
 & 15 \% & 76.54 & 70.44 & 82.72 & 74.07 & 75.31 & 74.07 \\
 &  & $(10^{-3})$ & $(10)$ & $(10^{-2}, 10^{-2})$ & $(10^{-1}, 10^{-1})$ & $(10^{-3}, 10^{-3})$ & $(10^{-5}, 10^{-4})$ \\
 & 20 \% & 76.54 & 73.33 & 77.78 & 74.07 & 79.01 & 76.54 \\
 &  & $(10^{-3})$ & $(1)$ & $(10^{-5}, 10^{-5})$ & $(10^{-5}, 10^{-5})$ & $(10^{2}, 10^{-3})$ & $(10^{-5}, 10^{-5})$ \\ \hline
{\begin{tabular}[c]{@{}c@{}}tic\_tac\_toe\\ (958 x 10)\end{tabular}} & 0 \% & 75.69 & 76.88 & 68.66 & 99.65 & 76.04 & 72.92 \\
 &  & $(10^{-3})$ & $(10^{5})$ & $(10^{-5}, 10^{-5})$ & $(10^{-5}, 10^{-5})$ & $(10^{-1}, 10^{-5})$ & $(10^{-5}, 10^{-5})$ \\
 & 5 \% & 75.35 & 70.47 & 95.65 & 99.65 & 83.61 & 71.88 \\
 &  & $(10^{-3})$ & $(10^{3})$ & $(10^{-5}, 10^{-5})$ & $(10^{-5}, 10^{-5})$ & $(10^{-4}, 10^{-3})$ & $(10^{-5}, 10^{-5})$ \\
 & 10 \% & 76.04 & 70.97 & 89.65 & 99.65 & 74.31 & 70.49 \\
 &  & $(10^{-3})$ & $(10^{2})$ & $(10^{-5}, 10^{-5})$ & $(10^{-1}, 10^{-1})$ & $(10, 10^{-4})$ & $(10^{-5}, 10^{-5})$ \\
 & 15 \% & 74.65 & 60.76 & 89.65 & 99.65 & 69.1 & 73.61 \\
 &  & $(10^{-3})$ & $(10)$ & $(10^{-5}, 10^{-5})$ & $(10^{-1}, 10^{-1})$ & $(10^{-5}, 10^{-4})$ & $(10^{-5}, 10^{-5})$ \\
 & 20 \% & 73.96 & 63.33 & 89.65 & 99.65 & 69.1 & 67.71 \\
 &  & $(10^{-3})$ & $(10^{5})$ & $(10, 10^{2})$ & $(10^{-4}, 10^{-5})$ & $(10^{-2}, 10^{-3})$ & $(10^{-5}, 10^{-5})$ \\ \hline
 {\begin{tabular}[c]{@{}c@{}}vehicle1\\ (846 x 19)\end{tabular}} & 0 \% & 76.38 & 73.62 & 59.44 & 79.53 & 65.35 & 71.65 \\
 &  & $(10^{-5})$ & $(10)$ & $(10, 1)$ & $(1, 1)$ & $(10^{-5}, 10^{-3})$ & $(10^{-5}, 10^{-4})$ \\
 & 5 \% & 76.38 & 73.62 & 78.74 & 76.77 & 74.8 & 75.59 \\
 &  & $(10^{-5})$ & $(10)$ & $(10, 1)$ & $(10^{-5}, 10^{-5})$ & $(10^{-2}, 10^{-2})$ & $(10^{-5}, 10^{-5})$ \\
 & 10 \% & 76.38 & 71.26 & 81.5 & 76.38 & 70.87 & 75.2 \\
 &  & $(10^{-5})$ & $(10)$ & $(10^{-1}, 10^{-1})$ & $(10^{-5}, 10^{-5})$ & $(10^{-4}, 10^{-3})$ & $(10^{-5}, 10^{-5})$ \\
 & 15 \% & 76.38 & 70.47 & 80.71 & 75.59 & 65.75 & 72.44 \\
 &  & $(10^{-5})$ & $(10)$ & $(10^{-2}, 10^{-2})$ & $(10^{-5}, 10^{-5})$ & $(10^{-5}, 10^{-5})$ & $(10^{-5}, 10^{-5})$ \\
 & 20 \% & 76.38 & 76.38 & 78.35 & 75.98 & 63.78 & 75.98 \\
 &  & $(10^{-5})$ & $(10)$ & $(10, 1)$ & $(10^{-4}, 10^{-5})$ & $(10^{-3}, 10^{-3})$ & $(10^{-5}, 10^{-5})$ \\ \hline
 {\begin{tabular}[c]{@{}c@{}}vehicle2\\ (846 x 19)\end{tabular}} & 0 \% & 71.65 & 53.54 & 64.86 & 94.09 & 84.65 & 62.2 \\
 &  & $(10^{-3})$ & $(10^{5})$ & $(1, 1)$ & $(10^{-1}, 10^{-1})$ & $(10^{-3}, 10^{-3})$ & $(10^{-5}, 10^{-5})$ \\
 & 5 \% & 75.2 & 53.54 & 86.85 & 81.1 & 83.46 & 67.32 \\
 &  & $(10^{-3})$ & $(10^{4})$ & $(1, 1)$ & $(10^{-2}, 1)$ & $(10^{-3}, 10^{-2})$ & $(10^{-5}, 10^{-3})$ \\
 & 10 \% & 72.44 & 65.04 & 96.85 & 95.28 & 66.54 & 66.54 \\
 &  & $(10^{-3})$ & $(10^{4})$ & $(1, 1)$ & $(10^{-2}, 1)$ & $(10^{-5}, 10^{-2})$ & $(10^{-5}, 10^{-3})$ \\
 & 15 \% & 74.8 & 57.09 & 65.67 & 92.13 & 81.5 & 60.63 \\
 &  & $(10^{-3})$ & $(10)$ & $(1, 1)$ & $(10^{-2}, 1)$ & $(10^{-4}, 10^{-3})$ & $(10^{-5}, 10^{-5})$ \\
 & 20 \% & 72.05 & 61.97 & 70.88 & 94.09 & 78.35 & 70.08 \\
 &  & $(10^{-5})$ & $(10)$ & $(1, 1)$ & $(10^{-2}, 1)$ & $(10^{5}, 10^{-4})$ & $(10^{-5}, 10^{-3})$ \\ \hline
 {\begin{tabular}[c]{@{}c@{}}vertebral\_column\_2clases\\ (310 x 7)\end{tabular}} & 0 \% & 75.27 & 68.82 & 63.13 & 69.89 & 78.49 & 78.49 \\
 &  & $(10^{-2})$ & $(10^{5})$ & $(1, 1)$ & $(10^{-4}, 10^{-5})$ & $(10^{-5}, 10^{-5})$ & $(10^{-5}, 10^{-5})$ \\
 & 5 \% & 75.27 & 68.82 & 88.17 & 69.89 & 79.57 & 77.42 \\
 &  & $(10^{-2})$ & $(10^{-5})$ & $(1, 1)$ & $(10^{-4}, 10^{-5})$ & $(10^{-5}, 10^{-5})$ & $(10^{-5}, 10^{-5})$ \\
 & 10 \% & 75.27 & 68.82 & 89.25 & 89.82 & 70.97 & 79.57 \\
 &  & $(10^{-2})$ & $(10)$ & $(1, 1)$ & $(10^{-5}, 10^{-5})$ & $(10^{3}, 10^{-4})$ & $(10^{-5}, 10^{-5})$ \\
 & 15 \% & 74.19 & 68.82 & 79.25 & 79.57 & 75.27 & 76.34 \\
 &  & $(10^{-2})$ & $(1)$ & $(1, 1)$ & $(10^{-5}, 10^{-5})$ & $(10^{-5}, 10^{-5})$ & $(10^{-5}, 10^{-5})$ \\
 & 20 \% & 69.89 & 67.74 & 86.02 & 69.89 & 70.97 & 75.27 \\
 &  & $(10^{-5})$ & $(10^{5})$ & $(10^{-5}, 10^{-5})$ & $(10^{-5}, 10^{-5})$ & $(10^{-5}, 10^{-5})$ & $(10^{-5}, 10^{-5})$ \\ \hline
{\begin{tabular}[c]{@{}c@{}}wpbc\\ (194 x 34)\end{tabular}} & 0 \% & 77.97 & 57.63 & 60 & 77.97 & 77.97 & 77.97 \\
 &  & $(10^{-5})$ & $(10)$ & $(10, 10^{-1})$ & $(10^{-5}, 10^{-5})$ & $(10^{-5}, 10^{-5})$ & $(10^{-5}, 10^{-5})$ \\
 & 5 \% & 70.97 & 70.85 & 77.97 & 79.66 & 72.88 & 74.58 \\
 &  & $(10^{-5})$ & $(10)$ & $(10^{-4}, 10^{-5})$ & $(10^{-5}, 10^{-5})$ & $(10^{-3}, 10^{-3})$ & $(10^{-5}, 10^{-4})$ \\
 & 10 \% & 77.97 & 52.67 & 76.27 & 77.97 & 77.63 & 78.41 \\
 &  & $(10^{-5})$ & $(10^{-2})$ & $(10^{-2}, 10^{-3})$ & $(10^{-5}, 10^{-5})$ & $(10^{-5}, 10^{-3})$ & $(10^{-5}, 10^{-3})$ \\
 & 15 \% & 72.38 & 45.76 & 77.97 & 76.27 & 59.32 & 59.32 \\
 &  & $(10^{-5})$ & $(10^{5})$ & $(10^{-1}, 10^{-2})$ & $(10^{-4}, 10^{-5})$ & $(10^{-5}, 10^{-2})$ & $(10^{-5}, 10^{-3})$ \\
 & 20 \% & 72.05 & 54.24 & 77.97 & 79.66 & 71.19 & 57.63 \\
 &  & $(10^{-2})$ & $(10)$ & $(10^{-3}, 10^{-5})$ & $(10^{-5}, 10^{-5})$ & $(10^{-5}, 10^{-4})$ & $(10^{-5}, 10^{-3})$ \\ \hline
 \multicolumn{8}{l}{$^{\dagger}$ represents the proposed models. ACC represents the accuracy metric.}
\end{tabular}}
\end{table*}

\begin{table*}[htp]
\ContinuedFloat
\centering
    \caption{(Continued)}
    \resizebox{0.8\textwidth}{!}{                                                                                          
    \begin{tabular}{cccccccc}
\hline
Dataset   & Noise & \begin{tabular}[c]{@{}c@{}}SVM \cite{cortes1995support} \\ ACC (\%)\\ ($d_1$)\end{tabular} & \begin{tabular}[c]{@{}c@{}}GBSVM (PSO) \cite{xia2022gbsvm} \\ ACC (\%)\\ ($d_1$)\end{tabular} & \begin{tabular}[c]{@{}c@{}}TSVM \cite{khemchandani2007twin} \\ ACC (\%)\\ $(d_1, d_2)$\end{tabular} & \begin{tabular}[c]{@{}c@{}}GBTSVM$^{\dagger}$\\ ACC (\%)\\ $(d_1, d_2)$\end{tabular} & \begin{tabular}[c]{@{}c@{}}LS-GBTSVM$^{\dagger}$\\ ACC (\%)\\ $(d_1, d_3)$\end{tabular} & \begin{tabular}[c]{@{}c@{}}LS-GBTSVM (SOR)$^{\dagger}$\\ ACC (\%)\\ $(d_1, d_3)$\end{tabular} \\ \hline
 {\begin{tabular}[c]{@{}c@{}}yeast-0-2-5-6\_vs\_3-7-8-9\\ (1004 x 9)\end{tabular}} & 0 \% & 83.71 & 64.9 & 66.23 & 88.08 & 92.38 & 85.1 \\
 &  & $(10^{-3})$ & $(10)$ & $(10, 10)$ & $(10^{-5}, 10^{-5})$ & $(10^{-1}, 10^{-3})$ & $(10^{-5}, 10^{-5})$ \\
 & 5 \% & 88.38 & 82.65 & 81.06 & 91.39 & 92.72 & 92.05 \\
 &  & $(10^{-3})$ & $(10^{2})$ & $(10^{-5}, 10^{-5})$ & $(10^{-5}, 10^{-5})$ & $(10^{-3}, 10^{-3})$ & $(10^{-5}, 10^{-5})$ \\
 & 10 \% & 91.72 & 54.97 & 82.05 & 92.39 & 93.38 & 91.72 \\
 &  & $(10^{-3})$ & $(10^{2})$ & $(10^{-1}, 10^{-1})$ & $(10^{-5}, 10^{-5})$ & $(10^{-5}, 10^{-5})$ & $(10^{-5}, 10^{-5})$ \\
 & 15 \% & 88.01 & 75.5 & 81.06 & 91.39 & 92.72 & 93.05 \\
 &  & $(10^{-3})$ & $(10)$ & $(10^{-5}, 10^{-5})$ & $(10^{-5}, 10^{-5})$ & $(10^{-5}, 10^{-5})$ & $(10^{-5}, 10^{-4})$ \\
 & 20 \% & 87.35 & 84.7 & 87.55 & 91.39 & 93.71 & 92.38 \\
 &  & $(10^{-3})$ & $(10^{2})$ & $(10^{2}, 10^{2})$ & $(10^{-5}, 10^{-5})$ & $(10^{-4}, 10^{-3})$ & $(10^{-5}, 10^{-5})$ \\ \hline  
{\begin{tabular}[c]{@{}c@{}}yeast-0-2-5-7-9\_vs\_3-6-8\\ (1004 x 9)\end{tabular}} & 0 \% & 86.79 & 68.55 & 66.67 & 87.5 & 91.19 & 84.87 \\
 &  & $(10^{-3})$ & $(10)$ & $(1, 1)$ & $(1, 10^{-2})$ & $(10^{5}, 10^{-5})$ & $(10^{-5}, 10^{-5})$ \\
 & 5 \% & 88.14 & 68.42 & 80.79 & 87.5 & 89.31 & 87.5 \\
 &  & $(10^{-3})$ & $(10)$ & $(1, 1)$ & $(10^{-2}, 10^{-3})$ & $(1, 10^{-5})$ & $(10^{-5}, 10^{-5})$ \\
 & 10 \% & 87.02 & 87.15 & 94.37 & 93.71 & 95.03 & 93.05 \\
 &  & $(10^{-3})$ & $(10^{4})$ & $(1, 1)$ & $(10^{-2}, 1)$ & $(10^{-4}, 10^{-3})$ & $(10^{-5}, 10^{-5})$ \\
 & 15 \% & 87.5 & 74.34 & 80.79 & 87.5 & 87.5 & 87.5 \\
 &  & $(10^{-3})$ & $(10)$ & $(10^{-1}, 10^{-1})$ & $(10^{-2}, 10)$ & $(10^{-5}, 10^{-4})$ & $(10^{-5}, 10^{-5})$ \\
 & 20 \% & 87.5 & 77.63 & 80.79 & 87.5 & 90.79 & 91.19 \\
 &  & $(10^{-3})$ & $(10^{3})$ & $(10^{4}, 10^{4})$ & $(10^{-2}, 1)$ & $(10^{-4}, 10^{-3})$ & $(10^{-5}, 10^{-4})$ \\ \hline
{\begin{tabular}[c]{@{}c@{}}yeast-0-5-6-7-9\_vs\_4 \\ (528 x 9)\end{tabular}} & 0 \% & 81.19 & 56.6 & 68.29 & 84.91 & 91.19 & 91.19 \\
 &  & $(10^{-5})$ & $(10^{5})$ & $(10^{-1}, 10^{-1})$ & $(10^{-5}, 10^{-5})$ & $(10^{-4}, 10^{-3})$ & $(10^{-5}, 10^{-5})$ \\
 & 5 \% & 81.19 & 68.55 & 81.19 & 90.57 & 81.13 & 90.57 \\
 &  & $(10^{-5})$ & $(10^{5})$ & $(10^{-1}, 10^{-2})$ & $(10^{-4}, 10^{-5})$ & $(10^{-5}, 10^{-5})$ & $(10^{-5}, 10^{-5})$ \\
 & 10 \% & 91.19 & 83.4 & 91.82 & 91.19 & 91.19 & 91.19 \\
 &  & $(10^{-5})$ & $(10^{4})$ & $(10^{-1}, 10^{-1})$ & $(10^{-5}, 10^{-5})$ & $(10^{-5}, 10^{-5})$ & $(10^{-5}, 10^{-5})$ \\
 & 15 \% & 81.19 & 89.94 & 81.82 & 91.82 & 91.19 & 90.57 \\
 &  & $(10^{-5})$ & $(10)$ & $(10^{-3}, 10^{-4})$ & $(10^{-5}, 10^{-5})$ & $(10^{-5}, 10^{-5})$ & $(10^{-5}, 10^{-5})$ \\
 & 20 \% & 91.19 & 83.65 & 91.19 & 91.19 & 92.45 & 92.45 \\
 &  & $(10^{-5})$ & $(10)$ & $(10^{-3}, 10^{-5})$ & $(10^{-5}, 10^{-5})$ & $(10^{-3}, 10^{-2})$ & $(10^{-5}, 10^{-5})$ \\ \hline
{\begin{tabular}[c]{@{}c@{}}yeast-2\_vs\_4\\ (514 x 9)\end{tabular}} & 0 \% & 85.81 & 54.19 & 67.7 & 87.74 & 95.48 & 85.16 \\
 &  & $(10^{-5})$ & $(10)$ & $(10^{-1}, 1)$ & $(10^{-1}, 10^{-1})$ & $(1, 10^{-3})$ & $(10^{-5}, 10^{-5})$ \\
 & 5 \% & 85.81 & 85.81 & 84.84 & 85.81 & 94.84 & 86.45 \\
 &  & $(10^{-5})$ & $(10^{2})$ & $(10^{-1}, 1)$ & $(10^{-1}, 10^{-1})$ & $(10^{-3}, 10^{-2})$ & $(10^{-5}, 10^{-5})$ \\
 & 10 \% & 85.81 & 72.67 & 87.1 & 89.81 & 89.68 & 86.45 \\
 &  & $(10^{-5})$ & $(10^{2})$ & $(1, 1)$ & $(10^{-5}, 10^{-5})$ & $(10^{-2}, 10^{-2})$ & $(10^{-5}, 10^{-5})$ \\
 & 15 \% & 72.87 & 75.76 & 86.45 & 85.81 & 85.81 & 85.81 \\
 &  & $(10^{-5})$ & $(10^{4})$ & $(1, 1)$ & $(10^{-5}, 10^{-5})$ & $(10^{-5}, 10^{-5})$ & $(10^{-5}, 10^{-5})$ \\
 & 20 \% & 85.81 & 76.13 & 87.1 & 85.81 & 94.19 & 85.81 \\
 &  & $(10^{-5})$ & $(10)$ & $(10, 1)$ & $(10^{-5}, 10^{-5})$ & $(10^{-4}, 10^{-3})$ & $(10^{-5}, 10^{-5})$ \\ \hline
 {\begin{tabular}[c]{@{}c@{}}yeast3\\ (1484 x 9)\end{tabular}} & 0 \% & 79.91 & 79.03 & 67.15 & 80.04 & 90.81 & 80.92 \\
 &  & $(10^{-3})$ & $(10^{4})$ & $(10^{-1}, 1)$ & $(10^{-1}, 10^{-1})$ & $(10^{-1}, 10^{-3})$ & $(10^{-5}, 10^{-5})$ \\
 & 5 \% & 80.81 & 81.26 & 81.48 & 84.98 & 91.7 & 89.91 \\
 &  & $(10^{-3})$ & $(1)$ & $(1, 1)$ & $(10^{-5}, 10^{-4})$ & $(10^{-3}, 10^{-3})$ & $(10^{-5}, 10^{-5})$ \\
 & 10 \% & 89.69 & 82.74 & 89.69 & 91.03 & 93.27 & 88.34 \\
 &  & $(10^{-3})$ & $(10)$ & $(10^{-1}, 10^{-1})$ & $(10^{-1}, 1)$ & $(10^{-3}, 10^{-3})$ & $(10^{-5}, 10^{-5})$ \\
 & 15 \% & 89.46 & 82.87 & 88.79 & 90.13 & 88.12 & 91.93 \\
 &  & $(10^{-3})$ & $(10)$ & $(10^{-5}, 10^{-5})$ & $(10^{-1}, 1)$ & $(10^{-4}, 10^{-4})$ & $(10^{-5}, 10^{-3})$ \\
 & 20 \% & 88.12 & 76.23 & 80.81 & 86.55 & 89.69 & 88.12 \\
 &  & $(10^{-5})$ & $(10)$ & $(10^{-5}, 10^{-5})$ & $(10^{-2}, 1)$ & $(10^{-4}, 10^{-4})$ & $(10^{-5}, 10^{-5})$ \\ \hline
{Average ACC} & 0 \% & 82.21 & 72.46 & 65.47 & 85.30 & \textbf{85.32} & 81.34 \\
 & 5 \% & 82.11 & 75.92 & 82.27 & \textbf{83.75} & 82.96 & 82.49 \\
 & 10 \% & 83.06 & 75.36 & 84.64 & \textbf{85.82} & 82.11 & 82.25 \\
 & 15 \% & 81.25 & 74.44 & 81.11 & \textbf{85.30} & 81.5 & 80.67 \\
 & 20 \% & 82.41 & 74.89 & 82.46 & \textbf{84.86} & 83.11 & 82.75 \\ \hline
{Average Rank} & 0 \% & 3.46 & 4.79 & 5.69 & 1.97 & 1.94 & 3.14 \\
 & 5 \% & 3.74 & 4.78 & 3.88 & 2.83 & 2.68 & 3.10 \\
 & 10 \% & 3.65 & 4.97 & 3.10 & 2.42 & 3.28 & 3.58 \\
 & 15 \% & 3.44 & 4.94 & 3.53 & 2.26 & 3.44 & 3.38 \\
 & 20 \% & 3.35 & 4.90 & 3.40 & 2.99 & 3.19 & 3.17 \\ \hline
 \multicolumn{8}{l}{$^{\dagger}$ represents the proposed models. ACC represents the accuracy metric.} \\
 \multicolumn{8}{l}{Bold text denotes the model with the highest average ACC.}
 \end{tabular}}
\end{table*}

\renewcommand{\thetable}{S.IV}
\begin{table*}[ht!]
\centering
    \caption{Performance comparison of the proposed GBTSVM and LS-GBTSVM along with the baseline models over for UCI and KEEL datasets with non-linear kernel.}
    \label{Classification performance in NonLinear Case.}
    \resizebox{0.79\textwidth}{!}{
\begin{tabular}{cccccccc}
\hline
Dataset & Noise & \begin{tabular}[c]{@{}c@{}}SVM \cite{cortes1995support} \\ ACC (\%)\\ ($d_1$, $\sigma$)\end{tabular} & \begin{tabular}[c]{@{}c@{}}GBSVM (PSO) \cite{xia2022gbsvm}\\ ACC (\%)\\ ($d_1$, $\sigma$)\end{tabular} & \begin{tabular}[c]{@{}c@{}}TSVM \cite{khemchandani2007twin} \\ ACC (\%)\\ ($d_1$, $d_2$, $\sigma$)\end{tabular} & \begin{tabular}[c]{@{}c@{}}GBTSVM$^{\dagger}$\\ ACC (\%)\\ ($d_1$, $d_2$, $\sigma$)\end{tabular} & \begin{tabular}[c]{@{}c@{}}LS-GBTSVM$^{\dagger}$\\ ACC (\%)\\ ($d_1$, $d_3$, $\sigma$)\end{tabular} &\begin{tabular}[c]{@{}c@{}}LS-GBSVM (SMO)$^{\dagger}$\\ ACC (\%)\\ ($d_1$, $d_3$, $\sigma$)\end{tabular} \\ \hline
{\begin{tabular}[c]{@{}c@{}}aus\\ (690 x 15)\end{tabular}} & 0 \% & 56.25 & 81.73 & 81.25 & 87.98 & 83.08 & 83.75 \\
 &  & $(10^{-5}, 2^{-5})$ & $(10^{1}, 2^{5})$ & $(10^{-5}, 10^{-5}, 2^{-5})$ & $(10^{-1}, 1, 2^{5})$ & $(10^{3}, 10^{-5}, 2^{5})$ & $(10^{-5}, 10^{-5}, 2^{5})$ \\
 & 5 \% & 56.25 & 56.73 & 80.29 & 83.65 & 83.08 & 83.25 \\
 &  & $(10^{-5}, 2^{-5})$ & $(10^{2}, 2^{5})$ & $(10^{-5}, 10^{-5}, 2^{-5})$ & $(10^{-2}, 1, 2^{5})$ & $(10^{5}, 10^{-5}, 2^{5})$ & $(10^{-5}, 10^{-3}, 2^{5})$ \\
 & 10 \% & 56.25 & 87.5 & 82.69 & 87.98 & 85.75 & 86.25 \\
 &  & $(10^{-5}, 2^{-5})$ & $(10^{2}, 2^{5})$ & $(10^{-1}, 10^{-1}, 2^{-5})$ & $(10^{-1}, 1, 2^{5})$ & $(10^{-4}, 10^{-2}, 2^{5})$ & $(10^{-5}, 10^{-4}, 2^{5})$ \\
 & 15 \% & 56.25 & 80.77 & 80.77 & 89.42 & 78.37 & 76.25 \\
 &  & $(10^{-5}, 2^{-5})$ & $(10^{5}, 2^{5})$ & $(10^{-5}, 10^{-5}, 2^{-5})$ & $(10^{-1}, 10^{-1}, 2^{5})$ & $(10^{-3}, 1, 2^{5})$ & $(10^{-5}, 10^{-4}, 2^{5})$ \\
 & 20 \% & 69.25 & 70.78 & 77.4 & 86.54 & 83.75 & 86.25 \\
 &  & $(10^{-5}, 2^{-5})$ & $(10^{5}, 2^{5})$ & $(10^{-5}, 10^{-5}, 2^{-5})$ & $(10^{-1}, 10^{-1}, 2^{5})$ & $(10^{1}, 10^{-5}, 2^{5})$ & $(10^{-5}, 10^{-3}, 2^{5})$ \\ \hline
 {\begin{tabular}[c]{@{}c@{}}breast\_cancer\\ (286 x 10)\end{tabular}} & 0 \% & 74.42 & 100 & 67.44 & 62.79 & 74.42 & 74.42 \\
 &  & $(10^{-5}, 2^{-5})$ & $(10^{5}, 2^{5})$ & $(10^{-1}, 10^{-2}, 2^{-4})$ & $(10^{1}, 10^{1}, 2^{5})$ & $(10^{-5}, 10^{-2}, 2^{2})$ & $(10^{-5}, 10^{-5}, 2^{5})$ \\
 & 5 \% & 74.42 & 74.42 & 67.44 & 65.12 & 74.42 & 74.42 \\
 &  & $(10^{-5}, 2^{-5})$ & $(10^{5}, 2^{5})$ & $(10^{-2}, 10^{-1}, 2^{-3})$ & $(1, 1, 2^{5})$ & $(10^{-5}, 10^{-2}, 2^{5})$ & $(10^{-5}, 10^{-5}, 2^{5})$ \\
 & 10 \% & 74.42 & 90.78 & 69.77 & 72.79 & 74.42 & 74.42 \\
 &  & $(10^{-5}, 2^{-5})$ & $(10^{5}, 2^{5})$ & $(10^{-1}, 10^{-2}, 2^{-1})$ & $(10^{1}, 10^{1}, 2^{5})$ & $(10^{-5}, 10^{-2}, 2^{5})$ & $(10^{-5}, 10^{-5}, 2^{5})$ \\
 & 15 \% & 74.42 & 58.84 & 68.6 & 73.26 & 74.42 & 74.42 \\
 &  & $(10^{-5}, 2^{-5})$ & $(10^{5}, 2^{1})$ & $(10^{-1}, 10^{-2}, 2^{-5})$ & $(10^{-5}, 10^{-5}, 2^{5})$ & $(10^{5}, 10^{-5}, 2^{3})$ & $(10^{-5}, 10^{-5}, 2^{5})$ \\
 & 20 \% & 74.42 & 77.91 & 74.42 & 74.42 & 74.42 & 74.42 \\
 &  & $(10^{-5}, 2^{-5})$ & $(10^{5}, 2^{5})$ & $(10^{-2}, 10^{-5}, 2^{-5})$ & $(10^{-4}, 10^{-5}, 2^{5})$ & $(10^{-5}, 10^{-5}, 2^{5})$ & $(10^{-5}, 10^{-5}, 2^{5})$ \\ \hline
{\begin{tabular}[c]{@{}c@{}}checkerboard\_Data\\ (690 x 15)\end{tabular}} & 0 \% & 56.25 & 81.73 & 81.25 & 87.98 & 82.44 & 81.75 \\
 &  & $(10^{-5}, 2^{-5})$ & $(10^{1}, 2^{2})$ & $(10^{-2}, 10^{-2}, 2^{-5})$ & $(10^{-1}, 1, 2^{5})$ & $(10^{-4}, 10^{-1}, 2^{5})$ & $(10^{-5}, 10^{-5}, 2^{5})$ \\
 & 5 \% & 56.25 & 56.73 & 80.29 & 83.65 & 82.44 & 76.25 \\
 &  & $(10^{-5}, 2^{-5})$ & $(10^{2}, 2^{2})$ & $(10^{-5}, 10^{-5}, 2^{-5})$ & $(10^{-2}, 1, 2^{5})$ & $(10^{5}, 10^{-5}, 2^{5})$ & $(10^{-5}, 10^{-3}, 2^{4})$ \\
 & 10 \% & 56.25 & 87.5 & 82.69 & 87.98 & 83.75 & 86.25 \\
 &  & $(10^{-5}, 2^{-5})$ & $(10^{1}, 2^{5})$ & $(10^{-1}, 10^{-1}, 2^{-2})$ & $(10^{-1}, 1, 2^{4})$ & $(10^{4}, 10^{-2}, 2^{4})$ & $(10^{-5}, 10^{-4}, 2^{4})$ \\
 & 15 \% & 56.25 & 80.87 & 80.77 & 89.42 & 78.37 & 76.25 \\
 &  & $(10^{-5}, 2^{-5})$ & $(10^{5}, 2^{5})$ & $(10^{-5}, 10^{-5}, 2^{-5})$ & $(10^{-1}, 10^{-1}, 2^{5})$ & $(10^{-3}, 1, 2^{5})$ & $(10^{-5}, 10^{-4}, 2^{4})$ \\
 & 20 \% & 56.25 & 87.98 & 77.4 & 86.54 & 83.75 & 86.25 \\
 &  & $(10^{-5}, 2^{-5})$ & $(10^{2}, 2^{5})$ & $(10^{-5}, 10^{-5}, 2^{-5})$ & $(10^{-1}, 10^{-1}, 2^{5})$ & $(10^{1}, 10^{-5}, 2^{5})$ & $(10^{-5}, 10^{-3}, 2^{2})$ \\ \hline
{\begin{tabular}[c]{@{}c@{}}chess\_krvkp\\ (3196 x 37)\end{tabular}} & 0 \% & 52.35 & 50.39 & 90.41 & 97.08 & 87.65 & 87.65 \\
 &  & $(10^{-3}, 2^{3})$ & $(10^{5}, 2^{5})$ & $(10^{-1}, 10^{-1}, 2^{-5})$ & $(1, 1, 2^{5})$ & $(10^{-3}, 10^{-3}, 2^{3})$ & $(10^{-5}, 10^{-5}, 2^{5})$ \\
 & 5 \% & 52.35 & 80.76 & 87.28 & 96.87 & 87.65 & 82.35 \\
 &  & $(10^{-3}, 2^{3})$ & $(10^{-5}, 2^{-5})$ & $(10^{-1}, 10^{-1}, 2^{-5})$ & $(10^{1}, 10^{-1}, 2^{3})$ & $(10^{3}, 10^{-5}, 2^{4})$ & $(10^{-5}, 10^{-2}, 2^{3})$ \\
 & 10 \% & 52.35 & 79.65 & 84.15 & 97.08 & 87.55 & 87.67 \\
 &  & $(10^{-3}, 2^{3})$ & $(10^{2}, 2^{-3})$ & $(10^{-1}, 10^{-1}, 2^{-5})$ & $(1, 1, 2^{5})$ & $(10^{-4}, 10^{-2}, 2^{4})$ & $(10^{-5}, 10^{-5}, 2^{5})$ \\
 & 15 \% & 53.28 & 79.89 & 80.5 & 92.6 & 92.87 & 92.35 \\
 &  & $(10^{-2}, 2^{2})$ & $(10^{4}, 2^{-2})$ & $(10^{-1}, 10^{-1}, 2^{-5})$ & $(1, 1, 2^{5})$ & $(10^{-5}, 10^{-5}, 2^{5})$ & $(10^{-5}, 10^{-2}, 2^{5})$ \\
 & 20 \% & 73.18 & 75.76 & 76.64 & 90.09 & 87.76 & 82.35 \\
 &  & $(10^{-2}, 2^{2})$ & $(10^{4}, 2^{1})$ & $(10^{-1}, 10^{-1}, 2^{-5})$ & $(1, 1, 2^{5})$ & $(10^{-5}, 10^{-1}, 2^{4})$ & $(10^{-5}, 10^{-3}, 2^{4})$ \\ \hline
{\begin{tabular}[c]{@{}c@{}}crossplane130\\ (130 x 3)\end{tabular}} & 0 \% & 51.28 & 100 & 100 & 100 & 91.28 & 91.28 \\
 &  & $(10^{-3}, 2^{4})$ & $(10^{5}, 2^{5})$ & $(10^{-5}, 10^{-5}, 2^{-5})$ & $(10^{-5}, 10^{-5}, 2^{5})$ & $(10^{-3}, 10^{-5}, 2^{4})$ & $(10^{-5}, 10^{-5}, 2^{5})$ \\
 & 5 \% & 81.28 & 90.44 & 100 & 100 & 91.28 & 88.72 \\
 &  & $(10^{-3}, 2^{2})$ & $(10^{2}, 2^{1})$ & $(10^{-2}, 10^{-3}, 2^{2})$ & $(10^{-5}, 10^{-5}, 2^{5})$ & $(10^{-3}, 10^{-4}, 2^{3})$ & $(10^{-5}, 10^{-5}, 2^{5})$ \\
 & 10 \% & 51.28 & 89.74 & 97.87 & 100 & 98.72 & 91.28 \\
 &  & $(10^{-3}, 2^{1})$ & $(10^{1}, 1)$ & $(1, 1, 2^{-5})$ & $(10^{-5}, 10^{-5}, 2^{5})$ & $(1, 10^{-3}, 2^{4})$ & $(10^{-5}, 10^{-2}, 2^{4})$ \\
 & 15 \% & 81.28 & 94.78 & 100 & 100 & 98.79 & 98.28 \\
 &  & $(10^{-3}, 2^{1})$ & $(10^{5}, 2^{5})$ & $(1, 1, 2^{-5})$ & $(10^{-5}, 10^{-5}, 2^{5})$ & $(10^{4}, 10^{-2}, 2^{3})$ & $(10^{-5}, 10^{-5}, 2^{5})$ \\
 & 20 \% & 81.28 & 89.19 & 90.44 & 94.87 & 91.28 & 91.28 \\
 &  & $(10^{-2}, 2^{-2})$ & $(10^{5}, 2^{5})$ & $(10^{-1}, 10^{-1}, 2^{4})$ & $(10^{-5}, 10^{-5}, 2^{5})$ & $(10^{-1}, 10^{-5}, 2^{3})$ & $(10^{-5}, 10^{-5}, 2^{5})$ \\ \hline
 {\begin{tabular}[c]{@{}c@{}}ecoli-0-1\_vs\_2-3-5\\ (244 x 8)\end{tabular}} & 0 \% & 91.89 & 89.19 & 90.59 & 90.81 & 90.81 & 90.81 \\
 &  & $(1, 1)$ & $(10^{5}, 2^{3})$ & $(10^{-4}, 10^{-3}, 2^{5})$ & $(10^{-2}, 1, 2^{2})$ & $(10^{-4}, 10^{-3}, 2^{3})$ & $(10^{-5}, 10^{-5}, 2^{5})$ \\
 & 5 \% & 89.19 & 80.81 & 93.24 & 91.89 & 90.81 & 89.19 \\
 &  & $(10^{-5}, 2^{-4})$ & $(10^{5}, 2^{-4})$ & $(10^{-5}, 10^{-5}, 2^{-5})$ & $(10^{-5}, 10^{-5}, 2^{5})$ & $(10^{-5}, 10^{-2}, 2^{1})$ & $(10^{-5}, 10^{-5}, 2^{2})$ \\
 & 10 \% & 89.19 & 100 & 94.59 & 99.19 & 89.19 & 89.19 \\
 &  & $(10^{-5}, 2^{-4})$ & $(10^{5}, 2^{-4})$ & $(10^{-5}, 10^{-5}, 2^{-5})$ & $(10^{-2}, 1, 2^{5})$ & $(10^{-3}, 10^{-2}, 2^{4})$ & $(10^{-5}, 10^{-5}, 2^{5})$ \\
 & 15 \% & 89.19 & 89.78 & 93.24 & 89.19 & 90.81 & 89.19 \\
 &  & $(10^{-5}, 2^{-5})$ & $(10^{5}, 2^{1})$ & $(1, 10^{-1}, 2^{-3})$ & $(10^{-5}, 10^{-4}, 2^{5})$ & $(10^{-5}, 10^{-5}, 2^{5})$ & $(10^{-5}, 10^{-5}, 2^{5})$ \\
 & 20 \% & 89.19 & 90.87 & 91.89 & 91.89 & 89.19 & 89.19 \\
 &  & $(1, 2^{-2})$ & $(10^{5}, 2^{1})$ & $(1, 10^{-2}, 2^{-3})$ & $(10^{-1}, 10^{-1}, 2^{5})$ & $(10^{-5}, 10^{-2}, 2^{5})$ & $(10^{-5}, 10^{-5}, 2^{5})$ \\ \hline
 {\begin{tabular}[c]{@{}c@{}}ecoli-0-1\_vs\_5\\ (240 x7)\end{tabular}} & 0 \% & 94.44 & 97.22 & 97.22 & 98.89 & 93.06 & 88.89 \\
 &  & $(1, 1)$ & $(10^{1}, 2^{3})$ & $(10^{-5}, 10^{-4}, 2^{-3})$ & $(10^{-4}, 10^{-5}, 2^{5})$ & $(10^{4}, 10^{-2}, 2^{3})$ & $(10^{-5}, 10^{-4}, 2^{3})$ \\
 & 5 \% & 88.89 & 90.22 & 90.22 & 95.83 & 93.06 & 88.89 \\
 &  & $(10^{-5}, 2^{-4})$ & $(10^{3}, 2^{2})$ & $(10^{-2}, 10^{-2}, 2^{-5})$ & $(10^{-1}, 10^{-1}, 2^{5})$ & $(10^{-5}, 10^{-2}, 2^{3})$ & $(10^{-4}, 10^{-5}, 2^{5})$ \\
 & 10 \% & 88.89 & 100 & 95.83 & 88.89 & 88.89 & 88.89 \\
 &  & $(10^{-5}, 2^{-4})$ & $(10^{5}, 1)$ & $(10^{-2}, 10^{-2}, 2^{-5})$ & $(10^{-4}, 10^{-5}, 2^{5})$ & $(10^{-5}, 10^{-5}, 2^{5})$ & $(10^{-5}, 10^{-5}, 2^{4})$ \\
 & 15 \% & 88.89 & 100 & 94.44 & 98.61 & 88.89 & 88.89 \\
 &  & $(10^{-5}, 2^{-5})$ & $(10^{5}, 2^{1})$ & $(10^{-2}, 10^{-2}, 2^{-5})$ & $(10^{-5}, 10^{-3}, 2^{3})$ & $(10^{-5}, 10^{-3}, 2^{4})$ & $(10^{-5}, 10^{-5}, 2^{5})$ \\
 & 20 \% & 88.89 & 88.89 & 95.83 & 88.89 & 88.89 & 88.89 \\
 &  & $(10^{-5}, 2^{-5})$ & $(10^{5}, 2^{1})$ & $(10^{-5}, 10^{-5}, 2^{-5})$ & $(10^{-5}, 10^{-5}, 2^{5})$ & $(10^{-5}, 10^{-2}, 2^{3})$ & $(10^{-1}, 10^{-3}, 2^{4})$ \\ \hline
 \multicolumn{8}{l}{$^{\dagger}$ represents the proposed models. ACC represents the accuracy metric.}
\end{tabular}}
\end{table*}

\begin{table*}[htp]
\ContinuedFloat
\centering
    \caption{(Continued)}
    \resizebox{0.8\textwidth}{!}{                                                                                          
    \begin{tabular}{cccccccc}
\hline
Dataset & Noise & \begin{tabular}[c]{@{}c@{}}SVM \cite{cortes1995support} \\ ACC (\%)\\ ($d_1$, $\sigma$)\end{tabular} & \begin{tabular}[c]{@{}c@{}}GBSVM (PSO) \cite{xia2022gbsvm}\\ ACC (\%)\\ ($d_1$, $\sigma$)\end{tabular} & \begin{tabular}[c]{@{}c@{}}TSVM \cite{khemchandani2007twin} \\ ACC (\%)\\ ($d_1$, $d_2$, $\sigma$)\end{tabular} & \begin{tabular}[c]{@{}c@{}}GBTSVM$^{\dagger}$\\ ACC (\%)\\ ($d_1$, $d_2$, $\sigma$)\end{tabular} & \begin{tabular}[c]{@{}c@{}}LS-GBTSVM$^{\dagger}$\\ ACC (\%)\\ ($d_1$, $d_3$, $\sigma$)\end{tabular} & \begin{tabular}[c]{@{}c@{}}LS-GBSVM (SMO)$^{\dagger}$\\ ACC (\%)\\ ($d_1$, $d_3$, $\sigma$)\end{tabular} \\ \hline
{\begin{tabular}[c]{@{}c@{}}ecoli-0-1-4-6\_vs\_5\\ (280 x 7)\end{tabular}} & 0 \% & 98.81 & 100 & 100 & 97.62 & 94.05 & 94.05 \\
 &  & $(1, 1)$ & $(10^{5}, 2^{5})$ & $(10^{-2}, 10^{-1}, 2^{-3})$ & $(1, 1, 2^{5})$ & $(10^{4}, 10^{-5}, 2^{5})$ & $(10^{-4}, 10^{-4}, 2^{5})$ \\
 & 5 \% & 94.05 & 88.92 & 92.81 & 98.81 & 94.05 & 94.05 \\
 &  & $(10^{-5}, 2^{-5})$ & $(10^{5}, 2^{5})$ & $(10^{-5}, 10^{-5}, 2^{-5})$ & $(10^{-1}, 10^{-1}, 2^{2})$ & $(10^{5}, 10^{-4}, 2^{5})$ & $(10^{-4}, 10^{-2}, 2^{1})$ \\
 & 10 \% & 94.05 & 94.05 & 94.05 & 97.62 & 96.43 & 94.05 \\
 &  & $(10^{-5}, 2^{-5})$ & $(10^{5}, 2^{5})$ & $(10^{-2}, 10^{-2}, 2^{-5})$ & $(10^{-1}, 10^{-1}, 2^{2})$ & $(10^{-5}, 10^{-5}, 2^{5})$ & $(10^{-2}, 10^{-4}, 2^{1})$ \\
 & 15 \% & 94.05 & 90.87 & 94.05 & 96.43 & 94.05 & 94.05 \\
 &  & $(10^{-5}, 2^{-5})$ & $(10^{5}, 2^{5})$ & $(10^{-2}, 10^{-2}, 2^{-5})$ & $(10^{-1}, 10^{-1}, 2^{2})$ & $(10^{-5}, 10^{-2}, 2^{2})$ & $(10^{-1}, 10^{-2}, 2^{4})$ \\
 & 20 \% & 94.05 & 73.1 & 90.62 & 94.05 & 94.05 & 94.05 \\
 &  & $(10^{-5}, 2^{-5})$ & $(10^{5}, 2^{5})$ & $(10^{-5}, 10^{-5}, 2^{-5})$ & $(10^{-1}, 10^{-1}, 2^{2})$ & $(10^{-5}, 10^{-5}, 2^{5})$ & $(10^{-4}, 10^{-2}, 2^{5})$ \\ \hline
{\begin{tabular}[c]{@{}c@{}}ecoli-0-1-4-7\_vs\_2-3-5-6\\ (336 x 8)\end{tabular}} & 0 \% & 87.13 & 82.69 & 96.04 & 88.12 & 87.13 & 87.13 \\
 &  & $(10^{-5}, 2^{-5})$ & $(10^{5}, 2^{3})$ & $(10^{-1}, 1, 2^{-1})$ & $(1, 10^{-1}, 2^{2})$ & $(10^{-4}, 10^{-5}, 2^{5})$ & $(10^{-5}, 10^{-5}, 2^{5})$ \\
 & 5 \% & 87.13 & 80.54 & 90.04 & 96.04 & 87.13 & 87.13 \\
 &  & $(10^{-5}, 2^{-5})$ & $(10^{5}, 2^{1})$ & $(10^{-1}, 10^{-1}, 2^{-2})$ & $(10^{-1}, 1, 2^{2})$ & $(10^{-5}, 10^{-5}, 2^{3})$ & $(10^{-5}, 10^{-5}, 2^{5})$ \\
 & 10 \% & 87.13 & 85.69 & 83.07 & 88.12 & 87.13 & 87.13 \\
 &  & $(10^{-5}, 2^{-5})$ & $(10^{5}, 2^{2})$ & $(10^{-2}, 10^{-2}, 2^{-5})$ & $(1, 10^{1}, 2^{5})$ & $(10^{-5}, 10^{-3}, 2^{4})$ & $(10^{-5}, 10^{-5}, 2^{5})$ \\
 & 15 \% & 87.13 & 92.54 & 93.07 & 91.09 & 87.13 & 87.13 \\
 &  & $(10^{-5}, 2^{-5})$ & $(10^{5}, 2^{2})$ & $(1, 10^{-2}, 2^{-5})$ & $(10^{-1}, 10^{-1}, 2^{5})$ & $(10^{-5}, 10^{-2}, 2^{3})$ & $(10^{-5}, 10^{-5}, 2^{5})$ \\
 & 20 \% & 87.13 & 73.47 & 93.07 & 93.07 & 87.13 & 87.13 \\
 &  & $(10^{-5}, 2^{-5})$ & $(10^{5}, 2^{-2})$ & $(10^{-1}, 10^{-2}, 2^{-3})$ & $(1, 10^{-1}, 2^{5})$ & $(10^{-5}, 10^{-2}, 2^{4})$ & $(10^{-5}, 10^{-5}, 2^{5})$ \\ \hline
 {\begin{tabular}[c]{@{}c@{}}ecoli-0-1-4-7\_vs\_5-6\\ (332 x 7)\end{tabular}} & 0 \% & 91 & 80 & 96 & 94 & 97 & 93 \\
 &  & $(10^{1}, 1)$ & $(10^{5}, 2^{5})$ & $(10^{-3}, 10^{-2}, 2^{2})$ & $(1, 1, 2^{5})$ & $(10^{-4}, 10^{-2}, 2^{5})$ & $(10^{-5}, 10^{-3}, 2^{5})$ \\
 & 5 \% & 93 & 94 & 96 & 94 & 97 & 94 \\
 &  & $(10^{-5}, 2^{-5})$ & $(10^{5}, 2^{5})$ & $(10^{-2}, 10^{-2}, 2^{-5})$ & $(10^{-1}, 10^{-1}, 2^{5})$ & $(10^{-4}, 10^{-2}, 2^{2})$ & $(10^{-3}, 10^{-2}, 2^{4})$ \\
 & 10 \% & 93 & 87.65 & 93 & 94 & 93 & 93 \\
 &  & $(10^{-5}, 2^{-5})$ & $(10^{5}, 2^{5})$ & $(10^{-5}, 10^{-5}, 2^{-5})$ & $(1, 1, 2^{5})$ & $(10^{-5}, 10^{-5}, 2^{3})$ & $(10^{-2}, 10^{-1}, 2^{1})$ \\
 & 15 \% & 93 & 88 & 95 & 94 & 96 & 93 \\
 &  & $(10^{-5}, 2^{-5})$ & $(10^{5}, 2^{5})$ & $(1, 10^{-2}, 2^{-4})$ & $(1, 1, 2^{5})$ & $(10^{2}, 10^{-2}, 2^{4})$ & $(10^{-4}, 10^{-2}, 2^{1})$ \\
 & 20 \% & 91 & 93 & 92 & 95 & 93 & 93 \\
 &  & $(10^{-5}, 2^{-5})$ & $(10^{5}, 2^{5})$ & $(10^{-5}, 10^{-5}, 2^{-5})$ & $(1, 1, 2^{5})$ & $(10^{-5}, 10^{-2}, 2^{3})$ & $(10^{-4}, 10^{-3}, 2^{5})$ \\ \hline
{\begin{tabular}[c]{@{}c@{}}haber\\ (306 x 4)\end{tabular}} & 0 \% & 82.61 & 57.61 & 75.35 & 77.17 & 82.61 & 82.61 \\
 &  & $(10^{-5}, 2^{-5})$ & $(10^{5}, 2^{-2})$ & $(1, 10^{-1}, 2^{-3})$ & $(10^{1}, 1, 2^{3})$ & $(10^{-5}, 10^{-5}, 2^{5})$ & $(10^{-4}, 10^{-2}, 2^{4})$ \\
 & 5 \% & 82.61 & 72.1 & 70.26 & 75 & 82.61 & 82.61 \\
 &  & $(10^{-5}, 2^{-5})$ & $(10^{5}, 2^{-2})$ & $(1, 10^{-1}, 2^{-3})$ & $(1, 1, 2^{5})$ & $(10^{-5}, 10^{-5}, 2^{5})$ & $(10^{-2}, 10^{-4}, 2^{4})$ \\
 & 10 \% & 82.61 & 74 & 75.35 & 77.17 & 77.39 & 82.61 \\
 &  & $(10^{-5}, 2^{-5})$ & $(10^{5}, 2^{-2})$ & $(1, 10^{-1}, 2^{-3})$ & $(10^{1}, 1, 2^{3})$ & $(10^{-5}, 10^{-5}, 2^{5})$ & $(10^{-5}, 10^{-5}, 2^{5})$ \\
 & 15 \% & 82.61 & 80.67 & 79.35 & 79.35 & 82.61 & 82.61 \\
 &  & $(10^{-5}, 2^{-5})$ & $(10^{5}, 2^{1})$ & $(1, 10^{-1}, 2^{-3})$ & $(1, 1, 2^{5})$ & $(10^{-5}, 10^{-5}, 2^{5})$ & $(10^{-5}, 10^{-5}, 2^{5})$ \\
 & 20 \% & 82.61 & 78.7 & 78.26 & 78.26 & 82.61 & 82.61 \\
 &  & $(10^{-5}, 2^{-5})$ & $(10^{5}, 2^{-2})$ & $(1, 10^{-1}, 2^{-3})$ & $(10^{1}, 1, 2^{3})$ & $(10^{-5}, 10^{-5}, 2^{5})$ & $(10^{-4}, 10^{-5}, 2^{5})$ \\ \hline
{\begin{tabular}[c]{@{}c@{}}haberman\\ (306 x 4)\end{tabular}} & 0 \% & 81.52 & 57.61 & 75.35 & 77.17 & 82.61 & 82.61 \\
 &  & $(10^{-1}, 1)$ & $(10^{5}, 2^{-2})$ & $(10^{-1}, 1, 2^{-2})$ & $(1, 10^{1}, 2^{2})$ & $(10^{-5}, 10^{-5}, 2^{5})$ & $(10^{-4}, 10^{-5}, 2^{5})$ \\
 & 5 \% & 81.52 & 72.1 & 70.26 & 75 & 82.61 & 82.61 \\
 &  & $(10^{-1}, 1)$ & $(10^{5}, 2^{-2})$ & $(10^{-1}, 1, 2^{-3})$ & $(1, 1, 2^{5})$ & $(10^{-5}, 10^{-5}, 2^{5})$ & $(10^{-4}, 10^{-3}, 2^{4})$ \\
 & 10 \% & 82.61 & 79 & 79.35 & 77.17 & 77.39 & 82.61 \\
 &  & $(10^{-5}, 2^{-5})$ & $(10^{5}, 2^{-2})$ & $(10^{-1}, 1, 2^{-3})$ & $(1, 10^{1}, 2^{4})$ & $(10^{-5}, 10^{-5}, 2^{5})$ & $(10^{-4}, 10^{-5}, 2^{5})$ \\
 & 15 \% & 82.61 & 80.67 & 79.35 & 79.35 & 82.61 & 82.61 \\
 &  & $(10^{-5}, 2^{-5})$ & $(10^{5}, 2^{1})$ & $(10^{-1}, 1, 2^{-3})$ & $(1, 1, 2^{5})$ & $(10^{-5}, 10^{-5}, 2^{5})$ & $(10^{-3}, 10^{-3}, 2^{5})$ \\
 & 20 \% & 82.61 & 75.7 & 78.26 & 78.26 & 82.61 & 82.61 \\
 &  & $(10^{-5}, 2^{-5})$ & $(10^{5}, 2^{-2})$ & $(10^{-1}, 1, 2^{-3})$ & $(1, 10^{1}, 2^{4})$ & $(10^{-5}, 10^{-5}, 2^{5})$ & $(10^{-5}, 10^{-5}, 2^{2})$ \\ \hline
{\begin{tabular}[c]{@{}c@{}}haberman\_survival\\ (306 x 4)\end{tabular}} & 0 \% & 82.61 & 57.61 & 79.35 & 79.35 & 82.61 & 82.61 \\
 &  & $(10^{-5}, 2^{-5})$ & $(10^{5}, 2^{5})$ & $(1, 10^{-1}, 2^{-5})$ & $(10^{1}, 10^{1}, 2^{5})$ & $(10^{-5}, 10^{-2}, 2^{4})$ & $(10^{-2}, 10^{-2}, 2^{3})$ \\
 & 5 \% & 82.61 & 70.97 & 78.26 & 79.35 & 82.61 & 82.61 \\
 &  & $(10^{-5}, 2^{-5})$ & $(10^{5}, 2^{5})$ & $(1, 10^{-1}, 2^{-5})$ & $(10^{-1}, 1, 2^{5})$ & $(10^{-5}, 10^{-5}, 2^{3})$ & $(10^{-5}, 10^{-5}, 2^{5})$ \\
 & 10 \% & 75.85 & 78.59 & 79.35 & 79.35 & 82.61 & 82.61 \\
 &  & $(10^{-5}, 2^{-5})$ & $(10^{5}, 2^{5})$ & $(1, 10^{-1}, 2^{-5})$ & $(10^{1}, 10^{1}, 2^{5})$ & $(10^{-5}, 10^{-5}, 2^{5})$ & $(10^{-5}, 10^{-5}, 2^{5})$ \\
 & 15 \% & 82.61 & 79.42 & 79.35 & 71.74 & 82.61 & 82.61 \\
 &  & $(10^{-5}, 2^{-5})$ & $(10^{5}, 2^{5})$ & $(1, 10^{-1}, 2^{-5})$ & $(10^{-1}, 1, 2^{5})$ & $(10^{-5}, 10^{-5}, 2^{5})$ & $(10^{-5}, 10^{-5}, 2^{5})$ \\
 & 20 \% & 82.61 & 75.65 & 78.26 & 78.26 & 82.61 & 82.61 \\
 &  & $(10^{-5}, 2^{-5})$ & $(10^{5}, 2^{5})$ & $(1, 10^{-1}, 2^{-5})$ & $(10^{-1}, 10^{1}, 2^{-5})$ & $(10^{-5}, 10^{-5}, 2^{5})$ & $(10^{-5}, 10^{-5}, 2^{5})$ \\ \hline
 {\begin{tabular}[c]{@{}c@{}}heart-stat\\ (270 x 14)\end{tabular}} & 0 \% & 56.79 & 77.65 & 70.37 & 79.01 & 76.79 & 76.79 \\
 &  & $(10^{-5}, 2^{-5})$ & $(10^{1}, 2^{5})$ & $(10^{-1}, 10^{-1}, 2^{-5})$ & $(10^{-5}, 10^{-5}, 2^{5})$ & $(10^{-2}, 10^{-5}, 2^{2})$ & $(10^{-5}, 10^{-2}, 2^{3})$ \\
 & 5 \% & 56.79 & 65.78 & 69.14 & 83.95 & 76.79 & 76.79 \\
 &  & $(10^{-5}, 2^{-5})$ & $(10^{5}, 2^{5})$ & $(10^{-1}, 10^{-1}, 2^{-5})$ & $(10^{1}, 10^{-1}, 2^{3})$ & $(10^{-5}, 10^{-5}, 2^{2})$ & $(10^{-5}, 10^{-5}, 2^{5})$ \\
 & 10 \% & 56.79 & 86.42 & 71.6 & 79.01 & 73.21 & 76.79 \\
 &  & $(10^{-5}, 2^{-5})$ & $(10^{4}, 2^{5})$ & $(10^{-1}, 10^{-1}, 2^{-5})$ & $(10^{-5}, 10^{-5}, 2^{5})$ & $(10^{-5}, 10^{-5}, 2^{5})$ & $(10^{-4}, 10^{-4}, 2^{3})$ \\
 & 15 \% & 56.79 & 66.67 & 72.84 & 82.72 & 76.79 & 76.79 \\
 &  & $(10^{-5}, 2^{-5})$ & $(10^{5}, 2^{5})$ & $(10^{-1}, 10^{-1}, 2^{-5})$ & $(10^{-1}, 10^{1}, 2^{5})$ & $(10^{-5}, 10^{-5}, 2^{5})$ & $(10^{-4}, 10^{-4}, 2^{3})$ \\
 & 20 \% & 76.79 & 80.89 & 74.07 & 86.42 & 86.79 & 86.79 \\
 &  & $(10^{-4}, 2^{3})$ & $(10^{4}, 2^{3})$ & $(10^{-1}, 10^{-1}, 2^{-5})$ & $(10^{-1}, 10^{1}, 2^{5})$ & $(10^{-3}, 10^{-2}, 2^{2})$ & $(10^{-3}, 10^{-4}, 2^{5})$ \\ \hline
{\begin{tabular}[c]{@{}c@{}}led7digit-0-2-4-5-6-7-8-9\_vs\_1\\ (443 x 8)\end{tabular}} & 0 \% & 81.95 & 100 & 93.98 & 94.74 & 93.23 & 93.23 \\
 &  & $(1, 1)$ & $(10^{5}, 2^{5})$ & $(10^{-5}, 10^{-4}, 2^{-3})$ & $(10^{-5}, 10^{-5}, 2^{5})$ & $(10^{-5}, 10^{-5}, 2^{3})$ & $(10^{-5}, 10^{-5}, 2^{5})$ \\
 & 5 \% & 93.98 & 93.23 & 93.98 & 95.49 & 93.98 & 93.98 \\
 &  & $(10^{-1}, 1)$ & $(10^{3}, 2^{3})$ & $(10^{-5}, 10^{-5}, 2^{-5})$ & $(10^{-1}, 1, 2^{5})$ & $(10^{-5}, 10^{-2}, 2^{2})$ & $(10^{-5}, 10^{-5}, 2^{5})$ \\
 & 10 \% & 93.98 & 90.87 & 91.73 & 94.74 & 93.23 & 93.23 \\
 &  & $(10^{-1}, 1)$ & $(10^{5}, 2^{5})$ & $(10^{-2}, 10^{-2}, 2^{-5})$ & $(10^{-5}, 10^{-5}, 2^{5})$ & $(10^{-5}, 10^{-3}, 2^{4})$ & $(10^{-5}, 10^{-5}, 2^{5})$ \\
 & 15 \% & 87.22 & 80.76 & 81.73 & 84.21 & 93.23 & 93.23 \\
 &  & $(10^{-1}, 1)$ & $(10^{5}, 2^{5})$ & $(10^{-5}, 10^{-5}, 2^{-5})$ & $(1, 1, 2^{5})$ & $(10^{-5}, 10^{-5}, 2^{4})$ & $(10^{-5}, 10^{-5}, 2^{5})$ \\
 & 20 \% & 93.23 & 92.89 & 90.23 & 93.23 & 93.23 & 93.23 \\
 &  & $(10^{-5}, 2^{-5})$ & $(10^{5}, 2^{5})$ & $(10^{-5}, 10^{-5}, 2^{-5})$ & $(10^{-2}, 10^{-3}, 2^{5})$ & $(10^{-5}, 10^{-3}, 2^{3})$ & $(10^{-5}, 10^{-5}, 2^{5})$ \\ \hline
  \multicolumn{8}{l}{$^{\dagger}$ represents the proposed models. ACC represents the accuracy metric.}
\end{tabular}}
\end{table*}

\begin{table*}[htp]
\ContinuedFloat
\centering
    \caption{(Continued)}
    \resizebox{0.78\textwidth}{!}{                                                                                          
    \begin{tabular}{cccccccc}
\hline
Dataset & Noise & \begin{tabular}[c]{@{}c@{}}SVM \cite{cortes1995support} \\ ACC (\%)\\ ($d_1$, $\sigma$)\end{tabular} & \begin{tabular}[c]{@{}c@{}}GBSVM (PSO) \cite{xia2022gbsvm}\\ ACC (\%)\\ ($d_1$, $\sigma$)\end{tabular} & \begin{tabular}[c]{@{}c@{}}TSVM \cite{khemchandani2007twin} \\ ACC (\%)\\ ($d_1$, $d_2$, $\sigma$)\end{tabular} & \begin{tabular}[c]{@{}c@{}}GBTSVM$^{\dagger}$\\ ACC (\%)\\ ($d_1$, $d_2$, $\sigma$)\end{tabular} & \begin{tabular}[c]{@{}c@{}}LS-GBTSVM$^{\dagger}$\\ ACC (\%)\\ ($d_1$, $d_3$, $\sigma$)\end{tabular} & \begin{tabular}[c]{@{}c@{}}LS-GBSVM (SMO)$^{\dagger}$\\ ACC (\%)\\ ($d_1$, $d_3$, $\sigma$)\end{tabular} \\ \hline
 {\begin{tabular}[c]{@{}c@{}}mammographic\\ (961 x 6)\end{tabular}} & 0 \% & 52.94 & 75.76 & 79.93 & 81.31 & 82.94 & 79.09 \\
 &  & $(10^{-5}, 2^{-5})$ & $(10^{5}, 2^{5})$ & $(10^{-1}, 10^{-1}, 2^{-5})$ & $(10^{-1}, 10^{-1}, 2^{5})$ & $(10^{-5}, 10^{-1}, 2^{3})$ & $(10^{-3}, 10^{-2}, 2^{3})$ \\
 & 5 \% & 52.94 & 80 & 81.66 & 81.31 & 82.94 & 82.94 \\
 &  & $(10^{-4}, 2^{1})$ & $(10^{5}, 2^{5})$ & $(10^{-1}, 10^{-1}, 2^{-5})$ & $(1, 1, 2^{5})$ & $(10^{-4}, 1, 2^{4})$ & $(10^{-4}, 10^{-1}, 2^{1})$ \\
 & 10 \% & 72.94 & 80.76 & 81.31 & 81.31 & 77.06 & 72.94 \\
 &  & $(1, 2^{-4})$ & $(10^{5}, 2^{5})$ & $(10^{-1}, 10^{-1}, 2^{-5})$ & $(10^{-1}, 10^{-1}, 2^{5})$ & $(10^{1}, 10^{-3}, 2^{4})$ & $(1, 1, 2^{5})$ \\
 & 15 \% & 52.94 & 100 & 81.66 & 82.35 & 80.06 & 82.94 \\
 &  & $(1, 2^{-4})$ & $(10^{5}, 2^{5})$ & $(10^{-1}, 10^{-1}, 2^{-5})$ & $(1, 1, 2^{5})$ & $(1, 10^{-4}, 2^{4})$ & $(10^{-4}, 10^{-4}, 2^{5})$ \\
 & 20 \% & 68.13 & 74.05 & 80.28 & 79.58 & 79.94 & 79.94 \\
 &  & $(10^{-1}, 2^{-1})$ & $(10^{4}, 2^{5})$ & $(10^{-1}, 10^{-1}, 2^{-5})$ & $(1, 1, 2^{5})$ & $(1, 10^{-5}, 2^{4})$ & $(10^{-4}, 10^{-3}, 2^{2})$ \\ \hline
 {\begin{tabular}[c]{@{}c@{}}monks\_3\\ (554 x 7)\end{tabular}} & 0 \% & 46.11 & 69.52 & 75.21 & 80.24 & 76.11 & 76.11 \\
 &  & $(10^{-5}, 2^{-5})$ & $(10^{5}, 2^{5})$ & $(1, 10^{-1}, 2^{-2})$ & $(1, 10^{1}, 2^{5})$ & $(10^{-5}, 10^{-1}, 2^{5})$ & $(10^{-3}, 10^{-2}, 2^{5})$ \\
 & 5 \% & 69.89 & 70.67 & 74.61 & 77.84 & 76.11 & 76.11 \\
 &  & $(10^{-5}, 2^{-5})$ & $(10^{5}, 2^{5})$ & $(10^{-1}, 10^{-1}, 2^{-3})$ & $(1, 10^{1}, 2^{5})$ & $(10^{-1}, 10^{-1}, 2^{5})$ & $(10^{-4}, 10^{-5}, 2^{5})$ \\
 & 10 \% & 66.11 & 71.26 & 81.62 & 80.24 & 78.26 & 76.11 \\
 &  & $(10^{-5}, 2^{-5})$ & $(10^{5}, 2^{5})$ & $(10^{-1}, 10^{-1}, 2^{-3})$ & $(1, 10^{1}, 2^{5})$ & $(10^{1}, 10^{-2}, 2^{3})$ & $(10^{-5}, 10^{-5}, 2^{4})$ \\
 & 15 \% & 76.11 & 80 & 85.03 & 80.84 & 83.89 & 86.11 \\
 &  & $(10^{-5}, 2^{-5})$ & $(10^{5}, 2^{5})$ & $(10^{-1}, 10^{-1}, 2^{-5})$ & $(10^{1}, 1, 2^{5})$ & $(10^{4}, 10^{-3}, 2^{5})$ & $(10^{-1}, 10^{-1}, 2^{5})$ \\
 & 20 \% & 76.11 & 70.06 & 79.64 & 82.04 & 83.89 & 86.11 \\
 &  & $(10^{-5}, 2^{-5})$ & $(10^{5}, 2^{5})$ & $(10^{-1}, 10^{-1}, 2^{-5})$ & $(1^{1}, 10^{1}, 2^{5})$ & $(1, 10^{-5}, 2^{1})$ & $(10^{-1}, 10^{-1}, 2^{5})$ \\ \hline
 {\begin{tabular}[c]{@{}c@{}}mushroom\\ (8124 x 22)\end{tabular}} & 0 \% & 63.41 & 71.89 & 70.65 & 84.91 & 70.86 & 70.86 \\
 &  & $(10^{-3}, 2^{1})$ & $(10^{2}, 2^{-5})$ & $(10^{2}, 10^{3}, 2^{5})$ & $(1, 10^{-1}, 2^{3})$ & $(10^{-2}, 10^{-5}, 2^{3})$ & $(10^{-4}, 10^{-5}, 2^{5})$ \\
 & 5 \% & 57.88 & 100 & 90.02 & 94.38 & 90.86 & 90.86 \\
 &  & $(10^{-3}, 2^{1})$ & $(10^{3}, 2^{-5})$ & $(10^{-1}, 10^{3}, 2^{-2})$ & $(10^{-1}, 10^{-1}, 2^{5})$ & $(1, 10^{-3}, 2^{3})$ & $(10^{-4}, 10^{-5}, 2^{5})$ \\
 & 10 \% & 52.91 & 85.42 & 88.03 & 99.67 & 86.46 & 86.46 \\
 &  & $(10^{-3}, 2^{1})$ & $(10^{1}, 2^{-4})$ & $(10^{2}, 10^{-1}, 2^{-4})$ & $(10^{-1}, 10^{-1}, 2^{5})$ & $(10^{-5}, 1, 2^{5})$ & $(10^{-4}, 10^{-5}, 2^{5})$ \\
 & 15 \% & 80.86 & 80.67 & 89.89 & 99.84 & 98.78 & 98.86 \\
 &  & $(10^{-3}, 2^{1})$ & $(10^{3}, 2^{4})$ & $(10^{-1}, 10^{-1}, 2^{-5})$ & $(1, 1, 2^{5})$ & $(10^{-5}, 10^{-5}, 2^{5})$ & $(10^{-4}, 10^{-5}, 2^{5})$ \\
 & 20\% & 80.86 & 82.76 & 90.78 & 99.71 & 95.67 & 95.24 \\
 &  & $(10^{-3}, 2^{1})$ & $(10^{5}, 2^{5})$ & $(10^{3}, 10^{-1}, 2^{-1})$ & $(1, 1, 2^{5})$ & $(10^{-5}, 10^{-5}, 2^{5})$ & $(10^{-4}, 10^{-5}, 2^{5})$ \\ \hline
{\begin{tabular}[c]{@{}c@{}}musk\_1\\ (476 x 167)\end{tabular}} & 0 \% & 53.15 & 46.85 & 83.15 & 91.61 & 89.23 & 83.15 \\
 &  & $(10^{-5}, 2^{-5})$ & $(10^{-1}, 2^{5})$ & $(10^{-5}, 10^{-5}, 2^{-5})$ & $(10^{-1}, 1, 2^{5})$ & $(10^{-2}, 10^{-3}, 2^{5})$ & $(10^{-4}, 10^{-5}, 2^{5})$ \\
 & 5 \% & 53.15 & 41.96 & 53.15 & 89.51 & 89.23 & 53.15 \\
 &  & $(10^{-5}, 2^{-5})$ & $(10^{5}, 2^{5})$ & $(10^{-4}, 10^{-5}, 2^{-5})$ & $(10^{-1}, 1, 2^{5})$ & $(10^{4}, 10^{-5}, 2^{3})$ & $(10^{-5}, 10^{-5}, 2^{5})$ \\
 & 10 \% & 53.15 & 61.54 & 53.15 & 91.61 & 55.24 & 53.15 \\
 &  & $(10^{-5}, 2^{-5})$ & $(10^{4}, 2^{5})$ & $(10^{-4}, 10^{-5}, 2^{-5})$ & $(10^{-1}, 1, 2^{5})$ & $(10^{-4}, 10^{-3}, 2^{1})$ & $(10^{-5}, 10^{-5}, 2^{5})$ \\
 & 15 \% & 69.23 & 52.27 & 53.15 & 81.12 & 46.85 & 53.15 \\
 &  & $(10^{-5}, 2^{-5})$ & $(10^{1}, 2^{3})$ & $(10^{-4}, 10^{-5}, 2^{-5})$ & $(10^{-5}, 10^{-5}, 2^{5})$ & $(10^{1}, 10^{-4}, 2^{4})$ & $(10^{-3}, 10^{-3}, 2^{5})$ \\
 & 20 \% & 58.04 & 51.76 & 53.15 & 53.15 & 56.64 & 53.15 \\
 &  & $(10^{-5}, 2^{-4})$ & $(1, 1)$ & $(10^{-4}, 10^{-5}, 2^{-5})$ & $(10^{-1}, 10^{-5}, 2^{4})$ & $(10^{-1}, 10^{-3}, 2^{4})$ & $(10^{-5}, 10^{-5}, 2^{4})$ \\ \hline
{\begin{tabular}[c]{@{}c@{}}new-thyroid1\\ (215 x 16)\end{tabular}} & 0 \% & 87.69 & 100 & 98.46 & 95.38 & 87.69 & 87.69 \\
 &  & $(10^{-5}, 2^{-5})$ & $(10^{2}, 2^{2})$ & $(10^{-2}, 10^{-2}, 2^{-3})$ & $(10^{-5}, 10^{-5}, 2^{4})$ & $(10^{-5}, 10^{-2}, 2^{2})$ & $(10^{-5}, 10^{-5}, 2^{5})$ \\
 & 5 \% & 87.69 & 100 & 96.92 & 100 & 87.69 & 87.69 \\
 &  & $(10^{-5}, 2^{-5})$ & $(10^{5}, 2^{5})$ & $(10^{-1}, 1, 2^{-2})$ & $(10^{-5}, 10^{-5}, 2^{4})$ & $(10^{-1}, 10^{-4}, 2^{4})$ & $(10^{-5}, 10^{-5}, 2^{5})$ \\
 & 10 \% & 87.69 & 100 & 98.46 & 95.38 & 87.69 & 87.69 \\
 &  & $(10^{-5}, 2^{-5})$ & $(10^{5}, 2^{5})$ & $(10^{-1}, 10^{-1}, 2^{-3})$ & $(10^{-5}, 10^{-5}, 2^{5})$ & $(10^{-5}, 10^{-2}, 2^{2})$ & $(10^{-5}, 10^{-5}, 2^{5})$ \\
 & 15 \% & 87.69 & 87.69 & 87.69 & 100 & 87.69 & 87.69 \\
 &  & $(10^{-5}, 2^{-5})$ & $(10^{5}, 2^{5})$ & $(10^{1}, 1, 2^{-2})$ & $(10^{-5}, 10^{-5}, 2^{5})$ & $(10^{2}, 10^{-2}, 2^{3})$ & $(10^{-5}, 10^{-5}, 2^{5})$ \\
 & 20 \% & 87.69 & 85.78 & 92.31 & 98.46 & 87.69 & 87.69 \\
 &  & $(10^{-5}, 2^{-5})$ & $(10^{5}, 2^{5})$ & $(10^{-5}, 10^{-5}, 2^{-5})$ & $(10^{-5}, 10^{-5}, 2^{5})$ & $(10^{-4}, 10^{-5}, 2^{3})$ & $(10^{-5}, 10^{-5}, 2^{5})$ \\ \hline
{\begin{tabular}[c]{@{}c@{}}oocytes\_merluccius\_nucleus\_4d\\ (1022 x 42)\end{tabular}} & 0 \% & 64.82 & 64.82 & 76.22 & 77.2 & 74.82 & 74.82 \\
 &  & $(10^{-5}, 2^{-5})$ & $(10^{5}, 2^{5})$ & $(10^{-2}, 10^{-1}, 2^{-5})$ & $(1, 1, 2^{4})$ & $(10^{-5}, 10^{-2}, 2^{4})$ & $(10^{-4}, 10^{-4}, 2^{3})$ \\
 & 5 \% & 64.82 & 64.82 & 70.18 & 77.2 & 74.82 & 64.82 \\
 &  & $(10^{-5}, 2^{-5})$ & $(10^{5}, 2^{4})$ & $(10^{-2}, 10^{-2}, 2^{-4})$ & $(1, 10, 2^{5})$ & $(10^{-5}, 10^{-2}, 2^{2})$ & $(10^{-3}, 10^{-2}, 2^{1})$ \\
 & 10 \% & 64.82 & 73.67 & 74.92 & 77.2 & 75.05 & 74.82 \\
 &  & $(10^{-5}, 2^{-5})$ & $(10^{5}, 2^{1})$ & $(10^{-5}, 10^{-5}, 2^{-5})$ & $(1, 1, 2^{5})$ & $(10^{-1}, 10^{-3}, 2^{3})$ & $(10^{-4}, 10^{-4}, 2^{5})$ \\
 & 15 \% & 64.82 & 70.78 & 76.55 & 77.85 & 74.82 & 74.82 \\
 &  & $(10^{-5}, 2^{-5})$ & $(10^{5}, 2^{5})$ & $(10^{-5}, 10^{-5}, 2^{-5})$ & $(10^{2}, 10^{1}, 2^{3})$ & $(10^{-5}, 10^{-5}, 2^{4})$ & $(10^{-4}, 10^{-4}, 2^{4})$ \\
 & 20 \% & 64.82 & 70.78 & 71.34 & 74.59 & 74.82 & 74.82 \\
 &  & $(10^{-5}, 2^{-5})$ & $(10^{5}, 2^{5})$ & $(10^{-5}, 10^{-5}, 2^{-5})$ & $(10^{1}, 10^{1}, 2^{5})$ & $(10^{-5}, 10^{-5}, 2^{5})$ & $(10^{-4}, 10^{-4}, 2^{5})$ \\ \hline
 {\begin{tabular}[c]{@{}c@{}}ozone\\ (2536 x 6)\end{tabular}} & 0 \% & 96.58 & 80 & 96.58 & 94.09 & 96.58 & 96.58 \\
 &  & $(10^{-5}, 2^{-5})$ & $(10^{5}, 2^{5})$ & $(10^{-5}, 10^{-5}, 2^{-5})$ & $(10^{1}, 1, 2^{5})$ & $(10^{-5}, 10^{-5}, 2^{3})$ & $(10^{-5}, 10^{-5}, 2^{5})$ \\
 & 5 \% & 94.58 & 80.67 & 94.58 & 96.58 & 96.58 & 96.58 \\
 &  & $(10^{-5}, 2^{-5})$ & $(10^{2}, 2^{2})$ & $(10^{-5}, 10^{-5}, 2^{-5})$ & $(10^{-5}, 10^{-5}, 2^{5})$ & $(10^{-5}, 10^{-2}, 2^{-1})$ & $(10^{-5}, 10^{-5}, 2^{5})$ \\
 & 10 \% & 96.58 & 96.58 & 96.58 & 96.58 & 96.58 & 96.58 \\
 &  & $(10^{-5}, 2^{-5})$ & $(10^{2}, 2^{-2})$ & $(10^{-5}, 10^{-5}, 2^{-5})$ & $(10^{-5}, 10^{-5}, 2^{5})$ & $(10^{-5}, 10^{-5}, 2^{4})$ & $(10^{-5}, 10^{-5}, 2^{3})$ \\
 & 15 \% & 96.58 & 95.65 & 96.58 & 96.58 & 96.58 & 96.58 \\
 &  & $(10^{-5}, 2^{-5})$ & $(10^{1}, 2^{-4})$ & $(10^{-5}, 10^{-5}, 2^{-5})$ & $(10^{-4}, 10^{-5}, 2^{3})$ & $(10^{-5}, 10^{-3}, 2^{5})$ & $(10^{-5}, 10^{-5}, 2^{3})$ \\
 & 20 \% & 96.58 & 95.78 & 96.58 & 96.45 & 96.58 & 96.58 \\
 &  & $(10^{-5}, 2^{-5})$ & $(10^{2}, 2^{-5})$ & $(10^{-5}, 10^{-5}, 2^{-5})$ & $(10^{-5}, 10^{-5}, 2^{3})$ & $(10^{-5}, 10^{-3}, 2^{5})$ & $(10^{-5}, 10^{-5}, 2^{4})$ \\ \hline
{\begin{tabular}[c]{@{}c@{}}ringnorm\\ (7400 x 21)\end{tabular}} & 0 \% & 90.42 & 92.95 & 92.65 & 96.94 & 91.8 & 88.2 \\
 &  & $(1, 2^{1})$ & $(10^{5}, 2^{5})$ & $(10^{-5}, 10^{-5}, 2^{-5})$ & $(10^{1}, 1, 2^{5})$ & $(10^{-5}, 10^{-5}, 2^{5})$ & $(10^{-4}, 10^{-3}, 2^{2})$ \\
 & 5 \% & 98.51 & 100 & 90.14 & 90.95 & 91.8 & 91.8 \\
 &  & $(10^{-3}, 2^{1})$ & $(10^{-5}, 2^{5})$ & $(10^{-5}, 10^{-5}, 2^{-5})$ & $(1, 1, 2^{2})$ & $(10^{2}, 10^{-3}, 2^{5})$ & $(10^{-3}, 10^{-3}, 2^{5})$ \\
 & 10 \% & 91.6 & 87.45 & 93.11 & 95.09 & 88.2 & 87.8 \\
 &  & $(10^{-3}, 2^{1})$ & $(10^{5}, 2^{5})$ & $(10^{-2}, 10^{-1}, 2^{-4})$ & $(10^{1}, 10^{1}, 2^{3})$ & $(1, 10^{-5}, 2^{5})$ & $(10^{-3}, 10^{-4}, 2^{5})$ \\
 & 15 \% & 98.24 & 91.54 & 95.59 & 93.83 & 91.8 & 91.8 \\
 &  & $(10^{-3}, 2^{1})$ & $(10^{5}, 2^{5})$ & $(10^{-2}, 10^{-1}, 2^{-4})$ & $(10^{-1}, 10^{-1}, 2^{5})$ & $(10^{-3}, 10^{-5}, 2^{5})$ & $(10^{-5}, 10^{-5}, 2^{5})$ \\
 & 20 \% & 97.12 & 92.67 & 97.52 & 90.99 & 88.2 & 88.2 \\
 &  & $(10^{-3}, 2^{1})$ & $(10^{5}, 2^{5})$ & $(10^{-2}, 10^{-1}, 2^{-4})$ & $(10^{-1}, 10^{-1}, 2^{5})$ & $(1, 10^{-3}, 2^{3})$ & $(10^{-5}, 10^{-5}, 2^{5})$ \\ \hline
 {\begin{tabular}[c]{@{}c@{}}shuttle-6\_vs\_2-3\\ (230 x 10)\end{tabular}} & 0 \% & 95.65 & 90.67 & 97.1 & 98.55 & 95.65 & 95.65 \\
 &  & $(10^{-5}, 2^{-5})$ & $(10^{5}, 2^{5})$ & $(10^{-5}, 10^{-4}, 2^{-4})$ & $(10^{-2}, 10^{-1}, 2^{4})$ & $(10^{-3}, 10^{-5}, 2^{4})$ & $(10^{-4}, 10^{-5}, 2^{5})$ \\
 & 5 \% & 95.65 & 90.54 & 95.65 & 98.55 & 95.65 & 95.65 \\
 &  & $(10^{-5}, 2^{-5})$ & $(10^{5}, 2^{5})$ & $(10^{-3}, 10^{-5}, 2^{-4})$ & $(10^{-1}, 10^{-1}, 2^{5})$ & $(10^{-4}, 10^{-5}, 2^{5})$ & $(10^{-5}, 10^{-5}, 2^{5})$ \\
 & 10 \% & 95.65 & 80.76 & 87.1 & 95.65 & 95.65 & 95.65 \\
 &  & $(10^{-5}, 2^{-5})$ & $(10^{-2}, 2^{3})$ & $(1, 10^{-1}, 2^{-5})$ & $(10^{-5}, 10^{-5}, 2^{5})$ & $(10^{-5}, 10^{-5}, 2^{5})$ & $(10^{-5}, 10^{-5}, 2^{5})$ \\
 & 15 \% & 95.65 & 90.67 & 95.65 & 100 & 95.65 & 95.65 \\
 &  & $(10^{-5}, 2^{-5})$ & $(10^{-2}, 2^{3})$ & $(10^{-4}, 10^{-5}, 2^{-5})$ & $(10^{-1}, 10^{-1}, 2^{3})$ & $(10^{-5}, 10^{-5}, 2^{5})$ & $(10^{-5}, 10^{-5}, 2^{5})$ \\
 & 20 \% & 95.65 & 95.65 & 95.65 & 97.1 & 95.65 & 95.65 \\
 &  & $(10^{-5}, 2^{-5})$ & $(10^{-2}, 2^{3})$ & $(10^{-3}, 10^{-5}, 2^{-5})$ & $(10^{-2}, 10^{-2}, 2^{3})$ & $(10^{-5}, 10^{-5}, 2^{5})$ & $(10^{-5}, 10^{-5}, 2^{5})$ \\ \hline
 \multicolumn{8}{l}{$^{\dagger}$ represents the proposed models. ACC represents the accuracy metric.}
\end{tabular}}
\end{table*}

\begin{table*}[htp]
\ContinuedFloat
\centering
    \caption{ (Continued)}
    \resizebox{0.78\textwidth}{!}{                                                                                          
    \begin{tabular}{cccccccc}
\hline
Dataset & Noise & \begin{tabular}[c]{@{}c@{}}SVM \cite{cortes1995support} \\ ACC (\%)\\ ($d_1$, $\sigma$)\end{tabular} & \begin{tabular}[c]{@{}c@{}}GBSVM (PSO) \cite{xia2022gbsvm}\\ ACC (\%)\\ ($d_1$, $\sigma$)\end{tabular} & \begin{tabular}[c]{@{}c@{}}TSVM \cite{khemchandani2007twin} \\ ACC (\%)\\ ($d_1$, $d_2$, $\sigma$)\end{tabular} & \begin{tabular}[c]{@{}c@{}}GBTSVM$^{\dagger}$\\ ACC (\%)\\ ($d_1$, $d_2$, $\sigma$)\end{tabular} & \begin{tabular}[c]{@{}c@{}}LS-GBTSVM$^{\dagger}$\\ ACC (\%)\\ ($d_1$, $d_3$, $\sigma$)\end{tabular} & \begin{tabular}[c]{@{}c@{}}LS-GBSVM (SMO)$^{\dagger}$\\ ACC (\%)\\ ($d_1$, $d_3$, $\sigma$)\end{tabular} \\ \hline
{\begin{tabular}[c]{@{}c@{}}spambase\\ (4601 x 58)\end{tabular}} & 0 \% & 62.2 & 70.65 & 84.79 & 90.88 & 88.49 & 82.27 \\
 &  & $(10^{-1}, 1)$ & $(10^{-5}, 2^{3})$ & $(10^{-2}, 1, 2^{-3})$ & $(10^{1}, 10^{1}, 2^{5})$ & $(10^{-4}, 10^{-2}, 2^{5})$ & $(10^{-5}, 10^{-5}, 2^{5})$ \\
 & 5 \% & 62.27 & 62 & 81.82 & 90.59 & 88.49 & 82.27 \\
 &  & $(1, 2^{-1})$ & $(10^{-2}, 2^{3})$ & $(1, 10^{-1}, 2^{-4})$ & $(10^{-2}, 10^{-2}, 2^{5})$ & $(10^{-5}, 10^{-2}, 2^{4})$ & $(10^{-4}, 10^{-4}, 2^{5})$ \\
 & 10 \% & 66.33 & 79.72 & 80.88 & 90.88 & 92.27 & 92.27 \\
 &  & $(10^{1}, 2^{-1})$ & $(10^{5}, 2^{3})$ & $(10^{-1}, 10^{-1}, 2^{-5})$ & $(10^{1}, 10^{1}, 2^{5})$ & $(10^{-4}, 1, 2^{4})$ & $(10^{-4}, 10^{-1}, 2^{3})$ \\
 & 15 \% & 67.05 & 70.87 & 79.44 & 89.93 & 72.35 & 72.27 \\
 &  & $(1, 2^{-1})$ & $(10^{-1}, 2^{5})$ & $(10^{-1}, 10^{-1}, 2^{-5})$ & $(10^{-1}, 10^{-1}, 2^{3})$ & $(10^{-3}, 10^{-1}, 2^{5})$ & $(10^{-3}, 10^{-2}, 2^{1})$ \\
 & 20 \% & 64.52 & 72.89 & 77.55 & 85.23 & 82.35 & 82.27 \\
 &  & $(1, 2^{-1})$ & $(10^{3}, 2^{-1})$ & $(10^{-1}, 10^{-1}, 2^{-5})$ & $(10^{1}, 1, 2^{4})$ & $(10^{-1}, 10^{-3}, 2^{5})$ & $(10^{-5}, 10^{-5}, 2^{5})$ \\ \hline
 {\begin{tabular}[c]{@{}c@{}}spectf\\ (267 x 45)\end{tabular}} & 0 \% & 80.25 & 80.25 & 79.42 & 85.19 & 80.25 & 80.25 \\
 &  & $(10^{-5}, 2^{-5})$ & $(10^{-5}, 2^{5})$ & $(10^{-1}, 10^{-1}, 2^{-5})$ & $(10^{-1}, 10^{-3}, 2^{4})$ & $(10^{-5}, 10^{-2}, 2^{2})$ & $(10^{-4}, 10^{-3}, 2^{3})$ \\
 & 5 \% & 80.25 & 75.31 & 86.42 & 82.72 & 80.25 & 80.25 \\
 &  & $(10^{-5}, 2^{-5})$ & $(10^{2}, 2^{5})$ & $(10^{-1}, 10^{-1}, 2^{-5})$ & $(10^{-3}, 10^{-5}, 2^{5})$ & $(10^{-1}, 10^{-3}, 2^{4})$ & $(10^{-4}, 10^{-3}, 2^{3})$ \\
 & 10 \% & 80.25 & 80.25 & 82.42 & 85.19 & 86.25 & 80.25 \\
 &  & $(10^{-5}, 2^{-5})$ & $(10^{5}, 2^{5})$ & $(10^{-1}, 10^{-1}, 2^{-5})$ & $(10^{-1}, 10^{-3}, 2^{4})$ & $(10^{-5}, 10^{-2}, 2^{3})$ & $(10^{-4}, 10^{-3}, 2^{3})$ \\
 & 15 \% & 80.25 & 80.25 & 86.42 & 85.19 & 80.25 & 80.25 \\
 &  & $(10^{-5}, 2^{-5})$ & $(10^{-5}, 2^{3})$ & $(1, 10^{-1}, 2^{-5})$ & $(10^{-1}, 10^{-1}, 2^{3})$ & $(1, 10^{-4}, 2^{5})$ & $(10^{-4}, 10^{-3}, 2^{3})$ \\
 & 20 \% & 80.25 & 81.48 & 85.19 & 87.9 & 80.25 & 80.25 \\
 &  & $(10^{-5}, 2^{-5})$ & $(1, 2^{3})$ & $(1, 10^{-1}, 2^{-5})$ & $(10^{-4}, 10^{-5}, 2^{5})$ & $(10^{-1}, 10^{-3}, 2^{3})$ & $(10^{-4}, 10^{-3}, 2^{3})$ \\ \hline
{\begin{tabular}[c]{@{}c@{}}tic\_tac\_toe\\ (958 x 10)\end{tabular}} & 0 \% & 66.32 & 95 & 95 & 98.96 & 96.32 & 96.32 \\
 &  & $(10^{-5}, 2^{-5})$ & $(10^{5}, 2^{5})$ & $(10^{-1}, 10^{-2}, 2^{-2})$ & $(10^{1}, 10^{1}, 2^{4})$ & $(10^{-3}, 10^{-5}, 2^{4})$ & $(10^{-5}, 10^{-5}, 2^{5})$ \\
 & 5 \% & 66.32 & 92.89 & 97.57 & 97.92 & 96.32 & 96.32 \\
 &  & $(10^{-5}, 2^{-5})$ & $(10^{5}, 2^{5})$ & $(10^{-1}, 10^{-2}, 2^{-5})$ & $(10^{1}, 10^{1}, 2^{4})$ & $(10^{5}, 10^{-3}, 2^{3})$ & $(10^{-5}, 10^{-5}, 2^{5})$ \\
 & 10 \% & 66.32 & 66.67 & 97.92 & 98.96 & 66.32 & 66.32 \\
 &  & $(10^{-5}, 2^{-5})$ & $(10^{5}, 2^{5})$ & $(10^{-5}, 10^{-5}, 2^{-5})$ & $(10^{1}, 10^{1}, 2^{4})$ & $(10^{-3}, 1, 2^{4})$ & $(10^{-5}, 10^{-5}, 2^{5})$ \\
 & 15 \% & 66.32 & 90.76 & 93.4 & 94.44 & 96.32 & 96.32 \\
 &  & $(10^{-5}, 2^{-5})$ & $(10^{5}, 2^{5})$ & $(10^{-5}, 10^{-5}, 2^{-5})$ & $(10^{1}, 10^{1}, 2^{4})$ & $(10^{1}, 10^{-3}, 2^{4})$ & $(10^{-4}, 10^{-5}, 2^{5})$ \\
 & 20 \% & 86.32 & 87.57 & 92.71 & 97.22 & 93.68 & 96.32 \\
 &  & $(10^{-5}, 2^{-5})$ & $(10^{5}, 2^{5})$ & $(10^{-2}, 10^{-2}, 2^{-2})$ & $(10^{1}, 10^{1}, 2^{4})$ & $(10^{-3}, 10^{-5}, 2^{3})$ & $(10^{-5}, 10^{-3}, 2^{5})$ \\ \hline
{\begin{tabular}[c]{@{}c@{}}vehicle1\\ (846 x 19)\end{tabular}} & 0 \% & 75.98 & 73.62 & 80.31 & 81.5 & 76.38 & 76.38 \\
 &  & $(10^{-5}, 2^{-5})$ & $(10^{4}, 2^{5})$ & $(10^{-2}, 10^{-2}, 2^{-5})$ & $(1, 1, 2^{5})$ & $(10^{-5}, 10^{-5}, 2^{5})$ & $(10^{-5}, 10^{-3}, 2^{2})$ \\
 & 5 \% & 76.38 & 80.89 & 76.77 & 77.95 & 76.38 & 76.38 \\
 &  & $(10^{-5}, 2^{-5})$ & $(10^{5}, 2^{5})$ & $(1, 1, 2^{-5})$ & $(10^{1}, 10^{1}, 2^{5})$ & $(10^{-5}, 10^{-5}, 2^{5})$ & $(10^{-5}, 10^{-3}, 2^{3})$ \\
 & 10 \% & 76.38 & 72.87 & 77.17 & 81.5 & 76.38 & 76.38 \\
 &  & $(10^{-5}, 2^{-5})$ & $(10^{5}, 2^{5})$ & $(10^{-5}, 10^{-5}, 2^{-5})$ & $(1, 1, 2^{5})$ & $(10^{-5}, 10^{-5}, 2^{5})$ & $(10^{-3}, 10^{-2}, 2^{3})$ \\
 & 15 \% & 76.38 & 70 & 79.13 & 76.77 & 71.65 & 76.38 \\
 &  & $(10^{-5}, 2^{-5})$ & $(10^{5}, 2^{5})$ & $(1, 10^{-2}, 2^{-4})$ & $(10^{-4}, 10^{-5}, 2^{5})$ & $(10^{-4}, 10^{-2}, 2^{5})$ & $(10^{-5}, 10^{-5}, 2^{5})$ \\
 & 20 \% & 76.38 & 76.38 & 78.35 & 81.89 & 74.02 & 76.38 \\
 &  & $(10^{-5}, 2^{-5})$ & $(10^{5}, 2^{5})$ & $(1, 10^{-2}, 2^{-4})$ & $(1, 1, 2^{5})$ & $(10^{-2}, 10^{1}, 2^{4})$ & $(10^{-5}, 10^{-5}, 2^{5})$ \\ \hline
{\begin{tabular}[c]{@{}c@{}}vehicle2\\ (846 x 19)\end{tabular}} & 0 \% & 77.95 & 50.39 & 76.46 & 94.49 & 72.05 & 72.05 \\
 &  & $(1, 1)$ & $(10^{4}, 2^{3})$ & $(10^{-2}, 10^{-1}, 2^{-2})$ & $(10^{-1}, 10^{-1}, 2^{5})$ & $(10^{-4}, 10^{-3}, 2^{2})$ & $(10^{-3}, 10^{-3}, 2^{4})$ \\
 & 5 \% & 80.31 & 72.05 & 94.49 & 95.67 & 72.05 & 72.05 \\
 &  & $(10^{-1}, 1)$ & $(10^{5}, 2^{2})$ & $(10^{-2}, 10^{-1}, 2^{-3})$ & $(1, 1, 2^{5})$ & $(10^{1}, 10^{-4}, 2^{5})$ & $(10^{-4}, 10^{-4}, 2^{3})$ \\
 & 10 \% & 81.5 & 72.05 & 89.76 & 94.49 & 92.05 & 92.05 \\
 &  & $(10^{-1}, 1)$ & $(10^{5}, 2^{2})$ & $(1, 1, 2^{-5})$ & $(10^{-1}, 10^{-1}, 2^{5})$ & $(10^{-5}, 10^{-2}, 2^{5})$ & $(10^{-5}, 10^{-5}, 2^{5})$ \\
 & 15 \% & 81.89 & 90.87 & 91.73 & 92.91 & 92.05 & 92.05 \\
 &  & $(10^{-1}, 1)$ & $(10^{5}, 2^{5})$ & $(10^{-2}, 10^{-2}, 2^{-5})$ & $(10^{-1}, 10^{-1}, 2^{5})$ & $(10^{-5}, 10^{-5}, 2^{5})$ & $(10^{-5}, 10^{-5}, 2^{5})$ \\
 & 20 \% & 71.02 & 85.67 & 88.98 & 93.31 & 82.05 & 82.05 \\
 &  & $(10^{-1}, 1)$ & $(10^{5}, 2^{5})$ & $(10^{-5}, 10^{-5}, 2^{-5})$ & $(1, 1, 2^{5})$ & $(10^{-5}, 10^{-5}, 2^{5})$ & $(10^{-5}, 10^{-5}, 2^{5})$ \\ \hline
 {\begin{tabular}[c]{@{}c@{}}vertebral\_column\_2clases\\ (310 x 7)\end{tabular}} & 0 \% & 69.89 & 75.56 & 89.25 & 88.17 & 79.89 & 79.89 \\
 &  & $(10^{-5}, 2^{-5})$ & $(10^{5}, 2^{5})$ & $(10^{-2}, 10^{-2}, 2^{-5})$ & $(10^{-1}, 10^{-1}, 2^{4})$ & $(10^{-4}, 10^{-3}, 2^{4})$ & $(10^{-5}, 10^{-5}, 2^{5})$ \\
 & 5 \% & 69.89 & 70.11 & 89.25 & 91.4 & 90.89 & 90.89 \\
 &  & $(10^{-5}, 2^{-5})$ & $(10^{5}, 2^{5})$ & $(10^{-1}, 10^{-1}, 2^{-5})$ & $(10^{-1}, 10^{-1}, 2^{4})$ & $(10^{-4}, 10^{-4}, 2^{3})$ & $(10^{-5}, 10^{-5}, 2^{5})$ \\
 & 10 \% & 69.89 & 80.78 & 86.02 & 88.17 & 88.89 & 88.89 \\
 &  & $(10^{-5}, 2^{-5})$ & $(10^{5}, 2^{5})$ & $(10^{-1}, 10^{-1}, 2^{-5})$ & $(10^{-1}, 10^{-1}, 2^{4})$ & $(10^{-1}, 10^{-4}, 2^{3})$ & $(10^{-5}, 10^{-5}, 2^{5})$ \\
 & 15 \% & 69.89 & 78.65 & 82.8 & 86.02 & 80.11 & 79.89 \\
 &  & $(10^{-5}, 2^{-5})$ & $(10^{5}, 2^{5})$ & $(1, 10^{-1}, 2^{-5})$ & $(10^{1}, 1, 2^{2})$ & $(10^{-5}, 10^{-5}, 2^{5})$ & $(10^{-5}, 10^{-5}, 2^{5})$ \\
 & 20 \% & 76.34 & 75.45 & 79.57 & 68.82 & 75.11 & 69.89 \\
 &  & $(10^{1}, 2^{-1})$ & $(10^{5}, 2^{5})$ & $(10^{-1}, 10^{-2}, 2^{-2})$ & $(10^{-1}, 10^{1}, 2^{5})$ & $(10^{-2}, 10^{-5}, 2^{3})$ & $(10^{-5}, 10^{-5}, 2^{5})$ \\ \hline
{\begin{tabular}[c]{@{}c@{}}wpbc\\ (194 x 34)\end{tabular}} & 0 \% & 77.97 & 77.97 & 77.97 & 76.27 & 77.12 & 77.97 \\
 &  & $(10^{-5}, 2^{-5})$ & $(10^{-5}, 2^{5})$ & $(10^{-5}, 10^{-5}, 2^{-5})$ & $(10^{-5}, 10^{-5}, 2^{5})$ & $(10^{-5}, 10^{-5}, 2^{5})$ & $(10^{-5}, 10^{-5}, 2^{5})$ \\
 & 5 \% & 77.97 & 90 & 77.97 & 69.49 & 77.12 & 77.97 \\
 &  & $(10^{-5}, 2^{-5})$ & $(10^{5}, 2^{5})$ & $(10^{-5}, 10^{-5}, 2^{-5})$ & $(1, 10^{1}, 2^{5})$ & $(10^{-5}, 10^{-3}, 2^{4})$ & $(10^{-5}, 10^{-5}, 2^{5})$ \\
 & 10 \% & 77.97 & 77.97 & 77.97 & 77.97 & 77.97 & 77.97 \\
 &  & $(10^{-5}, 2^{-5})$ & $(10^{5}, 2^{5})$ & $(10^{-5}, 10^{-5}, 2^{-5})$ & $(10^{1}, 1, 2^{5})$ & $(10^{-5}, 10^{-2}, 2^{4})$ & $(10^{-5}, 10^{-5}, 2^{5})$ \\
 & 15 \% & 77.97 & 75.78 & 77.97 & 77.97 & 77.97 & 77.97 \\
 &  & $(10^{-5}, 2^{-5})$ & $(10^{5}, 2^{5})$ & $(10^{-5}, 10^{-5}, 2^{-5})$ & $(10^{-2}, 10^{-5}, 2^{4})$ & $(10^{-5}, 10^{-5}, 2^{4})$ & $(10^{-5}, 10^{-5}, 2^{5})$ \\
 & 20 \% & 69.49 & 77.97 & 77.97 & 77.97 & 77.97 & 77.97 \\
 &  & $(10^{-2}, 2^{2})$ & $(10^{5}, 2^{5})$ & $(10^{-5}, 10^{-5}, 2^{-5})$ & $(10^{-1}, 10^{-3}, 2^{3})$ & $(10^{-5}, 10^{-5}, 2^{5})$ & $(10^{-5}, 10^{-5}, 2^{5})$ \\ \hline
 {\begin{tabular}[c]{@{}c@{}}yeast-0-2-5-6\_vs\_3-7-8-9\\ (1004 x 9)\end{tabular}} & 0 \% & 91.39 & 100 & 84.04 & 94.04 & 91.39 & 91.39 \\
 &  & $(10^{-5}, 2^{-5})$ & $(10^{5}, 2^{5})$ & $(10^{-1}, 10^{-1}, 2^{1})$ & $(1, 1, 2^{5})$ & $(10^{-5}, 10^{-5}, 2^{5})$ & $(10^{-4}, 10^{-4}, 2^{3})$ \\
 & 5 \% & 91.39 & 90.89 & 93.38 & 93.05 & 91.39 & 91.39 \\
 &  & $(10^{-5}, 2^{-5})$ & $(10^{5}, 2^{5})$ & $(1, 10^{-1}, 2^{-3})$ & $(10^{-2}, 10^{-1}, 2^{5})$ & $(10^{-5}, 10^{-2}, 2^{4})$ & $(10^{-3}, 10^{-2}, 2^{3})$ \\
 & 10 \% & 91.39 & 90.47 & 92.38 & 94.04 & 91.39 & 91.39 \\
 &  & $(10^{-5}, 2^{-5})$ & $(10^{5}, 2^{5})$ & $(1, 10^{-1}, 2^{-3})$ & $(10^{-1}, 10^{-1}, 2^{5})$ & $(10^{-2}, 10^{-5}, 2^{4})$ & $(10^{-3}, 10^{-2}, 2^{3})$ \\
 & 15 \% & 91.39 & 90.78 & 91.39 & 92.72 & 91.39 & 91.39 \\
 &  & $(10^{-5}, 2^{-5})$ & $(10^{5}, 2^{5})$ & $(1, 10^{-1}, 2^{-3})$ & $(10^{1}, 1, 2^{4})$ & $(10^{-4}, 10^{-5}, 2^{5})$ & $(10^{-3}, 10^{-3}, 2^{5})$ \\
 & 20 \% & 91.39 & 90.87 & 91.39 & 92.72 & 91.39 & 91.39 \\
 &  & $(10^{-5}, 2^{-5})$ & $(10^{5}, 2^{5})$ & $(10^{1}, 10^{-2}, 2^{-3})$ & $(10^{1}, 1, 2^{4})$ & $(10^{-5}, 10^{-5}, 2^{5})$ & $(10^{-5}, 10^{-5}, 2^{5})$ \\ \hline
 {\begin{tabular}[c]{@{}c@{}}yeast-0-2-5-7-9\_vs\_3-6-8\\ (1004 x 9)\end{tabular}} & 0 \% & 87.5 & 69.48 & 53.15 & 90.79 & 90.73 & 90.73 \\
 &  & $(10^{-5}, 2^{-5})$ & $(10^{5}, 2^{5})$ & $(10^{-1}, 10^{-5}, 2^{-5})$ & $(10^{-5}, 10^{-5}, 2^{5})$ & $(10^{-1}, 10^{-4}, 2^{4})$ & $(10^{-2}, 10^{-2}, 2^{3})$ \\
 & 5 \% & 87.5 & 100 & 87.5 & 90.79 & 90.73 & 87.5 \\
 &  & $(10^{-5}, 2^{-5})$ & $(10^{5}, 2^{5})$ & $(10^{-5}, 10^{-5}, 2^{-5})$ & $(10^{-5}, 10^{-5}, 2^{5})$ & $(10^{-3}, 10^{-4}, 2^{2})$ & $(10^{-5}, 10^{-5}, 2^{5})$ \\
 & 10 \% & 87.5 & 95.03 & 86.18 & 87.5 & 87.5 & 90.73 \\
 &  & $(10^{-5}, 2^{-5})$ & $(10^{5}, 2^{5})$ & $(10^{-3}, 10^{-3}, 2^{-5})$ & $(10^{-2}, 10^{-5}, 2^{3})$ & $(10^{-5}, 10^{-2}, 2^{3})$ & $(10^{-5}, 10^{-5}, 2^{5})$ \\
 & 15 \% & 90.73 & 85.67 & 87.5 & 98.34 & 90.73 & 90.27 \\
 &  & $(10^{-5}, 2^{-5})$ & $(10^{5}, 2^{5})$ & $(10^{-2}, 10^{-3}, 2^{-5})$ & $(1, 1, 2^{5})$ & $(10^{-5}, 10^{-5}, 2^{5})$ & $(10^{-5}, 10^{-5}, 2^{5})$ \\
 & 20 \% & 90.73 & 89.76 & 87.5 & 90.13 & 90.73 & 87.5 \\
 &  & $(10^{-5}, 2^{-5})$ & $(10^{5}, 2^{4})$ & $(10^{-2}, 10^{-5}, 2^{-5})$ & $(10^{-5}, 10^{-5}, 2^{5})$ & $(10^{-5}, 10^{-5}, 2^{5})$ & $(10^{-5}, 10^{-5}, 2^{5})$ \\ \hline
 \multicolumn{8}{l}{$^{\dagger}$ represents the proposed models. ACC represents the accuracy metric.}
\end{tabular}}
\end{table*}

\begin{table*}[htp]
\ContinuedFloat
\centering
    \caption{(Continued)}
    \resizebox{0.8\textwidth}{!}{                                                                                          
    \begin{tabular}{cccccccc}
\hline
Dataset & Noise & \begin{tabular}[c]{@{}c@{}}SVM \cite{cortes1995support} \\ ACC (\%)\\ ($d_1$, $\sigma$)\end{tabular} & \begin{tabular}[c]{@{}c@{}}GBSVM (PSO) \cite{xia2022gbsvm}\\ ACC (\%)\\ ($d_1$, $\sigma$)\end{tabular} & \begin{tabular}[c]{@{}c@{}}TSVM \cite{khemchandani2007twin} \\ ACC (\%)\\ ($d_1$, $d_2$, $\sigma$)\end{tabular} & \begin{tabular}[c]{@{}c@{}}GBTSVM$^{\dagger}$\\ ACC (\%)\\ ($d_1$, $d_2$, $\sigma$)\end{tabular} & \begin{tabular}[c]{@{}c@{}}LS-GBTSVM$^{\dagger}$\\ ACC (\%)\\ ($d_1$, $d_3$, $\sigma$)\end{tabular} & \begin{tabular}[c]{@{}c@{}}LS-GBSVM (SMO)$^{\dagger}$\\ ACC (\%)\\ ($d_1$, $d_3$, $\sigma$)\end{tabular} \\ \hline
 {\begin{tabular}[c]{@{}c@{}}yeast-0-5-6-7-9\_vs\_4 \\ (528 x 9)\end{tabular}} & 0 \% & 91.19 & 100 & 82.45 & 91.19 & 91.19 & 91.19 \\
 &  & $(10^{-5}, 2^{-5})$ & $(10^{3}, 2^{1})$ & $(10^{-3}, 10^{-2}, 2^{-4})$ & $(10^{-2}, 10^{-2}, 2^{5})$ & $(10^{-1}, 10^{-4}, 2^{4})$ & $(10^{-5}, 10^{-5}, 2^{5})$ \\
 & 5 \% & 91.19 & 89.57 & 88.68 & 91.19 & 91.19 & 91.19 \\
 &  & $(10^{-5}, 2^{-5})$ & $(10^{3}, 2^{1})$ & $(10^{-1}, 10^{-1}, 2^{-5})$ & $(1, 10^{-1}, 2^{2})$ & $(10^{-4}, 10^{-2}, 2^{4})$ & $(10^{-5}, 10^{-5}, 2^{5})$ \\
 & 10 \% & 91.19 & 80.54 & 79.31 & 91.19 & 91.19 & 91.19 \\
 &  & $(10^{-5}, 2^{-5})$ & $(10^{3}, 2^{1})$ & $(10^{-1}, 10^{-2}, 2^{-3})$ & $(10^{-2}, 10^{-2}, 2^{3})$ & $(10^{-5}, 10^{-5}, 2^{5})$ & $(10^{-5}, 10^{-5}, 2^{5})$ \\
 & 15 \% & 91.19 & 79.34 & 90.57 & 91.19 & 91.19 & 91.19 \\
 &  & $(10^{-5}, 2^{-5})$ & $(10^{3}, 2^{2})$ & $(10^{-4}, 10^{-5}, 2^{-5})$ & $(10^{-4}, 10^{-5}, 2^{3})$ & $(10^{-5}, 10^{-5}, 2^{5})$ & $(10^{-5}, 10^{-5}, 2^{5})$ \\
 & 20 \% & 91.19 & 89.43 & 90.57 & 87.42 & 91.19 & 91.19 \\
 &  & $(10^{-5}, 2^{-5})$ & $(10^{3}, 2^{2})$ & $(10^{-4}, 10^{-5}, 2^{-5})$ & $(10^{-5}, 10^{-4}, 2^{3})$ & $(10^{-5}, 10^{-5}, 2^{5})$ & $(10^{-5}, 10^{-5}, 2^{5})$ \\ \hline
{\begin{tabular}[c]{@{}c@{}}yeast-2\_vs\_4\\ (514 x 9)\end{tabular}} & 0 \% & 85.81 & 100 & 94.19 & 97.42 & 94.19 & 85.81 \\
 &  & $(10^{-5}, 2^{-5})$ & $(10^{5}, 2^{5})$ & $(10^{-1}, 10^{-1}, 2^{-5})$ & $(10^{-1}, 10^{-1}, 2^{5})$ & $(10^{-5}, 10^{-3}, 2^{5})$ & $(10^{-4}, 10^{-4}, 2^{3})$ \\
 & 5 \% & 85.81 & 95.48 & 94.19 & 93.55 & 94.19 & 85.81 \\
 &  & $(10^{-5}, 2^{-3})$ & $(10^{5}, 2^{5})$ & $(10^{-3}, 10^{-2}, 2^{-3})$ & $(10^{-1}, 1, 2^{3})$ & $(10^{-5}, 10^{-2}, 2^{5})$ & $(10^{-2}, 10^{-2}, 2^{4})$ \\
 & 10 \% & 85.81 & 90.67 & 82.26 & 97.42 & 85.81 & 85.81 \\
 &  & $(10^{-5}, 2^{-3})$ & $(10^{4}, 2^{5})$ & $(10^{-1}, 10^{-2}, 2^{-5})$ & $(10^{-1}, 10^{-1}, 2^{4})$ & $(10^{-5}, 10^{-2}, 2^{3})$ & $(10^{-4}, 10^{-4}, 2^{2})$ \\
 & 15 \% & 85.81 & 100 & 90.32 & 88.39 & 85.81 & 85.81 \\
 &  & $(10^{-5}, 2^{-3})$ & $(10^{4}, 2^{5})$ & $(1, 10^{-2}, 2^{-3})$ & $(1, 1, 2^{4})$ & $(10^{-5}, 10^{-2}, 2^{3})$ & $(10^{-4}, 10^{-1}, 2^{3})$ \\
 & 20 \% & 85.81 & 82.78 & 89.68 & 92.26 & 84.52 & 85.81 \\
 &  & $(10^{-5}, 2^{-3})$ & $(10^{5}, 2^{5})$ & $(10^{-1}, 10^{-2}, 2^{-5})$ & $(1, 1, 2^{4})$ & $(10^{-5}, 10^{-2}, 2^{3})$ & $(10^{-4}, 10^{-2}, 2^{5})$ \\ \hline
{\begin{tabular}[c]{@{}c@{}}yeast3\\ (1484 x 9)\end{tabular}} & 0 \% & 89.24 & 60.84 & 92.38 & 92.83 & 88.12 & 90.56 \\
 &  & $(10^{-5}, 2^{-5})$ & $(10^{5}, 2^{5})$ & $(10^{-2}, 10^{-1}, 2^{-3})$ & $(1, 10^{2}, 2^{5})$ & $(10^{-5}, 10^{-2}, 2^{5})$ & $(10^{-5}, 10^{-5}, 2^{5})$ \\
 & 5 \% & 88.12 & 89.42 & 90.13 & 94.62 & 96.12 & 96.12 \\
 &  & $(10^{-5}, 2^{-5})$ & $(10^{5}, 2^{1})$ & $(1, 10^{-1}, 2^{-5})$ & $(10^{-1}, 1, 2^{5})$ & $(10^{-5}, 10^{-2}, 2^{5})$ & $(10^{-5}, 10^{-5}, 2^{5})$ \\
 & 10 \% & 88.12 & 90 & 80.81 & 91.93 & 87.44 & 88.12 \\
 &  & $(10^{-5}, 2^{-5})$ & $(10^{5}, 2^{1})$ & $(10^{-1}, 10^{-1}, 2^{-5})$ & $(10^{-3}, 10^{-1}, 2^{5})$ & $(10^{-5}, 10^{-5}, 2^{5})$ & $(10^{-5}, 10^{-5}, 2^{2})$ \\
 & 15 \% & 88.12 & 100 & 87.89 & 88.12 & 88.12 & 88.12 \\
 &  & $(10^{-5}, 2^{-5})$ & $(10^{3}, 2^{1})$ & $(10^{-1}, 10^{-1}, 2^{-5})$ & $(10^{-5}, 10^{-5}, 2^{5})$ & $(10^{-5}, 10^{-5}, 2^{5})$ & $(10^{-5}, 10^{-4}, 2^{5})$ \\
 & 20 \% & 88.12 & 88.87 & 90.13 & 94.17 & 90.67 & 88.12 \\
 &  & $(10^{-5}, 2^{-5})$ & $(10^{5}, 2^{1})$ & $(10^{-1}, 10^{-1}, 2^{-5})$ & $(1, 1, 2^{5})$ & $(10^{-3}, 10^{-3}, 2^{5})$ & $(10^{-5}, 10^{-5}, 2^{5})$ \\ \hline
{Average ACC} & 0 \% & 76.27 & 79.43 & 84.83 & \textbf{88.74} & 85.85 & 84.93 \\
 & 5 \% & 77.86 & 80.7 & 85.13 & \textbf{88.61} & 86.95 & 84.57 \\
 & 10 \% & 77.19 & 84.05 & 84.51 & \textbf{89.25} & 84.79 & 84.68 \\
 & 15 \% & 79.3 & 83.24 & 85.37 & \textbf{88.49} & 85.32 & 85.48 \\
 & 20 \% & 81.08 & 81.76 & 84.6 & \textbf{86.97} & 85.39 & 85.2 \\ \hline
{Average Rank} & 0 \% & 4.42 & 4.03 & 3.5 & 2.17 & 3.31 & 3.58 \\
 & 5 \% & 4.53 & 4.38 & 3.56 & 2.21 & 2.86 & 3.47 \\
 & 10 \% & 4.43 & 3.9 & 3.81 & 2.1 & 3.33 & 3.43 \\
 & 15 \% & 4.24 & 4.5 & 3.28 & 2.43 & 3.21 & 3.35 \\
 & 20 \% & 4.18 & 4.47 & 3.44 & 2.54 & 3.1 & 3.26 \\ \hline
 \multicolumn{8}{l}{$^{\dagger}$ represents the proposed models. ACC represents the accuracy metric.} \\
 \multicolumn{8}{l}{Bold text denotes the model with the highest average ACC.}
\end{tabular}}
\end{table*}

\end{document}